\documentclass{article}

\PassOptionsToPackage{numbers, sort&compress}{natbib}
\usepackage{natbib}
\usepackage[margin=1in]{geometry}
\date{}

\usepackage[utf8]{inputenc} 
\usepackage[T1]{fontenc}  
\usepackage{hyperref}   
\usepackage{url}       
\usepackage{booktabs}      
\usepackage{amsfonts}     
\usepackage{nicefrac}  
\usepackage{microtype}      
\usepackage{xcolor}        
\usepackage{amsmath}
\usepackage[capitalize,nameinlink]{cleveref}
\usepackage{amssymb}
\usepackage{amsthm}
\usepackage{algorithm}
\usepackage{algpseudocode}
\usepackage{tikz}
\usetikzlibrary{positioning, automata}
\usepackage{subcaption}
\usetikzlibrary{arrows}
\usepackage{enumerate}
\usepackage{wrapfig}

\hypersetup{
    colorlinks=false,
    linkcolor=blue,
    filecolor=magenta,      
    urlcolor=cyan,
}

\title{Sequential Strategic Classification with Multi-Stage Selective Classifiers}

\author{%
  \textbf{Ziyuan Huang}\thanks{Equal contribution} \quad 
  \textbf{Lina Alkarmi}\footnotemark[1] \quad 
  Mingyan Liu \\
  Department of Electrical and Computer Engineering\\
  University of Michigan\\
  Ann Arbor, MI 48105 \\
  \texttt{\{ziyuanh, lalkarmi, mingyan\}@umich.edu}
}

\renewcommand{\vec}[1]{\boldsymbol{#1}}

\newcommand{\iprod}[2]{\langle #1, #2\rangle}

\DeclareMathOperator*{\argmax}{arg\,max}

\newtheorem{theorem}{Theorem}[section]
\newtheorem{proposition}{Proposition}[section]

\newtheorem{lemma}{Lemma}[section]

\newtheorem{definition}{Definition}[section]
\newtheorem{assumption}{Assumption}[section]
\newtheorem*{rep@theorem}{\rep@title}
\newcommand{\newreptheorem}[2]{%
\newenvironment{rep#1}[1]{%
 \def\rep@title{#2 \ref{##1}}%
 \begin{rep@theorem}}%
 {\end{rep@theorem}}}

\crefname{assumption}{Assumption}{assumptions}
\crefname{condition}{Condition}{conditions}
\crefname{definition}{Definition}{definition}

\begin{document}

\newreptheorem{theorem}{Theorem}
\newreptheorem{proposition}{Proposition}
\newreptheorem{corollary}{Corollary}
\newreptheorem{lemma}{Lemma}
\newreptheorem{condition}{Condition}
\newreptheorem{assumption}{Assumption}
\newreptheorem{definition}{Definition}

\maketitle

\begin{abstract}
Strategic classification studies the problem where self-interested individuals or agents manipulate their response to obtain favorable decision outcomes made by classifiers, typically turning to dishonest actions when they are less costly than genuine efforts. Prior works have demonstrated a fundamental inability to get out of this conundrum by only focusing on the design of a classifier. 
{We note that prior work also heavily focuses on either one-shot settings or repeated interaction with the same classifier. Real-world decision making is often {\em multi-stage}, involving a sequence of potentially different classifiers as an agent progresses.} This paper introduces a sequential, {multi-stage} model of strategic classification, by capturing how agents adapt their behavior---through {\em improvement} actions (enhancing both observable features and true attributes) and {\em gaming} actions (enhancing only observable features)---over multiple levels of classification with increasing difficulty as well as reward. For each level, we adopt a {\em selective classifier} that can abstain from making a prediction at low confidence. Consequently, a positive (resp. negative) outcome leads to promotion (resp. demotion) of the agent to the next higher (resp. lower) level, while abstention keeps the agent at the same level. We fully characterize the agent's optimal instantaneous action under selective classifiers and compare the long-term properties and utility of the agent repeatedly following an optimal myopic policy of either {\em no-improvement} (never choose the improvement action) or {\em no-gaming} (never choose the gaming action). We further examine design principles over the sequence of classifiers that yield higher long-term utility for the latter policy, thereby effectively incentivizing genuine effort in the long run.
\end{abstract}

\section{Introduction}
\label{intro}

Whenever information is available about a decision process (e.g., what type of input it uses), individuals subject to the decision can respond strategically to maximize their chance of receiving a favorable decision outcome through both honest and dishonest means. For example, a student may focus all their effort on a particular chapter if they know that chapter carries a disproportionate weight on an exam, or a job-seeker may artificially alter their resume to meet a hiring criterion.  The same problem similarly exists in algorithmic decision making, fueled by advances in machine learning and becoming prevalent in a multitude of domains such as loan approval, admissions, and hiring \cite{hardt_strategic_2015,haghtalab_maximizing_2020}. 
This has given rise to the growing field of strategic classification that studies utility-maximizing agents who respond strategically to a (decision-making) classifier \cite{hardt_strategic_2015,miller_strategic_2020}. Such an agent is often characterized by an {\em observable feature}, a true underlying {\em unobservable attribute} which is the agent's private information, and the option to take actions to {\em improve} both their observable feature and true attribute \cite{kleinberg_how_2019,jin_incentive_2022,milli_social_2019,ahmadi_classification_2022}, or {\em game}, which only improves their observable feature \cite{milli_social_2019,jin_collaboration_2023,miller_strategic_2020}. As a one-shot decision, the agent takes the least costly action to get their feature to cross a decision boundary \cite{hardt_strategic_2015,jin_incentive_2022,chen2020strategic,braverman_role_2020,jin_collaboration_2023}, resulting in the fundamental challenge that incentivizing improvement effort in the case of higher improvement costs is impossible \cite{milli_social_2019,jin_network_2021} without external instruments such as recourse  \cite{chen2020strategic} and reward/subsidy \cite{jin_incentive_2022,jin_collaboration_2023}.  

We observe that in practice decision making often involves multiple stages as an agent progresses. For example, a student may take two midterm exams and a final exam, and how much they prepare for the midterms may well impact their posture heading into the final. {The coupled nature of these stages} introduces a type of {\em inter-temporal} incentive that does not exist in one-shot settings.
Specifically, we introduce a novel sequential model of {{\em multi-stage}} strategic classification that captures how agents behave as they move through a sequence of classifiers with increasing difficulty as well as reward. Each level is modeled as a {\em selective classifier} ~\cite{geifman_selective_2017,lee_fair_2021,ortner_learning_2016}, that can abstain from making a prediction when confidence is low. A positive (resp. negative) decision leads to promotion (resp. demotion) of the agent to the next higher (resp. lower) level, while abstention keeps the agent at the same level. 
The agent's attribute evolves over time as a result of not only the decisions they make but also a depreciation effect that reflects the degradation of skills or knowledge without sustained effort, a common feature in the dynamic decision-making literature~\cite{dohmen_reinforcement_2023,noauthor_voluntary_nodate,alston_dynamics_1998}.  While there have been prior studies on long-term properties restricted to {either improvement or gaming action \cite{xie_automating_2024,zrnic_who_2022}, our work explicitly models both.} 
{Our model also differs fundamentally from \cite{perdomo_performative_2021, zrnic_who_2022}, where agent-classifier interactions lead both to jointly evolve, whereas we are interested in designing a sequence of classifiers to induce the agent to advance to the highest level.} 
A more comprehensive literature review is provided in \cref{app:related-work}. 

We fully characterize the agent's optimal instantaneous action and examine the utility of a myopic policy where the agent either never takes the improvement action (NI) or never takes the gaming action (NG) but otherwise best responds instantaneously. We further examine designs over the sequence of classifiers that yields higher long-term utility for the latter, thereby incentivizes genuine effort, even when gaming incurs a lower unit cost. 
{Note that even with attribute depreciation, gaming may still be preferred if it's cheaper. Indeed, there exist decision rules under which gaming is the dominant policy, resulting in completely unqualified individuals reaching very high levels.} 

{To the best of our knowledge, this is not only the first work that examines strategic responses under selective classifiers, but also the first multi-stage model that captures both improvement and gaming actions.} Our main contributions are summarized as follows. 
\begin{enumerate}
    \item We formulate a sequential strategic classification problem with a progression of selective classifiers with increasing difficulty and reward. This is cast as a Markov decision process (MDP) with continuous state and action spaces.
    \item We fully characterize the agent's optimal instantaneous response to selective classifiers and show how the problem reduces to a countable-state Markov chain when the agent acts myopically (taking the instantaneous optimal action at each time). 
    \item We analyze a set of long-term properties for the two types of myopic agents (NG and NI), including their steady-state level distribution, long-term average attribute, and utility. 
    \item We derive design principles of the classifier sequence such that NI is dominated by NG as a strategy in terms of long-term utility, thereby showing that improvement can be incentivized in the long run even when improvement actions are more costly. 

\end{enumerate}

\section{Problem Formulation}
\label{sec:problem}

We consider a self-interested agent who is faced with a progression of ``levels'' with increasing reward associated with each. Going up/down the levels signifies promotion/demotion and is determined by a sequence of ``tests'' modeled as classifiers. To pass a test in order to be promoted or not be demoted, the agent has the option to make a costly effort to improve its preparedness, or to cheat (game the system), which is less costly. Our principal interest is in examining whether an agent can be incentivized to stay honest and obtain high rewards, and if so, how best to design the sequence of classifiers. We cast this as a sequential decision problem. Below, we detail each of its components.

\paragraph{{Agent}} Consider a discrete-time process with time steps indexed by $t\in\mathbb{N}$.
At the beginning of each time step $t$, an agent has a \emph{pre-response attribute} $X_t\in\mathbb{R}_{\geq0}$\footnote{For simplicity of presentation, our analysis is based on one-dimensional attributes. However, the conclusions also apply to multi-dimensional settings, as we can always project the attribute vector onto a one-dimensional (score) space on which a classification decision is based. \label{sec:myopic:magnitude}} determined by an unobservable \emph{qualification} level $Y_t\in\mathbb{R}_{\geq0}$. An agent can exert an effort $\vec{a}_t:=[a_t^+;a_t^-]\in\mathcal{A}\subseteq\mathbb{R}^2_{\geq0}$ in order to secure a favorable outcome, where $a_t^+$ (resp. $a_t^-$) denotes an \emph{improvement} (resp. \emph{gaming}) effort; they each is associated with a unit cost of effort $c^+\geq0$ and $c^-\geq0$. Denote by $\vec{c}:=[c^+;c^-]$ the unit cost vector. We adopt the common assumption that cheating is cheaper than improvement, i.e., $0<c^-<c^+<\infty$.

Only the improvement action leads to an inherent change to the agent's attribute. We denote $\tilde{{X}}_t:={X}_t+{a}_t^+$ (resp. $\tilde{Y}_t$) as the {\em post-response attribute} (resp. {\em qualification}) of the agent at {\em the end of} time step $t$. In contrast, the {\em (observable) feature} ${Z}_t:={X}_t+{a}^+_t+{a}_t^-\in\mathbb{R}_{\geq0}$ is causal to both improvement and gaming actions, reflecting their indistinguishable effect on the outcome -- the decision maker can only observe the feature, while pre- and post-response attributes and qualification are information/signals private to the agent. 
Additionally, we will assume that the agent's attribute depreciates over time in the form of ${X}_{t+1}=\gamma\tilde{X}_t$,  $\gamma\in(0,1)$, $\gamma$ being the retention factor. This models the fact that learned skills can degrade and more effort is required to keep them up to date.

\paragraph{{Classifier}} 

We consider a sequence of ternary classifiers as a combination of the \emph{proportional odds model (POM)} and the \emph{abstention decision}. The POM is a sequence of binary classifiers on ordinal outcomes, such as ``poor'', ``fair'', ``good'', and ``extraordinary''. This model adheres to the \emph{proportional odds assumption}, where the influence of predictors remains consistent across all (cumulative) thresholds. Let there be $I\geq 2$ levels indexed by $i_t\in[I]$, each associated with a desired qualification $\mu_i\geq 0$. The POM consists of $I$ binary classifiers predicting the variable $\mathbf{1}\{\tilde{Y}_t\geq \mu_i\}$ (i.e., whether the agent's post-response qualification satisfies the level requirement) based on the observation $Z_t$. Each classifier is modeled by $\log\frac{\mathbb{P}(\tilde{Y}_t\geq \mu_i|Z_t)}{1-\mathbb{P}(\tilde{Y}_t\geq \mu_i|Z_t)} = \theta Z_t-\mu_i$, with the posterior $\eta_i(Z_t):=\mathbb{P}(\tilde{Y}_t\geq \mu_i|Z_t)=\frac{1}{1+\exp(-\theta Z_t+\mu_i)}$. Whether an agent at level $i$ will be promoted depends exclusively on this posterior $\eta_t(Z_t)$.

An \emph{abstention} function $h(\eta_i(Z_t))=4\eta_i(Z_t)(1-\eta_i(Z_t))$ represents the probability the classifier will not make a decision on the prediction $\eta_i(Z_t)$. This abstention option can reduce decision error, albeit at the expense of lower coverage \cite{franc_optimal_nodate,ortner_learning_2016}. 
Notice that $h$ is also the Gini impurity measure of a Bernoulli distribution with parameter $\eta_i(Z_t)$, with a higher value indicating a more impure distribution and thus less confident prediction, and vice versa.

It turns out this classifier choice fits very nicely with our sequential setting, where an agent can be promoted, demoted, or stay at the same level in accordance with the ternary decision outcome. {More specifically, the classifier's decision can be expressed as follows:} 
\begin{equation}\label{eq:decision_outcome}
    D_i(Z_t) = 
    \begin{cases}
    +1 \quad \text{{(pass/promote)}} & \text{with probability } \big(1 - h(\eta_i(Z_t))\big) \cdot \eta_i(Z_t) \\
    -1 \quad \text{{(fail/demote)}} & \text{with probability } \big(1 - h(\eta_i(Z_t))\big) \cdot \big(1 - \eta_i(Z_t)\big) \\
    0 \quad \;\; \text{{(abstain/no change)}} & \text{with probability } h(\eta_i(Z_t))
    \end{cases}.
\end{equation}
All parameters of the decision rule are public knowledge to the agents.

\paragraph{{Temporal dynamics/transition}} 
Each level $i$ is associated with a reward of $R_i:=r\cdot i$ with $r>0$.
{We assume the agent starts at the first level: $i_0=1$, with some initial $(x_0, y_0)$. At each time $t$, the agent is at level $i_t$ with $(x_t, y_t)$, takes an action $\vec{a}_t$, which results in an observable feature $z_t$, and is subject to the classifier $D_{i_t}$. This results in one of the three outcomes given in \cref{eq:decision_outcome} and accordingly, at the next time step $t+1$ the agent moves to  $i_{t+1}=i_t+1$,  $i_{t+1}=i_t-1$, or stay at $i_{t+1}=i_t$, and collects the reward $R_{i_{t+1}}$. The process then repeats.} 
 In the cases of $i=1$ and $i=I$, the agent remains at the same level whenever $D_i\neq +1$ and $D_i\neq -1$, respectively. We define the transition probabilities $p_i^+(z_t):=\mathbb{P}(D_i(z_t)=1)$, $p_i^0(z_t):=\mathbb{P}(D_i(z_t)=0)$, and $p_i^-(z_t):=\mathbb{P}(D_i(z_t)=-1)$, and the $I\times I$ transition probability matrix $P(z_t)$ with row vector  $\vec{p}_i(z_t)$. The evolution of $\{i_t\}_t$ is depicted in \cref{fig:mc}, {which resembles that of a birth-death Markov chain, with the note that the state as depicted is incomplete} because the transition probabilities depend on the underlying pre-response attribute $x_t$ and action $\vec{a}_t$, {which are time dependent; for this reason this system is in general time-inhomogeneous.}

\begin{figure}[!htbp]
    \centering
        \vspace{-10pt}
    \resizebox{.6\textwidth}{.16\textwidth}{
        \begin{tikzpicture}[
            ->,>=stealth,shorten >=1pt,auto,thick,
            state/.style={circle,draw,fill=blue!20,thick,minimum size=3em}
        ]
            \node[state] (state_1) {$1$};
            \node (text_1) [right=of state_1] {$\dots$};
            \node[state] (state_i) [right=of text_1] {$i$};
            \node[state] (state_i_next) [right=of state_i] {${i+1}$};
            \node (text_2) [right=of state_i_next] {$\dots$};
            \node[state] (state_I) [right=of text_2] {$I$};
    
            \path[every node/.style={fill=white,sloped,font=\sffamily\small}]
                (state_1) edge [bend right=30] node[below=1mm] {$p_1^+$} (text_1)
                (state_i) edge [bend right=30] node[below=1mm] {$p_i^+$} (state_i_next)
                (state_i_next) edge [bend right=30] node[below=1mm] {$p_{i+1}^+$} (text_2)
                (state_i_next) edge [bend right=30] node [above=1mm] {$p_{i+1}^-$} (state_i)
                (state_i) edge [bend right=30] node[above=1mm] {$p_{i}^-$} (text_1)
                (state_I) edge [bend right=30] node[above=1mm] {$p_{I}^-$} (text_2)
                (state_1) edge [loop above] node[above=1mm] {$p_{1}^0$} (state_1)
                (state_i) edge [loop above] node[above=1mm] {$p_{i}^0$} (state_i)
                (state_i_next) edge [loop above] node[above=1mm] {$p_{i+1}^0$} (state_i_next)
                (state_I) edge [loop above] node[above=1mm] {$p_{I}^0$} (state_I);
        \end{tikzpicture}
    }
    \caption{The evolution of the agent's level.}
    \vspace{-10pt}
    \label{fig:mc}
\end{figure}
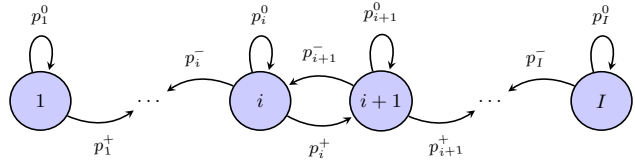
We write the state variable as $s_t:=(i_t,x_t)$. 
{The process $\{{s}_t\}_t$ is a discrete-time Markov process with an uncountable state space and potentially time-homogeneous state transitions (if the action at time $t$ is completely determined by ${s}_t$), as we will focus on for the rest of the paper.} We will shortly show that when the agent adopts a greedy strategy, $\{{s}_t\}_t$, conditioned on an initial ${x}_0$, the induced state space becomes countable.

\paragraph{Agent Policy} 
A \emph{greedy} agent maximizes its instantaneous profit at each step, i.e.,
\begin{eqnarray}
\label{eq:inst-utility}
    \pi(s) \in \argmax_{\vec{a}\in\mathcal{A}}\;\mathbb{E}[R(s,a)|i] - \vec{c}^\top\vec{a} & \forall s\in\mathcal{S}.
\end{eqnarray}
For an honest (resp. dishonest) greedy agent, who never chooses gaming (resp. improvement) actions, we have $\mathcal{A}=\mathbb{R}_{\geq0}\times\{0\}$ (resp. $\mathcal{A}=\{0\}\times \mathbb{R}_{\geq0}$). We denote the respective solutions as $\pi^{+}_g$ for the \emph{honest greedy policy} and $\pi^-_g$ for the \emph{dishonest greedy policy}. Notably, the general greedy policy can be obtained by setting $\mathcal{A}=\mathbb{R}_{\geq0}^2$. However, due to the assumption that $c^-<c^+$, this policy only exerts gaming efforts, which makes it the same as $\pi_g^-$ (or equivalently, $\pi_g^-$ is the optimal \emph{general greedy policy}). This is a commonly observed problem in the strategic classification literature \cite{jin_incentive_2022,jin_collaboration_2023}.

On the contrary, a \emph{farsighted} agent maximizes the (discounted) total profit over the entire interaction process, i.e.,\footnote{This equation implies that the optimal Markov policy is the optimal policy, which is true due to the underlying Markov reward process.}
\begin{eqnarray}\label{eq:farsighted-opt}
    \pi\in\argmax_{g\in\mathcal{S}\to \mathcal{A}}\;U(g):=\mathbb{E}\left[\sum_{t=0}^\infty\beta^t\left(R(S_t,g(S_t))-\vec{c}^\top g(S_t)\right)\bigg\lvert i_0=0, X_0=0\right]
\end{eqnarray}
where $\beta\in(0,1)$ is a discount factor. 
Similarly, we define the \emph{honest farsighted policy} $\pi_f^+$ and \emph{dishonest farsighted policy} $\pi_f^-$ respectively as the solutions to Eq. \eqref{eq:farsighted-opt} when $\mathcal{A}=\mathbb{R}_{\geq0}\times\{0\}$ and $\mathcal{A}=\{0\}\times \mathbb{R}_{\geq0}$. Furthermore, we are also interested in the \emph{general farsighted policy} $\pi_f^\circ$ where $\mathcal{A}=\mathbb{R}_{\geq0}^2$.

\section{Optimal Agent Policy and State Reduction}
\label{sec:incentive_improve}
Our goal in this section is to show that under a properly designed threshold sequence $\{\mu_i\}_{i\in[I]}$, $\pi_f^-$ is dominated by $\pi_f^+$ for a farsighted agent. This would then imply that the principal can achieve improvement through incentivization by classifier design alone, without resorting to external resources. This is accomplished in two steps: we first show that $\pi_f^-=\pi_g^-$ in this problem setting, and then construct the sequence $\{\mu_i\}_{i\in[I]}$ under which $U(\pi_g^+)\geq U(\pi_g^-)$.

\subsection{The Optimal Greedy Policy}
The constructed threshold sequence depends on the structure of the greedy policies, which we will characterize in this section.
Notice that $\pi_g^+$ and $\pi_g^-$ have similar shapes (they are obtained from the same optimization problems of different parameters). We temporarily drop the superscript in this sub-section for simplicity of presentation.

Define an auxiliary function $G:[I]\times\mathbb{R}\to\mathbb{R}$ as follows
\begin{equation}\label{eq:G}
    G(i,z)=\begin{cases}
        r(1-h(\sigma(\theta  z)))\sigma(\theta  z)  - cz & i=1 \\
        r(1-h(\sigma(\theta  z)))(2\sigma(\theta  z)-1) - cz & 2\leq i \leq I-1 \\
        -r(1-h(\sigma(\theta  z)))(1-\sigma(\theta  z)) - cz & i=I
    \end{cases}
\end{equation}
where $\sigma(x)=\frac{1}{1+e^{-x}}$ is the sigmoid function. Then, it immediately follows that the greedy policies can be fully determined by this auxiliary function.
\begin{lemma}\label{prop:br-form}
    $\pi_g(s) \in \mu_i-x+\argmax_{z\geq x-\mu_i}\;G(i,z)$.
\end{lemma}

\begin{wrapfigure}{}{0.55\textwidth}
    \begin{minipage}{0.55\textwidth} 
        \centering
        \begin{subfigure}[b]{.45\textwidth}
            \includegraphics[width=\textwidth]{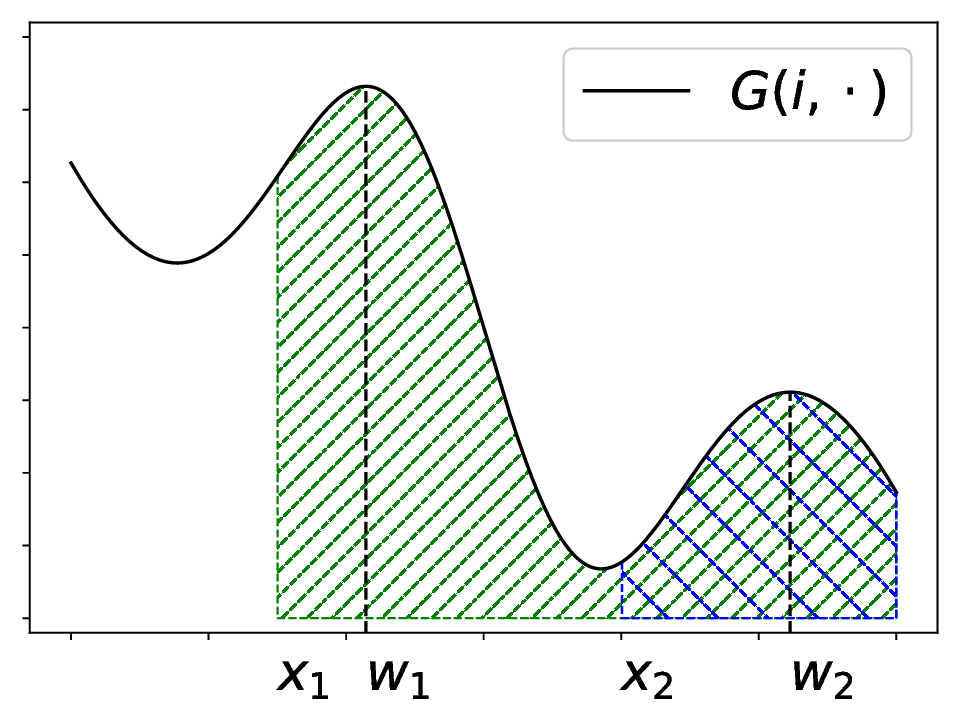}
            \subcaption{}
            \label{fig:br:type1-G}
        \end{subfigure}
        \begin{subfigure}[b]{.45\textwidth}
            \includegraphics[width=\textwidth]{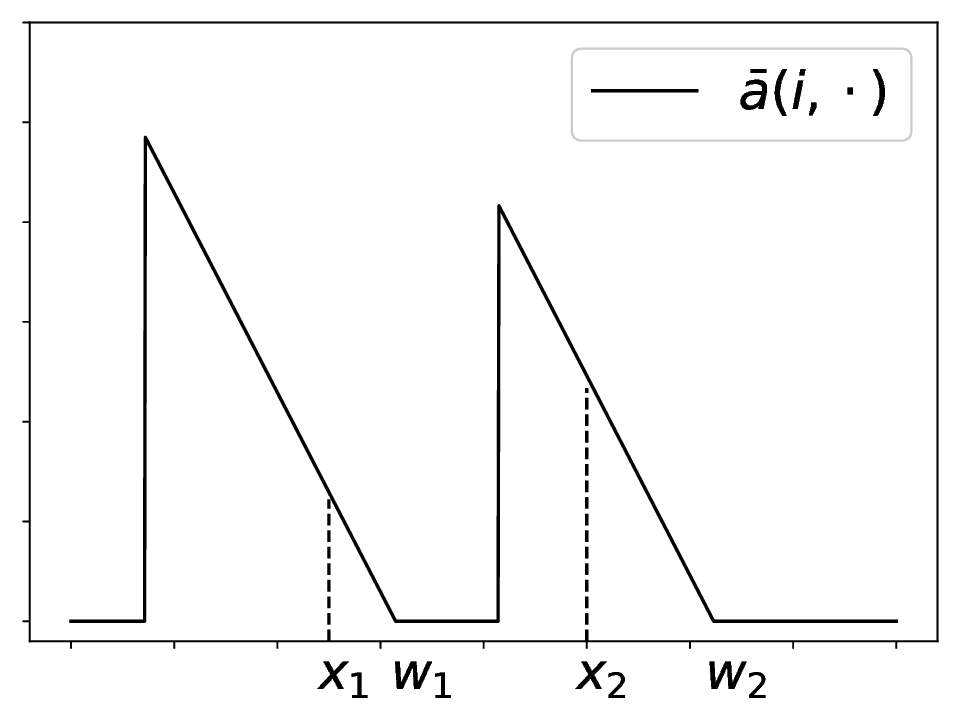}
            \subcaption{}
            \label{fig:br:type1-br}
        \end{subfigure}
        \begin{subfigure}[b]{.45\textwidth}
            \includegraphics[width=\textwidth]{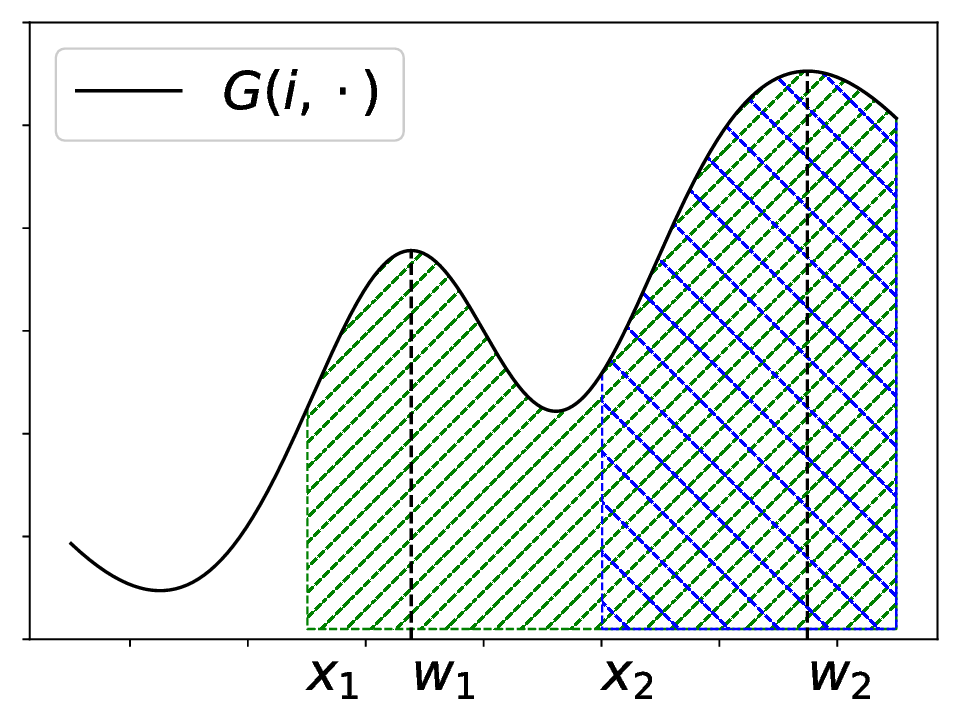}
            \subcaption{}
            \label{fig:br:type2-G}
        \end{subfigure}
        \begin{subfigure}[b]{.45\textwidth}
            \includegraphics[width=\textwidth]{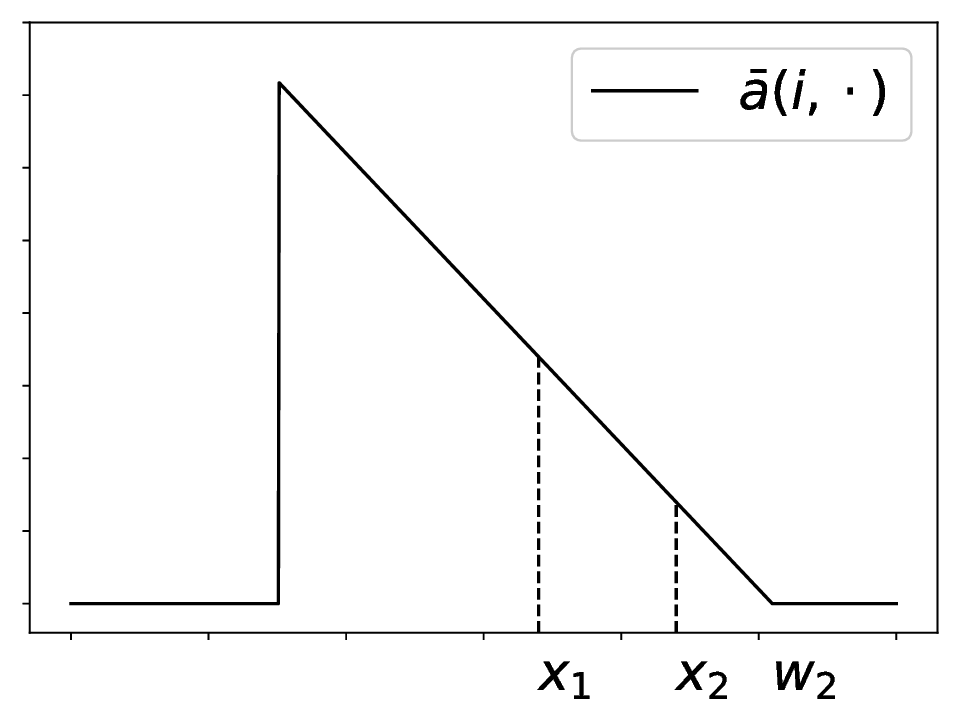}
            \subcaption{}
            \label{fig:br:type2-br}
        \end{subfigure}
        \caption{Illustration of $G(i,\cdot)$ (left, with $z$ being the x-axis) and $\pi_g(i,\cdot)$ (right, with $x$ being the x-axis).}
        \label{fig:best-response}
  \end{minipage}
\end{wrapfigure}
When $G(i,\cdot)$ has multiple local maximizers, $\pi_g(i,\cdot)$ can exhibit disconnected positive intervals, as shown in \cref{fig:best-response}. The left (resp. right) column depicts $G(i,\cdot)$ (resp. $\pi_g(i,\cdot)$), while the top (resp. bottom) row shows the case when the rightmost local maximizer is not (resp. is) the global maximizer. Let $w_1<w_2$ be the two local maximizers of $G(i,\cdot)$ in this example. When $w_2$ is not the global maximizer (i.e., \cref{fig:br:type1-G}), 
the agent with attribute $x_1$ (resp. $x_2$) would select $\pi_g=\mu_1+w_1-x_1$ (resp. $\mu_2+w_2-x_2$), by maximization over the green diagonal (resp. blue backward-diagonal) regions.
This creates disconnected positive segments in $\pi_g(i,\cdot)$ (i.e., \cref{fig:br:type1-br}), where the lower-attribute agent targets at staying in the current level (left triangle) and the higher-attribute agent seeks promotion to the next level (right triangle). 
It is easy to see that $\pi_g(i,\cdot)$ has at most two disconnected segments.
When $w_2$ is the global maximizer (i.e., \cref{fig:br:type2-G}), agents with both $x_1$ and $x_2$ would choose actions associated with $w_2$, yielding a single-triangle in the best-response action (i.e., \cref{fig:br:type2-br}). In this case, aiming for promotion is more beneficial to the agent than staying at the current level.
We will focus on this latter case by introducing the following assumption.
\begin{assumption}\label{asm:local-maximizer}
    Let $\textrm{LM}_i$ be the set of local maximizer of $G(i,\cdot)$. Assume for every $i\in[I]$, $\textrm{LM}_i$ is non-empty and $G(i,\max\textrm{LM}_i)\geq G(i,z)$, $\forall z\in\textrm{LM}_i$.
\end{assumption}

Under \cref{asm:local-maximizer}, $\pi_g(i,x)$ conforms with \cref{fig:br:type2-br} and can be written analytically as
\begin{equation}\label{eq:br-simplified}
    \pi_g(i,x)=(\overline{w}_i+\mu_i-x)\mathbf{1}\{\mu_i + \underline{w}_i\leq x<\mu_i+\overline{w}_i\}~, 
\end{equation}
where $\overline{w}_i$ is the largest local maximizer and $\underline{w}_i:=\inf\{z:G(i,z)\leq G(i,\overline{w}^i)\}$. {The relationship of $\underline{w}_i$ and $\overline{w}_i$ is illustrated in \cref{fig:w_underbar}.} %
For ease of notation, we will use $\overline{\mu}_i:=\mu_i+\overline{w}_i$ and $\underline{\mu}_i:=\mu_i+\underline{w}_i$ to denote the boundaries of the indicator function.
\begin{proposition}\label{prop:boundary-pr}
    $p_i^+(\overline{\mu}_i)\geq p_i^-(\overline{\mu}_i)$, $\forall i\in[I-1]$.
\end{proposition}
\begin{wrapfigure}{!htbp}{0.35\textwidth}
    \begin{minipage}{0.35\textwidth} 
        \centering
        \vspace{-30pt}
        \includegraphics[width=\textwidth]{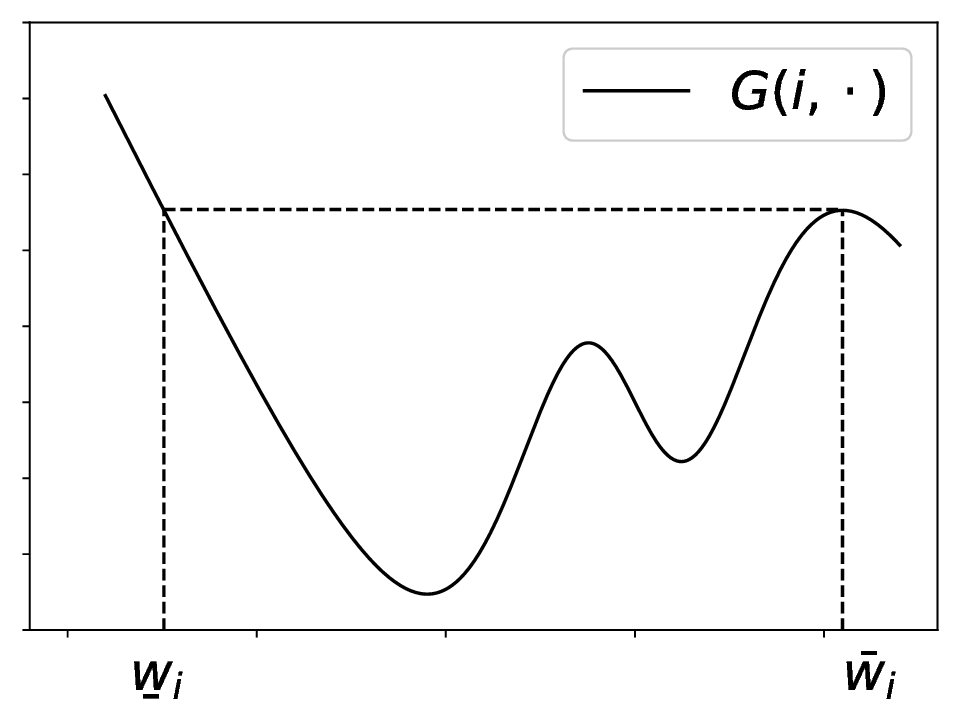}
        \caption{Relationship between $\underline{w}_i$ and $\overline{w}_i$.}
        \label{fig:w_underbar}
        \vspace{-20pt}
    \end{minipage}
\end{wrapfigure}
\cref{prop:boundary-pr} says that (as a consequence of adopting \cref{asm:local-maximizer}) for non-terminal levels ($i<I$), the agent exerts effort {(either improvement or gaming)} primarily for promotion. For the terminal level $I$, however, demotion probability may exceed promotion probability, reflecting the agent's shifted focus to status maintenance instead of advancement. 
 
\paragraph{Reduction to countable-state Markov chain.} The piecewise form of instantaneous best response allows reformulation of the state process $\{s_t\}_t$ as a countable-state Markov chain conditioned on $x_0$. According to \cref{eq:br-simplified}, whenever the agent inputs nontrivial efforts in improvement, $x_{t+1}$ can only be $\gamma\overline{\mu}_i$ for some $i\in[I]$. When the agent chooses not to improve (either zero effort or gaming), its next attribute is a $\gamma$-multiple of its current attribute. Thus, we only need to consider the discretized attribute space $\mathcal{X}(x_0):=\{\gamma^tx_0:t=0,1,\dots\}\cup\{\gamma^t\overline{\mu}_i:i\in[I],t=1,2,\dots\}$, which is clearly a countable set. An example illustration of the discretization process is provided in \cref{fig:discrete-state}. 

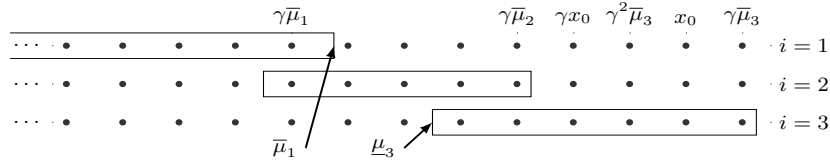
\begin{figure}[!htbp]
    \centering
    \resizebox{.8\textwidth}{.16\textwidth}{
        \begin{tikzpicture}
    \draw (4.8,0.8) -- (4.8,0.8) node[above,font=\footnotesize] {$\gamma\overline{\mu}_3$};
    \draw (1.6,0.8) -- (1.6,0.8) node[above,font=\footnotesize] {$\gamma\overline{\mu}_2$};
    \draw (-1.6,0.8) -- (-1.6,0.8) node[above,font=\footnotesize] {$\gamma\overline{\mu}_1$};
    \draw (4.0,0.8) -- (4.0,0.8) node[above,font=\footnotesize] {\color{black}$x_0$};
    \draw (3.2,0.8) -- (3.2,0.8) node[above,font=\footnotesize] {$\gamma^2\overline{\mu}_3$};
    \draw (2.4,0.8) -- (2.4,0.8) node[above,font=\footnotesize] {\color{black}$\gamma x_0$};
    
    \draw (5.2,0.6) -- (5.2,0.6) node[right,font=\footnotesize] {$i=1$};
    \draw (5.2,0) -- (5.2,0) node[right,font=\footnotesize] {$i=2$};
    \draw (5.2,-0.6) -- (5.2,-0.6) node[right,font=\footnotesize] {$i=3$};
    \draw [-latex, thick] (-1.4, -1.0) to (-1.0,0.6) {};
    \draw (-1.4, -1.0) -- (-1.4, -1.0) node[left, font=\footnotesize] {$\overline{\mu}_1$};
    \draw [-latex, thick] (0.0, -1.0) to (0.4,-0.6) {};
    \draw (0.0, -1.0) -- (0.0, -1.0) node[left, font=\footnotesize] {$\underline{\mu}_3$};
    \draw (5.0,-0.8) rectangle (0.4,-0.4);
    \draw (1.8,-0.2) rectangle (-2.0,0.2);
    \draw (-5.6,0.8) -- (-1.0,0.8) -- (-1.0,0.4)  -- (-5.6,0.4);
    
    \foreach \x in {-4.8,-4.0,...,5.6}{
        \foreach \y in {-0.6,0,0.6}{
            \fill[black!80] (\x,\y) circle[radius=1.4pt];
    }}

    \draw (-5.0,0.6) -- (-5.0,0.6) node[left,font=\footnotesize] {$\cdots$};
    \draw (-5.0,0) -- (-5.0,0) node[left,font=\footnotesize] {$\cdots$};
    \draw (-5.0,-0.6) -- (-5.0,-0.6) node[left,font=\footnotesize] {$\cdots$};
  \end{tikzpicture}
  }
  \caption{A countable state space with $x_t$ plotted along the horizontal axis (increasing to the {right}) and $i_t$ along the vertical axis (increasing downward). Each rectangle corresponds to the interval $[\underline{\mu}_i,\overline{\mu}_i]$, {though $\underline{\mu}_i$ and $\overline{\mu}_i$ themselves are in general not part of the state space.}
  } 
  \label{fig:discrete-state}
\end{figure}

\section{Incentivizing Improvement in the Long Term}\label{sec:long-term}

{In this section we will analyze and compare two types of greedy agents (or policies): one that never games/cheats (denoted as NG, i.e., does nothing or exerts improvement effort) over an infinite horizon, and one that never improves (denoted as NI, i.e., does nothing or exerts gaming effort). We will adopt the one-shot best response derived in \cref{sec:incentive_improve}, specifically, by setting $\vec{c}^+<\vec{c}^-$ (resp. $\vec{c}^-<\vec{c}^+$). We will also examine the conditions under which the former outperforms the latter.} 
{Our simulation results in \cref{sec:simulation} show that NG is indeed the optimal over the class of random mixed strategies.}

Let $c^g$, $G^g$, $\overline{w}^g_i$, $\underline{w}_i^g$, $\overline{\mu}^g_i$, and $\underline{\mu}_i^g$ denote the corresponding entities defined in \cref{sec:incentive_improve} for $g\in\{\textrm{NG},\textrm{NI}$\}, and let $\overline{a}^\textrm{NG}$ and $\overline{a}^{\textrm{NI}}$ denote the best-response actions in the NG and NI policies,   respectively.
We are interested in the following {long-term time-average} quantities for each $g\in\{\textrm{NG},\textrm{NI}\}$:
\begin{enumerate}
    \item Level distribution: $\hat{q}^g_i:=\lim_{T\to\infty}\frac{1}{T}\sum_{t=0}^T\mathbf{1}\{i^g=i\}$, $\forall i$, {i.e., the agent's achievement}. 
    \item Average attribute: $\hat{x}^g:=\lim_{T\to\infty}\frac{1}{T}\sum_{t=0}^Tx_t^g$, {i.e., the agent's true quality}. 
    \item Average utility: $\hat{u}^g:=\lim_{T\to\infty}\frac{1}{T}\sum_{t=0}^Tu(i_t,x_t,\overline{a}^g(i_t,x_t))$.
\end{enumerate}

{Obtaining these quantities depends on the the design of the sequence of classifiers. Below we first describe such a sequence that is aimed at incentivizing maximum improvement from an NG agent.} 

\subsection{{A sequence of classifiers: incremental thresholding}}

{Let the classifier parameters in \cref{sec:problem} be fixed, except for the thresholds $\mu_i$. The set of conditions below define a sequence of classifiers with incremental increases in thresholds, representing tests that are progressively harder but not significantly so. We show in the next two subsections that this incentivizes an NG agent to always exert positive improvement effort, resulting in steady improvement in its attribute (true qualification).} 

\begin{definition}[Incremental Thresholding]\label{cond:level-cond} {We say a sequence of classifiers with fixed $\alpha, \sigma$, and $h$ have incremental thresholds if it satisfies the following set of conditions:}
    \begin{enumerate}[(a)]
        \item $\underline{\mu}_1^\textrm{NG}\leq 0$ (incentivize zero-attribute agents);
        \item $\overline{\mu}^\textrm{NG}_{i-1} < \overline{\mu}_i^\textrm{NG}$, $\forall i\geq 2$; for non-boundary $i\in\{3, \cdots, {I}\}$, this is equivalent to $\mu_{i-1}<\mu_i$ (strictly increasing thresholds to induce improvement aimed at advancing to higher levels); 
        \item $\gamma\overline{\mu}_i^\textrm{NG}\geq \max\left\{\underline{\mu}_{i+1}^\textrm{NG},\underline{\mu}_i^\textrm{NG}\right\}$, $\forall i\leq I-1$, and $\gamma\overline{\mu}_I^\textrm{NG}\geq \underline{\mu}_I^\textrm{NG}$ ({maintain incentives despite depreciation in attributes).} 
    \end{enumerate}
\end{definition}

{Later in \cref{alg:find-level} we will show how to construct such a set of classifiers.} 
Each of these conditions has clear interpretations. \cref{cond:level-cond}-(a) ensures that a NG agent with zero initial attribute {at level 1} will exert positive effort. This allows us to eliminate the trivial (and uninteresting) scenario where an agent with attribute below $\underline{\mu}_1^\textrm{NG}$ will persistently exert zero effort.
\begin{wrapfigure}{!htbp}{0.5\textwidth}
    \vspace{-20pt}
    \begin{minipage}{0.5\textwidth} 
        \begin{algorithm}[H]
          \caption{Finding the Longest Level Sequence Satisfying \Cref{cond:level-cond}.}
          \label{alg:find-level}
          \begin{algorithmic}[1]
            \Require increment $\Delta\mu>0$, some upper bound on the terminal threshold $M>\Delta\mu$, $\overline{w}_1^\textrm{NG}$, $\overline{w}_2^\textrm{NG}$, $\overline{w}_I^\textrm{NG}$, $\underline{w}_1^\textrm{NG}$, $\underline{w}_2^\textrm{NG}$ and $\underline{w}_I^\textrm{NG}$.
            \State $\mu_1\gets0,\;\mu_2\gets\Delta\mu,\;i\gets2$
            \If{$\overline{w}_{1}^\textrm{NG}-\overline{w}_2^\textrm{NG} < \Delta\mu\leq \gamma\overline{w}_1^\textrm{NG}-\underline{w}_2^\textrm{NG}$}
                    \While{$\mu_{i}+\Delta\mu\leq  M$ and $\Delta\mu \leq (\gamma-1)\mu_{i}+\gamma\overline{w}_{2}^\textrm{NG}-\underline{w}_2^\textrm{NG}$}
                        \State $\mu_{i+1}\gets\mu_{i}+\Delta\mu$;~~ $i\gets i+1$
                    \EndWhile
            \EndIf
            \State $\mu_{i+1}\gets\mu_{i}+\Delta\mu$; ~~
             $i\gets i+1$
            \State Define $k_x:=\mathbf{1}\{x=1\}+2\mathbf{1}\{x>1\}$
            \For{$j=i,\dots,2$}
                \If{$\mu_{j-1}+\overline{w}_{k_{j-1}}^\textrm{NG}-\overline{w}_I^\textrm{NG}<\mu_j\leq \gamma(\mu_{j-1}+\overline{w}_{k_{j-1}}^\textrm{NG})-\underline{w}_I^\textrm{NG}$}
                    \State Return $\{\mu_1,\dots,\mu_j\}$
                \EndIf
            \EndFor
            \State No $\{\mu_i\}_{i=1}^I$ with $I\geq 2$ exists; return null.
          \end{algorithmic}
        \end{algorithm}
      \end{minipage}
          \vspace{-30pt}
\end{wrapfigure}

\Cref{cond:level-cond}-(b) ensures that the levels {represent a progression/advancement of skills/attributes: at each level, an agent's improvement effort leads to increased post-response attribute (with the possible exception of the boundary state $i=1$).} 
\Cref{cond:level-cond}-(c) {ensures that when an NG agent best responds at level $i$,  its depreciated post-response attribute $\gamma\overline{\mu}_i^\textrm{NG}$ remains in or above the incentivizable interval of the next level it lands in ($i-1$, $i$, or $i+1$, corresponding to demotion, no change, and promotion, respectively) regardless of the classifier outcome; consequently its chance of promotion at the next time step, after best-responding, will  remain higher than that of demotion.} 
Taken together, such a design allows an NG agent to gradually and consistently obtain positive decision outcomes with higher probability.

It is interesting to note that when combined, \Cref{cond:level-cond}-(b) and (c) also imply that  $\gamma>\max_{i\geq2}\left(1-\frac{\overline{w}^\textrm{NG}_i-\underline{w}^\textrm{NG}_i}{\mu_{i-1}+\overline{w}^\textrm{NG}_{i-1}}\right)$, which $\to1$ as $\mu_{i-1}\to\infty$, This suggests that for a fixed {retention} $\gamma<1$, there does not exist an infinite sequence of incremental thresholding classifiers that satisfy the above conditions.
{In fact, since higher attribute depreciates more, there does not exist a decision rule that can incentivize an NG agent to increase its attribute beyond $\max_{i\in[I]}\frac{\gamma\overline{w}^\textrm{NG}_i-\underline{w}_i^\textrm{NG}}{1-\gamma}$.}

\subsection{Analysis of long-term level distribution and attribute}\label{sec:level-attribute}

\begin{theorem}\label{thm:long-term-levels}
    {For the sequence of classifiers satisfying \cref{cond:level-cond}, $\hat{q}^g_i=O((\frac{1-\sigma^g}{\sigma^g})^{|\hat{i}^{g}-i|})$ where $\sigma^g:=\sigma(\alpha\overline{w}_2^g)$, $\hat{i}^\textrm{NI}:=l=\max\{j\in[I]:\underline{\mu}_j^\textrm{NI}\leq 0\}$, and $\hat{i}^\textrm{NG}:=I$, $\forall g\in\{\textrm{NI},\textrm{NG}\}$.}
\end{theorem}

\cref{thm:long-term-levels} says that for an NI agent, the proportion of visits to each level on either side of level $l$ decays exponentially in their distance to $l$; i.e., such an agent's level is concentrated around  $l$. On the other hand, for an NG agent, the long-term distribution of level decays exponentially in the distance to the highest level $I$, thus  concentrating its most mass around $I$. Consequently, the NG agent achieves probabilistically higher steady-state levels than the NI agent {as long as} $l<I$, i.e., when the terminal threshold is high enough to deter gaming. {We provide experimental evidence on the gaming deterrence effect of the number of levels in \cref{sec:simulation}.} 

\begin{theorem}\label{thm:long-term-attribute}
    Using classifiers given by  \cref{cond:level-cond}, $\hat{x}^\textrm{NI}=0$ and $\hat{x}^\textrm{NG}\geq \gamma\min_{i\in[I]}\overline{\mu}_i^\textrm{NG}$.
\end{theorem}
\cref{thm:long-term-attribute} implies that the NG policy always yields a higher long-term attribute than the NI policy, regardless of the average level they each achieve in the long run (i.e., this holds even for $l=I$).
 
\subsection{Analysis of long-term utility}
We now compare the two types of agents/policy's long-term utility. 
Without loss of generality, we consider evenly-spaced levels, i.e., $\mu_{i-1}-\mu_i\equiv \Delta\mu$ for $2\leq i\leq I$.  {Notice that \cref{cond:level-cond}-(a) is satisfied on if $\underline{w}_1^\textrm{NG}\leq 0$ since $\mu_i\geq0$, a condition easily satisfied by our empirical evidence in \cref{sec:simulation}.}  When it holds, \cref{alg:find-level} shows a procedure that can be used to generate the longest sequence of levels with {constant increment $\Delta\mu$ with a terminal threshold not exceeding some prescribed $M$} (see \cref{app:alg-proof}).
\begin{theorem}\label{thm:long-term-utility}
    {When $\underline{w}_1^\textrm{NG}\leq 0$}, for $\{\mu_i\}_{i}$ generated by \cref{alg:find-level}, we have $\hat{u}^\textrm{NG}\geq\hat{u}^\textrm{NI}$ if 
    \begin{equation*}
        \Delta\mu\leq \frac{2\sigma^\textrm{NI}-1}{\sigma^\textrm{NI}}\left(\frac{2\sigma^\textrm{NI}-1}{(\sigma^\textrm{NI})^2}\overline{w}_l^\textrm{NI}-\underline{w}_2^\textrm{NI}-\frac{c^\textrm{NG}}{c^\textrm{NI}\sigma^\textrm{NI}}\Delta w^\textrm{NG}\right) -\frac{r}{c^\textrm{NI}\sigma^\textrm{NI}}\left(\frac{\sigma^\textrm{NG}(2\sigma^\textrm{NI}-1)}{\sigma^\textrm{NI}(2\sigma^\textrm{NG}-1)}-1\right),
    \end{equation*}
    where $l$, $\sigma^\textrm{NI}$ and $\sigma^\textrm{NG}$ are defined in \cref{thm:long-term-levels} and $\Delta w^\textrm{NG} =\max_{i\in[I]}(\overline{w}_i^\textrm{NG}-\underline{w}_i^\textrm{NG})$.
\end{theorem}

\cref{thm:long-term-utility} says that as long as the increments in level thresholds are suitably small, the sequence of classifiers can yield  higher long-term utility for the NG agent, {the intuition being that small increments entail lower effort costs per time step, but the cumulative effort results in higher attribute and level attainment.}
Note that \cref{cond:level-cond}-(b) requires $\Delta\mu>0$, meaning the condition in \cref{thm:long-term-utility} {fails if $c^\textrm{NI}\ll c^\textrm{NG}$ or $r$ is very large (cost of cheating is much lower or reward for cheating is high).}  Furthermore, $\sigma^\textrm{NI}:=\sigma(\alpha\overline{w}_2^\textrm{NI})$ is increasing in $\alpha$, so do functions $\sigma\mapsto \frac{2\sigma-1}{\sigma^2}$ and  $\sigma\mapsto \frac{2\sigma-1}{\sigma}$. Thus, the above condition is more likely to hold for larger $\alpha$, i.e., classifiers with higher confidence.
Lastly, since last term in the condition is positive as $\sigma^\textrm{NI}>\sigma^\textrm{NG}$, improvement is more likely to be incentivizable when the reward $r$ is not too large. 
\section{Simulation Results}\label{sec:simulation}

In addition to the properties defined in \cref{sec:long-term}, we will also examine numerically the decision maker's utility, defined as the average accuracy 

$$   \mathbb{E}^g_{(i,x)\sim P_\infty}\left\{(1-h(\sigma(\alpha z))\left[\sigma(\alpha z)\mathbf{1}\{\tilde{x}\geq \mu_i\}+(1-\sigma(\alpha z))\mathbf{1}\{\tilde{x}<\mu_i\}\right]\right\}$$ 

where $P_\infty$ denotes the stationary state distribution and $g$ denotes the agent's strategy.
In what follows, \cref{sec:sim:levels} examines how the threshold {increment} $\Delta\mu$ impacts different long-term properties, \cref{sec:sim:mixed} generalizes the agent policy to a mixed myopic strategy and a pure farsighted strategy, and \cref{sec:sim:ablation} summarizes results from various ablation tests; {complete results are in provided \cref{app:experiment}}.

\subsection{On the impact of threshold increment}
\label{sec:sim:levels}

We fix $\alpha = 4$, $\beta = 0.604$, $\gamma = 0.9$, and simulate for $T = 10,000$ steps; the unit cost for the NI agent and NG agent are 0.75 and 0.80, respectively; the normalized entropy function is used for abstention. The reward per level is $r = 1.0$.
{We obtain the thresholds from \cref{alg:find-level} with $M=3$ and $\Delta\mu=M/I$. The generated classifier sequence for each $I$ satisfies $\mu_I=M=3$.}
\cref{fig:ex4_summary_panel} compares the long-term properties of NI and NG as $I$ increases. 
There is a clear trend that NG achieves increasingly higher levels and utilities than NI, and the average attribute of NG is bounded away from that of NI.

\begin{figure}[!htbp]
    \centering
    \begin{subfigure}[t]{0.33\textwidth}
        \centering
        \includegraphics[width=\textwidth]{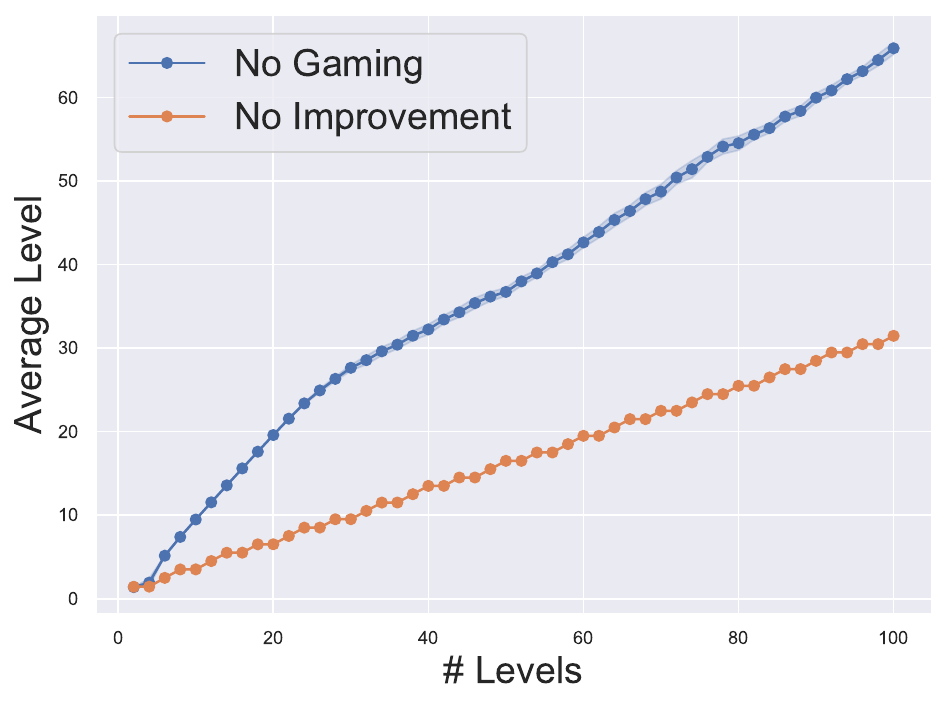}
        \caption{Average levels.}
        \label{fig:ex4_level}
    \end{subfigure}%
    \hfill
    \begin{subfigure}[t]{0.33\textwidth}
        \centering
        \includegraphics[width=\textwidth]{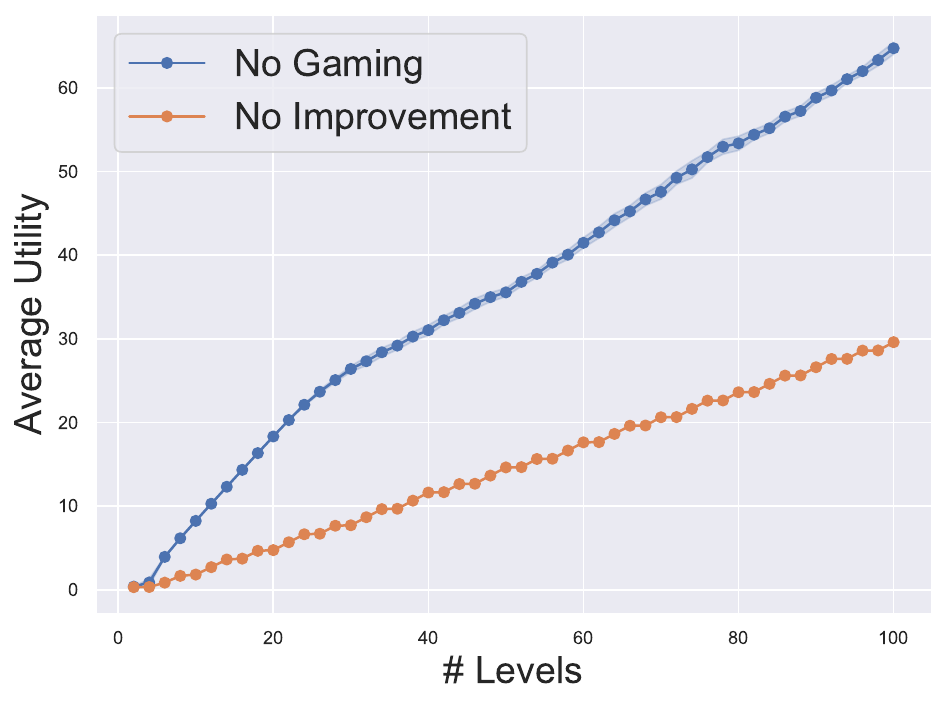}
        \caption{Average utilities.}
        \label{fig:ex4_utility}
    \end{subfigure}%
    \hfill
    \begin{subfigure}[t]{0.33\textwidth}
        \centering
        \includegraphics[width=\textwidth]{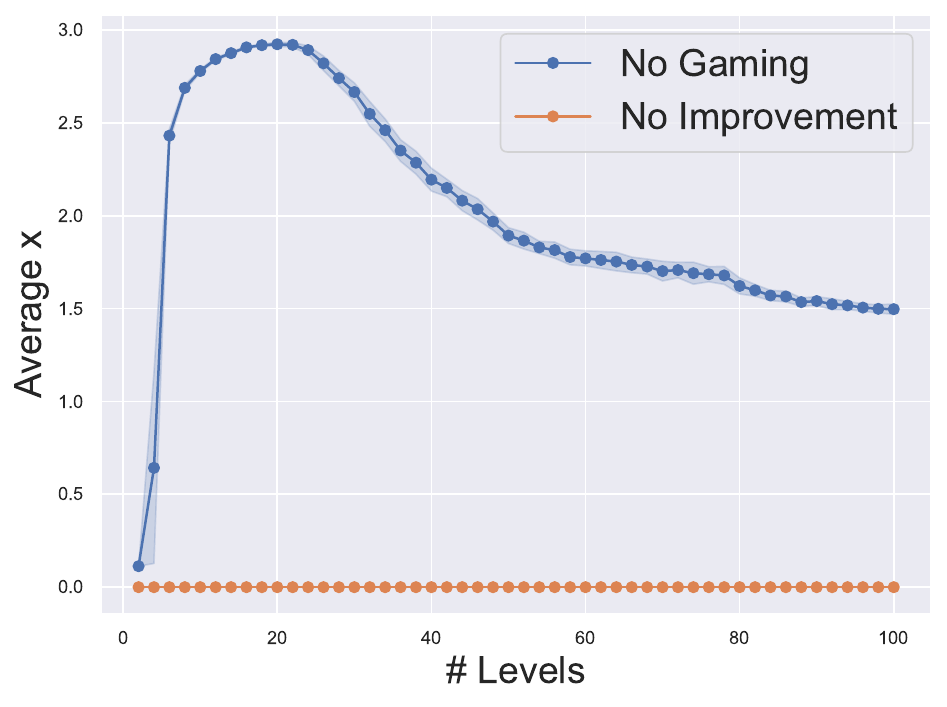}
        \caption{Average attributes.}
        \label{fig:ex4_qualification}
    \end{subfigure}
    \caption{Long-term properties across different numbers of levels ($I$) {$\pm2$ standard deviation (STD)}.}
    \label{fig:ex4_summary_panel}
\end{figure}

\subsection{Mixed strategies and the optimal pure strategy policy}
\label{sec:sim:mixed}

Consider a mixed myopic policy $\textrm{Mix}(\rho)$, whereby 
at each time $t$, the agent chooses the optimal instantaneous improvement action with probability $\rho\in(0,1)$ and the optimal instantaneous gaming action with probability $1-\rho$. For this set of experiments, we use the classifiers generated in \cref{sec:sim:levels} with $I=10$, and a polynomial abstention function (see \cref{app:alpha}). 
\cref{fig:ex5_stationary} displays the stationary level distribution under $\textrm{Mix}(\rho)$, and highlights that while a small fraction of improvement can be counterproductive (e.g., $\rho<0.3$), a moderate to majority proportion of improvement (e.g., $\rho>0.4$) leads to consistent concentrations at the highest level. \cref{fig:ex5_agent} shows the long-term utility of the agent and the decision maker, respectively. Both long-term utilities increase as more improvement actions are adopted, despite a significant initial drop of the decision maker's utility when the gaming action is the majority (e.g., $\rho<0.1$). 
This is because when improvement is lightly mixed into the gaming actions,  the occasional attribute improvement makes the agent reach and cheat at more levels than no improvement at all, leading to a decrease in the classifier's accuracy. 

\begin{figure}[htbp]
    \centering
    \hfill
    \begin{subfigure}[t]{0.45\textwidth}
        \centering
        \includegraphics[width=.8\textwidth]{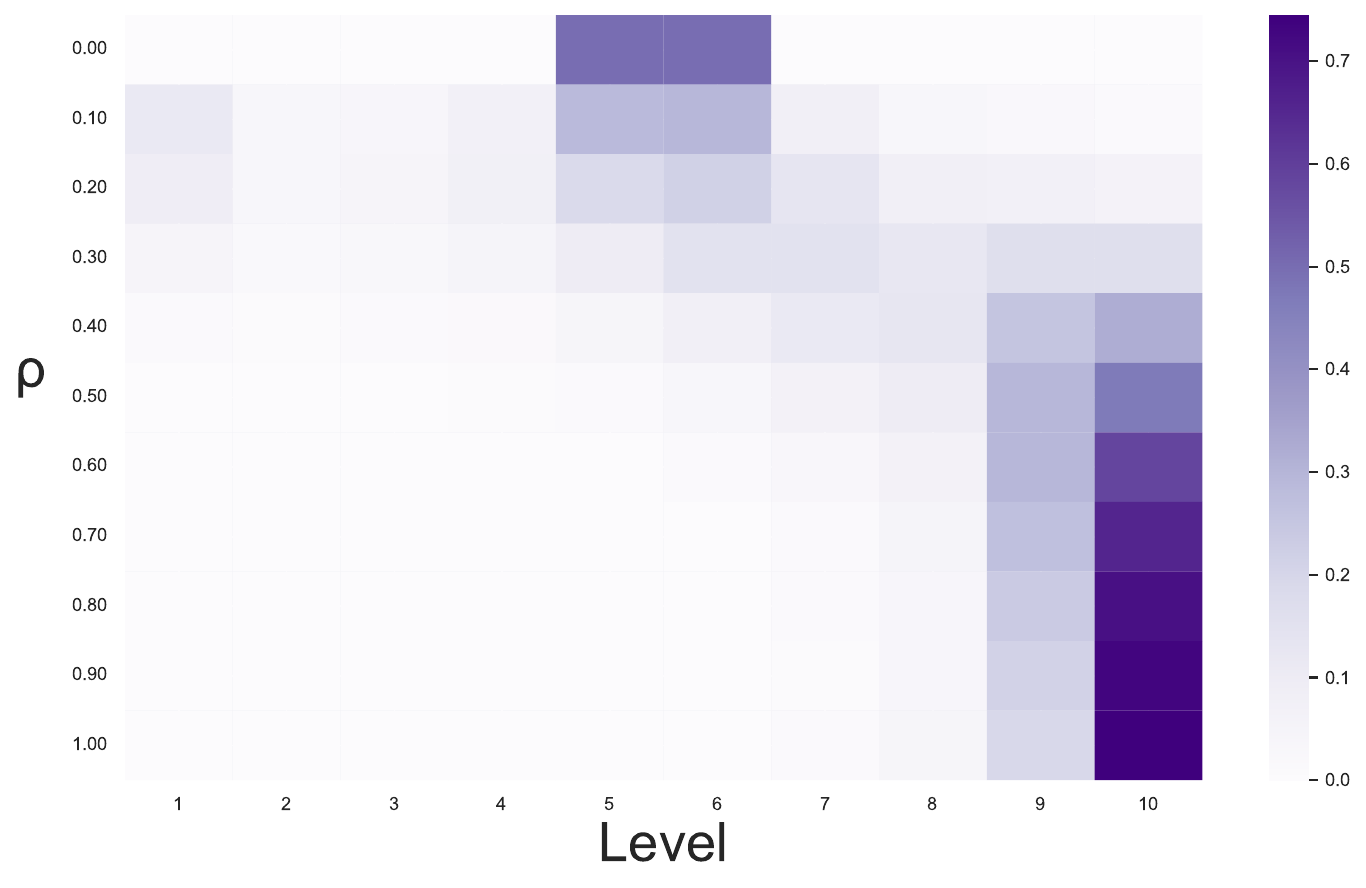}
        \caption{Level distribution under $\textrm{Mix}(\rho)$.}
        \label{fig:ex5_stationary}
    \end{subfigure}
    \hfill
    \begin{subfigure}[t]{0.45\textwidth}
        \centering
        \includegraphics[width=.8\textwidth]{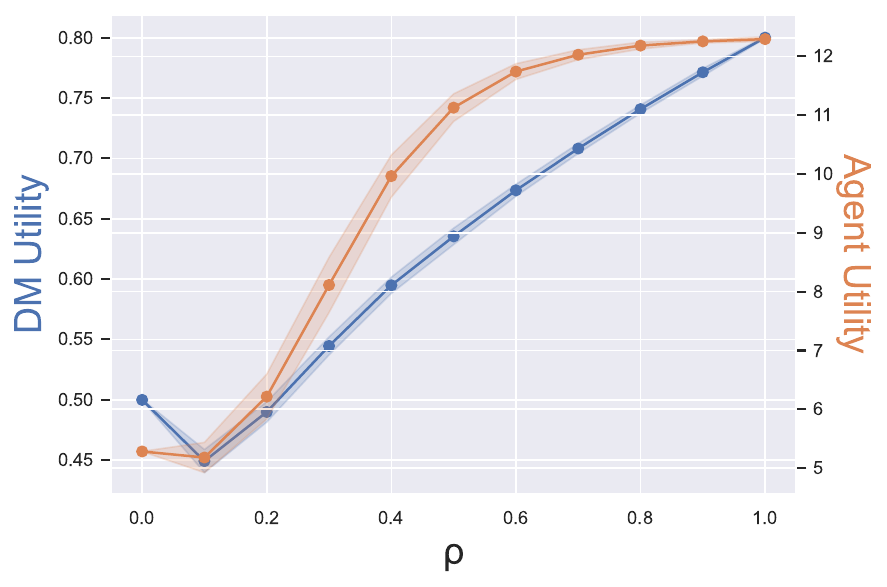}
        \caption{Utilities under $\textrm{Mix}(\rho)$ { $\pm2$ STD}.}
        \label{fig:ex5_agent}
    \end{subfigure}%
    \caption{Long-term properties under $\textrm{Mix}(\rho)$ versus $\rho$ (probability of no gaming).}
    \label{fig:ex5_summary_panel}
\end{figure}

{
As a further generalization, we consider the optimal pure strategy policy where the agent can choose either $a^+$, $a^-$, or both, to maximize the following discounted reward over a large finite horizon: 
\begin{equation*}
    \mathbb{E}\left[\eta\sum_{t=0}^{T-1}\left(r\cdot i_{t+1}-c^+a_t^+-c^-a_t^-\right)\Bigg|x_0=0\right]~. 
\end{equation*}
We first quantize the action/attribute space and then run an RL algorithm for 2000 episodes, each with $T=2500$ steps (see the algorithm detail and other parameter setup in \cref{app:exp10}). The solution corresponds to the true optimal policy with pure strategies and the complete action space. 
We find that even when the improvement cost is larger, the optimal policy can still exert more improvement efforts than gaming efforts. This phenomenon is more pronounced with smaller threshold gaps (see details in \cref{app:exp10}), aligning with the implications of \cref{thm:long-term-utility}.
}

\subsection{Ablation tests}
\label{sec:sim:ablation}

We show that higher rewards $r$ increase long-term average levels for both NI and NG, but narrow their utility gap $\hat{u}^\textrm{NG}-\hat{u}^\textrm{NI}$ (see \cref{app:r}). Raising classifier confidence $\alpha$ increases both strategies’ utilities, while NG benefits more consistently (see \cref{app:alpha}). While abstention sensitivity $\tilde{\beta}:=h(1/2)$ has mild effects, absolute-value abstention reduces NG’s utility as $\tilde{\beta}$ grows, suggesting smoother functions (e.g., entropy/parabolic) better incentivize improvement (see \cref{app:beta}).
\section{Conclusion and Limitation}
\label{sec:conclusion}
This paper studies a long-term, multi-stage strategic classification problem where decision making is driven by a sequence of progressively more difficult classifiers representing advancement in levels.  We derived conditions for the classifier design that incentivize long-term improvement actions, even when improvement is more expensive than gaming. A major limitation of this study is that the analysis is restricted to various myopic policies. While we do present simulation results, it remains an open question to determine how a farsighted (or truly optimal) policy would behave, and how social well-being measures such as fairness and welfare may be impacted by this type of classifier design. {Notably, the setting of multi-stage classification can lead to different fairness issues and new insights compared to the one-shot problems, as depreciation can also differ across sensitivity groups.}

\bibliography{reference}

\appendix
\crefalias{section}{appendix}
\crefalias{subsection}{subappendix}
\crefalias{subsubsection}{subsubappendix}

\section{Related Works}\label{app:related-work}
\paragraph{Strategic classification.}
Strategic classification aims at enhancing the robustness of machine learning models under data providers' strategic manipulation~\cite{hardt_strategic_2015,braverman_role_2020,miller_strategic_2020,milli_social_2019,haghtalab_maximizing_2020,jin_incentive_2022,chen2020strategic,kleinberg_how_2019,chen_learning_2020,bechavod_gaming_2021}. \cite{hardt_strategic_2015} pioneered the research of strategic classification and formulated the problem as a Stackelberg game, where the decision maker leads by publishing a classifier and the agents best respond by strategically manipulating their features. Derivative works have included optimal classifier design in anticipation of the agent's strategic response~\cite{braverman_role_2020,miller_strategic_2020,zrnic_who_2022}, discussions on social welfare~\cite{milli_social_2019,haghtalab_maximizing_2020}, mechanism design that limits manipulations~\cite{jin_incentive_2022,chen2020strategic,kleinberg_how_2019}, and learning algorithms under imperfect information~\cite{chen_learning_2020,bechavod_gaming_2021,dong_strategic_2017}. 



\paragraph{Improvement incentivization.} While a large body of literature has focused on malicious manipulation/gaming actions~\cite{hardt_strategic_2015,braverman_role_2020}, a comparably large amount of works also consider the existence of both improvement (causal to both true attribute and observable feature) and gaming (causal only to the observable feature) actions~\cite{miller_strategic_2020,bechavod_gaming_2021,ahmadi_classification_2022,jin_incentive_2022,chen2020strategic}. \cite{jin_incentive_2022} points out that improvement can only be incentivized when there exists an improvement direction that is more cost efficient than all gaming directions and \cite{bechavod_gaming_2021} shows the possibility of learning such direction when it exists. Improvement incentivization becomes much challenging in a more practical situation when such direction does not exist. \cite{jin_incentive_2022,jin_collaboration_2023,chen2020strategic} propose to leverage external instruments, such as transfer mechanisms, collaborations, or recourse policies, to address this problem. In contrast, the present paper shows the possibility of incentivizing improvement, in the long run, by classifier design only.


\paragraph{Long-term strategic classification.} Long-term impacts of machine learning systems have been extensively studied in the literature. \cite{zhang_group_2019,zhang_how_2020} examine the distribution dynamics of the agent under indefinite and repeated interactions with fair classifiers, where agents' attribute distribution is updated each time as a result of responding to the classification outcome. Performative prediction~\cite{perdomo_performative_2021,jin_performative_2024,jin_addressing_2024} further abstracts the agent's response as the sensitivity of the population distribution and studies the model and population dynamics in the long run. However, long-term properties with explicit best-response actions are less commonly studied, with the exception of \cite{xie_automating_2024} that examines the evolving dynamics between machine learning models and the population where human best responses are taken as annotated samples in the next round of interaction and \cite{zrnic_who_2022} that characterizes the equilibrium outcome of the repeated Stackelberg game between the decision maker and the strategic agent. 
All works discussed above assume a singular effect of the agent's actions (i.e., either causal/improvement or non-causal/gaming to the agent's attribute), while our work is the first to model both improvement and gaming actions explicitly in the study of long-term properties.

\paragraph{Selective classification.} In this work, we adopt the concept of abstention in {\em selective classifiers} or {\em classifiers with rejection option}~\cite{shah_selective_nodate,el-yaniv_foundations_nodate,geifman_selective_2017,lee_fair_2021,ortner_learning_2016}. Selective classifiers reject the prediction in case of high classification loss, which is almost always related to the highly uncertain predictions~\cite{franc_optimal_nodate}, regardless of what loss functions are adopted. This was first introduced by~\cite{el-yaniv_foundations_nodate}, who introduced the risk-overage (RC) trade-off that balances prediction accuracy against coverage, i.e., the fraction of instances that the classifier is willing to classify. Derivative works have expanded selective classification in multiple ways, including generalization to arbitrary uncertainty score~\cite{franc_optimal_nodate}, adoption in deep learning models~\cite{geifman_selective_2017}, and fairness analysis of selective classification~\cite{lee_fair_2021}.

\section{Direction of the best-response action} 

\begin{wrapfigure}{!htbp}{0.55\textwidth}
\vspace{-20pt}
    \begin{minipage}{0.55\textwidth} 
        \centering
        \begin{subfigure}[b]{.45\textwidth}
            \centering
            \begin{tikzpicture}[scale=0.6]
            \draw[very thin,color=gray,step=0.5,draw opacity=0.3] (-.1,-.1) grid (3.9,3.9);
            \draw[->] (-0.2,0) -- (4.2,0) node[right] {$a_1$};
            \draw[->] (0,-0.2) -- (0,4.2) node[above] {$a_2$};
            
            \draw[-latex,color=blue,thick] (0.5,0.5) -- (0.5,{0.5+1});
            \draw[-latex,color=blue,thick] (0.5,0.5) -- (0.5,{0.5+1*(2.9/1.5)});
    
            \draw[-latex,color=red,thick] (0.5,0.5) -- ({3.9-0.5}, 0.5);
            \draw[-latex,color=red,thick] (0.5,0.5) -- ({0.5+1.5},0.5);
    
            \draw[dashed,color=black] (0.5,{0.5+1}) -- ({0.5+1.5},0.5);
            \draw[dashed,color=black] (0.5,{0.5+1*(2.9/1.5)}) -- ({3.9-0.5}, 0.5);
            

            \filldraw[fill=black] (0.5,0.5) circle (2pt) node[below left=-1pt and -1.7pt] {$x$};

            \draw[color=black,thick] plot[domain=0:2.5] (\x,{2.5-\x}) node[above=.1in] {};
            \draw[color=black,thick] plot[domain=0:3.9] (\x,{3.9-\x}) node[above=.2in] {$\vec{\theta}$};
            \draw (.5,.5);
        \end{tikzpicture}
            \caption{Optimal action w/ single best direction $a_1$.}
            \label{fig:direction}
        \end{subfigure}
        \begin{subfigure}[b]{.45\textwidth}
            \centering
            \begin{tikzpicture}[scale=0.6]
                \draw[very thin,color=gray,step=0.5,draw opacity=0.3] (-.1,-.1) grid (3.9,3.9);
                \draw[->] (-0.2,0) -- (4.2,0) node[right] {$a_i$};
                \draw[->] (0,-0.2) -- (0,4.2) node[above] {$a_j$};
                
                \draw[-latex,color=blue,thick] (0.5,0.5) -- ({3.9-0.5}, 0.5);
                \draw[-latex,color=blue,thick] (0.5,0.5) -- (0.5,{0.5+1.5});
                \draw[-latex,color=blue,thick] (0.5,0.5) -- ({0.5+1.5},0.5);
                \draw[-latex,color=blue,thick] (0.5,0.5) -- (0.5,{3.9-0.5});
                
                \draw[-latex,color=red,thick] (0.5,0.5) -- ({0.5+1.5/4},{2.5-(0.5+1.5/4)});
                \draw[-latex,color=red,thick] (0.5,0.5) -- (1.25,{2.5-1.25});
                \draw[-latex,color=red,thick] (0.5,0.5) -- ({0.5+1.5*3/4},{2.5-(0.5+1.5*3/4)});
        
                \draw[-latex,color=red,thick] (0.5,0.5) -- ({0.5+2.9/4},{3.9-(0.5+2.9/4)});
                \draw[-latex,color=red,thick] (0.5,0.5) -- ({0.5+2.9/2},{3.9-(0.5+2.9/2)});
                \draw[-latex,color=red,thick] (0.5,0.5) -- ({0.5+2.9*3/4},{3.9-(0.5+2.9*3/4)});
                
                \filldraw[fill=black] (0.5,0.5) circle (2pt) node[below left=-1pt and -1.7pt] {$x$};
                
                \draw[color=black,thick] plot[domain=0:2.5] (\x,{2.5-\x}) node[above=.1in] {};
                \draw[color=black,thick] plot[domain=0:3.9] (\x,{3.9-\x}) node[above=.2in] {$\vec{\theta}$};
                \draw (.5,.5);
            \end{tikzpicture}
            \caption{Optimal action w/ two best directions $a_i$, $a_j$.}
            \label{fig:convex-comb}
        \end{subfigure}
        \hfill
        \caption{Agent's optimal actions (red arrows).}
    \end{minipage}
    \vspace{-10pt}
\end{wrapfigure}
Our result in this first step is a relatively minor generalization of a result in \cite{jin_incentive_2022}. For this reason, we will focus on illustrating the intuition behind the result and have included the technical details in \cref{app:sec:best-response}. \cref{fig:direction} illustrates a two-dimensional action space where the return on investment (ROI) differs: $\frac{\theta_1}{c_1}>\frac{\theta_2}{c_2}$. For an attribute $x$ below the decision threshold, a fixed investment in the higher-ROI direction  (i.e., $a_1$) always brings the agent closer to the decision boundary than the other (i.e., $a_2$). Thus, the agent's optimal action is to invest exclusively in $a_1$.
While non-convex utility functions may produce multiple separated optimal actions, these solutions differ only in magnitude, not direction, as shown in \cref{fig:direction}. Each magnitude corresponds to a unique estimated post-response qualification $\hat{y}$.
When multiple directions share the same ROI (i.e., $\frac{\theta_i}{c_i}=\frac{\theta_j}{c_j}$ for $i\neq j$), optimal actions form convex sets aligned with the highest ROI-directions. Each set corresponds to a unique $\hat{y}$ and consists of all convex combinations of the unidirectional actions that achieve that $\hat{y}$. This is depicted in \cref{fig:convex-comb}.

Let $K=\argmax_{l\in[2d]}\frac{\tilde{\theta}_l}{c_l}$, {where $\tilde{\vec{\theta}}:=[\vec{\theta};\vec{\theta}]\in\mathbb{R}_{\geq0}^{2d}$}, consists of the highest-ROI directions. 
Notice $K$ is independent of the state ($i$ and $x$). We break ties by letting the agent invest exclusively in an arbitrary direction in $K$, which is then fixed throughout the decision process. This allows us to focus on a one-dimensional problem (see \cref{sec:myopic:magnitude}). 
It can also be shown that the magnitude of this unidirectional optimal action, denoted as $\overline{a}_t$, depends on $\vec{x}_t$ only through $x_t:=\iprod{\vec{\theta}}{\vec{x}_t}$. Thus, we will write this optimal action as $\overline{a}(i,x_t)$ and adopt the following simplified notations:
\begin{eqnarray*}\label{eq:notation-simplification}
\tilde{x}_t:=\iprod{\vec{\theta}}{\tilde{\vec{x}}_t},\; z_t:= \iprod{\vec{\theta}}{\vec{z}_t}=\hat{y}_t,\;a_t:=\iprod{\vec{\theta}}{\vec{a}_t},\;s_t:=(i_t,x_t),\;\textrm{and}\;c:= c_{\min K}/\tilde{\theta}_{\min K}~. 
\end{eqnarray*}
We refer to $x_t$/$\tilde{x}_t$, $z_t$, $a_t$, and $s_t$ as well as their vectorized versions interchangeably as pre-/post-response attribute, observable feature, action, and state, respectively.  
The instantaneous utility can also be expressed in terms of the simplified notations by $u(i_t,x_t,a_t) = \mathbb{E}[ri_{t+1}|i_t]- c a_t$.

\section{Supplementaries of Section~\ref{sec:incentive_improve}}
\subsection{Finite Local Maximizer of $G(i,\cdot)$}\label{app:finite:lm}
\begin{proposition}
    The function $G(i,\cdot)$ has finite local maximizers for every $i$.
\end{proposition}
\begin{proof}
    Since $G(i,\cdot)$ is analytic, it is sufficient to show that $\frac{\partial G(i,z)}{\partial z}$ has finite roots. Let $LM_i\subseteq\mathbb{R}$ contains all critical of $G(i,\cdot)$. Let $\overline{z}\in\mathbb{R}$ be an accumulation point of $LM_i$. This implies that there exists a sequence $\{z_n\}_n\subseteq LM_i$ such that $z_n\to\overline{z}$. Due to continuity of $\frac{\partial G(i,z)}{\partial z}$, we have $\frac{\partial G(i,\overline{z})}{\partial z}=\lim_{n\to\infty}\frac{\partial G(i,z_n)}{\partial z}=0\implies \overline{z}\in LM_i$. By the identity theorem of analytic functions~\cite{krantz_primer_2002}, $\frac{\partial G(i,z)}{\partial z}\equiv 0$ on $z\in\mathbb{R}$ which is a contradiction. 
\end{proof}

\subsection{Detailed Analysis of the Structure of the Optimal Actions}\label{app:sec:best-response}
Define the {\em instantaneous best-response set} as $BR(i,\vec{x})=\argmax_{\vec{a}\in\mathbb{R}_{\geq0}^{2d}}u(i,\vec{x},\vec{a})$. We show that the best-response set in this model is a union of convex hulls. For easier subscripting, introduce $\tilde{\vec{\theta}}:=[\vec{\theta};\vec{\theta}]\in\mathbb{R}_{\geq0}^{2d}$ and then the estimated qualification can be written as $\hat{y}_t=\iprod{\vec{\theta}}{\vec{x}_t}+\iprod{\tilde{\vec{\theta}}}{\vec{a}_t}$.

\begin{lemma}\label{lemma:compact-maximizer}
    Let $f:\mathbb{R}\supset D\mapsto\mathbb{R}$ be continuous with compact domain $D$. Then, the set of maximizers $M=\argmax_{x\in D}f(x)$ is nonempty and compact. 
\end{lemma}
\begin{proof}
    By the Extreme Value Theorem, $f$ attains its maximum in $D$. Thus, $M$ is nonempty. Let $f(D)$ denote the (compact) range of $f$ over $D$. Thus, $M=\{x\in D:f(x)=\sup f(D)\}$. Let $(x_n)_n\subset M$ be an converging sequence with limit $\overline{x}\in D$. Then, $f(x_n)\to f(\overline{x})$ by continuity of $f$. Since $x_n\in M$, the sequence $f(x_n)$ is constant and equal to $\sup f(D)$ for all $n$. Thus, $f(\overline{x})=\sup f(D)$, implying $\overline{x}\in M$. So, $M$ is closed. Since $M$ is bounded because $D$ is bounded, we conclude that $M$ is compact. 
\end{proof}

\begin{lemma}\label{lemma:eq-space}
    Set $\vec{u}\in\mathbb{R}_{>0}^n$ and $\phi>0$. Then, $\{\vec{v}\in\mathbb{R}_{\geq0}^n:\iprod{\vec{u}}{\vec{v}}=\phi\}= \textrm{conv}\{(\phi/u_n)\vec{e}_n:n=1,\dots,N\}$.
\end{lemma}
\begin{proof}
    For $\vec{v}$ satisfying $\iprod{\vec{u}}{\vec{v}}=\sum_nu_nv_n=\phi$, define $\lambda_n:=u_nv_n/\phi$. Notice $\sum_n\lambda_n=1$ and $\lambda_n\geq0$, $\forall n$. Thus, $\vec{v}=\sum_nv_n\vec{e}_n=\sum_n\lambda_n(\phi/u_n)\vec{e}_n$, implying $\vec{v}\in\textrm{conv}\{(\phi/u_n)\vec{e}_n:n=1,\dots,N\}$.

    For the reverse direction, let $\lambda_n\geq0$, $\forall n$, and $\sum_n\lambda_n=1$. Then, we have $\iprod{\vec{u}}{\sum_n\lambda_n(\phi/u_n)\vec{e}_n}=\sum_n\lambda_n(\phi/u_n)u_n=\phi\sum_n\lambda_n=\phi$. This implies $\sum_n\lambda_n(\phi/u_n)\vec{e}_n\in\{\vec{v}\in\mathbb{R}_{\geq0}^n:\iprod{\vec{u}}{\vec{v}}=\phi\}$.
\end{proof}

\begin{lemma}\label{lemma:substitutability}
    For every $\overline{\vec{a}}\in BR(\vec{s})$, we have $\overline{a}_k=0$, $\forall k\notin K$.
\end{lemma}
\begin{proof}
    Denote $\psi(h(\vec{z}))=\psi(\iprod{\vec{\theta}}{\vec{z}}):=\mathbb{E}[D_t(i,\vec{z})]$ which is differentiable by the definition of $D_t$. Then, the Lagrangian of the instantaneous utility (\cref{eq:inst-utility}) is $L(\vec{a},\vec{\lambda})=\psi(\iprod{\vec{\theta}}{\vec{x}+P\vec{a}})-\iprod{\vec{c}}{\vec{a}}+\iprod{\vec{\lambda}}{\vec{a}}$ where $\vec{\lambda}\in\mathbb{R}_{\geq0}^{2d}$ and $P=\begin{pmatrix}
        I_{d} & I_{d}
    \end{pmatrix}$ is a projection matrix with $I_d$ representing the $d\times d$ identity matrix. Denote the solution to $\max_{\vec{a}}\min_{\vec{\lambda}}L(\vec{a},\vec{\lambda})$ as $(\vec{a}^*,\vec{\lambda}^*)$. Notice $\vec{a}^*$ will be the instantaneous best response $\overline{\vec{a}}$. The first-order condition on $\vec{a}$ yields
    \begin{equation}
         \psi'(\iprod{\vec{\theta}}{\vec{x}+P\vec{a}^*})P^T\vec{\theta}-\vec{c}+\vec{\lambda}^* = 0
    \end{equation}
    which implies $\psi'(\iprod{\vec{\theta}}{\vec{x}+P\vec{a}^*})P^T\vec{\theta}\leq \vec{c}$. This is equivalent to $\psi'(\iprod{\vec{\theta}}{\vec{x}+P\vec{a}^*})\leq \min_{l\in[2d]}\frac{c_{l}}{\theta_{l\bmod d}}$ (assume $c_l/0=\infty$). Therefore, for $k\notin K$, we have $\lambda^*_j>0$, implying $a_j^*=0$.
\end{proof}

The next theorem presents a formal statement of the findings. It is worth noting that this applies to arbitrary (not necessarily convex) utility functions.
\begin{theorem}\label{app:prop:best-response}
        Define $K=\argmax_{l\in[2d]}\frac{\tilde{\theta}_l}{c_l}$. For any $(i,\vec{x})\in[I]\times\mathcal{X}$, there exists a nonempty and compact subset $\Phi=\Phi(\iprod{\vec{\theta}}{\vec{x}})\subset\mathbb{R}_{\geq0}$ such that 
    \begin{equation*}
        BR(i,\vec{x})=\bigcup\limits_{\phi\in \Phi}\textrm{conv}\left\{\frac{\phi}{\tilde{\theta}_k}\vec{e}_k:k\in K\right\}.
    \end{equation*}
\end{theorem}
\begin{proof}
    Let $\overline{\vec{\theta}}:=[\vec{\theta};\vec{\theta}]\in\mathbb{R}_{\geq0}^{2d}$. 
    \cref{lemma:substitutability} implies that $\overline{\vec{a}}$ can be determined simply by maximizing 
    \begin{multline*}
        g\left(\iprod{\vec{\theta}}{\vec{x}}+\sum_{k\in K}\overline{\theta}_ka_k\right)-\sum_{k\in K}c_ka_k 
        =  g\left(\iprod{\vec{\theta}}{\vec{x}}+\sum_{k\in K}\overline{\theta}_ka_k\right)-\sum_{k\in K}\frac{c_k}{\overline{\theta}_k}\overline{\theta}_ka_k 
        \\=g\left(\iprod{\vec{\theta}}{\vec{x}}+\sum_{k\in K}\overline{\theta}_ka_k\right)-\frac{c_{\min K}}{\overline{\theta}_{\min K}}\sum_{k\in K}\overline{\theta}_ka_k = g\left(\iprod{\vec{\theta}}{\vec{x}}+\phi\right)-\frac{c_{\min K}}{\overline{\theta}_{\min K}}\phi=:J(\phi)
    \end{multline*}
    where $\phi:=\sum_{k\in K}\overline{\theta}_ka_k$. Let $\Phi$ denote the set of maximizers of $J(\phi)$ w.r.t. $\phi$. Since $D_t$ is ternary, $g$ is bounded and thus $\Phi\subseteq[0,\phi_M]$ for a large but finite $\phi_M$. Thus, we only consider optimization of $J(\phi)$ over the compact interval $[0,\phi_M]$. By \cref{lemma:compact-maximizer}, $\Phi$ is nonempty and compact. The best response set can then be written as $BR(\vec{s})=\bigcup_{\phi\in\Phi}\{\vec{a}\in\mathbb{R}_{\geq0}^d:\sum_{k\in K}\overline{\theta}_ka_k=\phi\}$. This is a union of interactions between hyperplanes and the nonnegative orthant. The proof is completed by applying \cref{lemma:eq-space} to each intersected hyperplane.
\end{proof}
All actions in the same convex hull are equivalent in terms of the estimated post-response qualification.
For any $\overline{\vec{a}}\in\textrm{conv}\{(\phi/\tilde{\theta}_k)\vec{e}_k:k\in K\}$ with weights $\lambda_k\geq0$, $\forall k\in K$, we obtain $\iprod{\vec{\theta}}{\overline{\vec{a}}^++\overline{\vec{a}}^-}=\sum_{k\in K}\lambda_k\tilde{\theta}_k(\phi/\tilde{\theta}_k)=\phi\sum_k\lambda_k=\phi$. Thus, $\hat{y}=\iprod{\vec{\theta}}{\vec{x}}+\phi$ holds for any action in the same convex hull. 
We select the maximal $\phi$ (via $\sup\Phi$, which exists due to compactness of $\Phi$) and fix the effort direction at the minimal index in $K$, yielding the tie-broken action  $(\sup\Phi/\tilde{\theta}_{\min K})\vec{e}_{\min K}$.
This represents preference for maximal manipulation (i.e., largest-$\phi$ hull) and a consistent effort direction (i.e., $\min K$ is independent of $i$ and $\vec{x}$). 

\subsection{Proofs of Main Results}
\subsubsection{Auxiliary Lemmas}

\begin{lemma}\label{app:lemma:second-der}
    $ h''(\frac{1}{2})\leq0$ and $\sigma''(0)=0$.
\end{lemma}
\begin{proof}
    Suppose the contrapositive $ h''(\frac{1}{2})>0$ holds. Then, there exists $\varepsilon>0$ such that $ h''(\sigma)>0$, $\forall \sigma\in[\frac{1}{2}-\varepsilon,\frac{1}{2}+\varepsilon]$ due to analyticity. This implies 
    \begin{equation*}
         h'(\frac{1}{2}+\varepsilon)= h'(\frac{1}{2}-\varepsilon)+\int_{\frac{1}{2}-\varepsilon}^{\frac{1}{2}+\varepsilon} h''(x)dx > 0
    \end{equation*}
    since $ h'(\frac{1}{2}-\varepsilon)>0$ by monotonicity. This contradicts with the assumption that $ h$ is strictly decreasing on $\sigma\in[\frac{1}{2},1]$. 
    Due to reflection symmetry, we obtain $\sigma''(z)=-\sigma''(-z)\implies \sigma''(z)+\sigma''(-z)=0$. Due to analyticity, we send $z$ to zero and obtain $2\sigma''(0)=0\implies \sigma''(0)=0$.
\end{proof}

\subsubsection{Proof of Proposition~\ref{prop:br-form}}
\begin{repproposition}{prop:br-form}    
    $\overline{a}(i,x) \in \mu_i-x+\argmax_{z\geq x-\mu_i}G(i,z)$.
\end{repproposition}
\begin{proof}
    Write $\tilde{G}(i,z):=G(i,z)+cz$. Then, the agent's utility function can be written as $u(i,x,a)=\tilde{G}(i,x+a-\mu_i)-ca + ri = \tilde{G}(i,x+a-\mu_i)-c(x+a-\mu_i) + c(x-\mu_i) + ir=G(i,x+a-\mu_i)+c(x-\mu_i) + ir$. Thus, the best-response action is obtained by $\overline{a}\in\argmax_{a\geq0} G(i,x+a-\mu_i)=\mu_i-x+\argmax_{z\geq x-\mu_i}G(i,z)$. 
\end{proof}

\subsubsection{Proof of Proposition~\ref{prop:boundary-pr}}
\begin{repproposition}{prop:boundary-pr}
    $p_i^+(\overline{\mu}_i)\geq p_i^-(\overline{\mu}_i)$, $\forall i\in[I-1]$.
\end{repproposition}
\begin{proof}
    It is sufficient to show that $\overline{w}_i> 0$, $\forall i\in[I-1]$. The proof is then automatically completed by the monotonicity of $\sigma(\alpha\cdot)$ and the fact that $\frac{p^+(z)}{p^-(z)}=\frac{\sigma(\alpha z)}{1-\sigma(\alpha z)}$. The rest of the proof is dedicated to showing $\overline{w}_i> 0$, $\forall i\in[I-1]$.
    
     When all local maximizers of $G(i,\cdot)$ are positive, the statement obviously holds. So, we focus on cases where non-positive local maximizers exist.
    Define $g_i(z):=\frac{\partial G(i,z)}{\partial z} + c$ and then the local maximizers correspond to roots of $g_i(z)=c$.  We consider the two cases of $G(i,z)$separately. When $i=1$, we have
    \begin{equation*}
        g_1(z) = [1- h(\sigma(\alpha z))- h'(\sigma(\alpha z))\sigma(\alpha z)]\times\underbrace{\alpha  r \sigma'(\alpha z)}_{>0}.
    \end{equation*}
    Let $f_1(\sigma)=1- h(\sigma)- h'(\sigma)\sigma$ and fix any $\sigma<\frac{1}{2}$ ($\iff z<0$). Notice 
    \begin{align*}
        f_1(1-\sigma) &= 1- h(1-\sigma)- h'(1-\sigma)(1-\sigma) \\
                    &\overset{(1)}{=} 1- h(\sigma) +  h'(\sigma)(1-\sigma) \\
                    &\overset{(2)}{>} 1 -  h(\sigma) -  h'(\sigma)\sigma = f_1(\sigma)
    \end{align*}
    where $(1)$ is due to symmetry of $ h$ around $\frac{1}{2}$ and $(2)$ is because $ h'(\sigma)>0$ when $\sigma<\frac{1}{2}$ (or equivalently, $z<0$). This then implies, with $\sigma'(z)=\sigma'(-z)>0$ due to inverse symmetry of $\sigma$, that $g_1(z)< g_1(-z)$ for $z<0$. Also, notice that $\lim_{|z|\to\infty}g_1(z)=0$. 
    
    Suppose $\overline{z}<0$ is a local maximizer of $G(i,\cdot)$. Then, we have $g_1(-\overline{z})>g_1(\overline{z})=c$ and there must exist a finite $\tilde{z}>-\overline{z}$ at which $g_1$ crosses the curve $y=c$ from above, according to the Intermediate Value Theorem. This means $\tilde{z}$ is a positive local maximizer of $G(1,\cdot)$, implying $\overline{w}_1>0$. 
    
    Suppose $z=0$ is a local maximizer of $G(i,\cdot)$. Differentiating $g_1(z)$ yields
    \begin{multline*}
        g_1'(z) = \left[-\alpha  h'(\sigma(\alpha z))\sigma'(\alpha z) -  h''(\sigma(\alpha z))\sigma(\alpha z) - \alpha  h'(\sigma(\alpha z))\sigma'(\alpha z)\right]\times \alpha   r \sigma'(\alpha z) \\ + [1- h(\sigma(\alpha z))- h'(\sigma(\alpha z))\sigma(\alpha z)]\times\alpha^2  r \sigma''(\alpha z) \\
        \implies g_1'(0) = -\frac{\alpha  r }{2} h''(\frac{1}{2})\sigma'(0) + \alpha^2  r (1- h(\frac{1}{2}))\sigma''(0)\geq 0
    \end{multline*}
    where the last inequality is due to \cref{app:lemma:second-der}. This is a contradiction because for $0$ to be a local maximizer, $g_1(z)$ must cross $y=c$ from above at $z=0$, which is equivalent to $g_1'(z)<0$.

    When $2\leq i\leq I-1$, we have $g_i(z)=[2(1- h(\sigma(\alpha z)))- h'(\sigma(\alpha z))(2\sigma(\alpha z) - 1)]\times\alpha \sigma'(\alpha z)$. It is not hard to see that $g_i(-z) =g_i(z)$. Then, if $\overline{z}<0$ is a local maximizer, $-\overline{z}>0$ is a local minimizer at which $g_i$ crosses the curve $\overline{z}=c$ from below. Due to analyticity, there exists an infinitesimal $\varepsilon>0$ such that $g_i(-\overline{z}+\varepsilon)>c$. Since $g_i(z)\to0$ as $z\to\infty$, there must exist a finite number $\tilde{z}>-\overline{z}+\varepsilon$ at which $g_i$ crosses the curve $y=c$ from above, according to the Intermediate Value Theorem. Therefore, $\tilde{z}$ is a strictly positive local maximizer, the largest of which must then satisfy $\overline{w}_i>0$. 
    Similarly, since $g'_i(0)=0$, $z=0$ cannot be a local maximizer of $G(i,\cdot)$. Thus, $\overline{w}_i>0$.
\end{proof}

\section{Supplementaries of \cref{sec:long-term}}\label{app:proof-long-term}

\subsection{Proofs of Main Results}
The main component of the proof of this section is to show that the discretized Markov process (see \cref{sec:incentive_improve}) has a unique stationary distribution that is independent of the initial attribute $x_0$ for either strategy. Then, all results follow from the analysis of the stationary distribution.
\subsubsection{Auxiliary Lemmas}
\begin{lemma}\label{app:lemma:benefit-gain}
    Let $p$ and $q$ be two finite distributions over $[I]$ such that $p_i=\alpha^{I-i} p_I$ and $q_i=\beta^{I-i} q_I$ for every $i\in[I]$ with $0<\alpha<\beta<1$. Then, $\sum_{i=1}^Iip_i-\sum_{i=1}^Iiq_i\leq \frac{\beta}{1-\beta} - \frac{\alpha}{1-\alpha}$.
\end{lemma}
\begin{proof}
    We can calculate the explicit expressions for the distributions $p$ and $q$ and it turns out that $p_i=\frac{1-\alpha}{1-\alpha^I}\alpha^{I-i}$ and $q_i=\frac{1-\beta}{1-\beta^I}\beta^{I-i}$ for $i\in[I]$. Notice that each sum can be split into sums of geometric series and its derivative, which yields $A=\frac{I}{1-\alpha^I}-\frac{\alpha}{1-\alpha}$ and $B=\frac{I}{1-\beta^I}-\frac{\beta}{1-\beta}$. So, the difference is
    \begin{equation*}
        A-B=N(\frac{1}{1-\alpha^I}-\frac{1}{1-\beta^I}) + \frac{\beta}{1-\beta} - \frac{\alpha}{1-\alpha}\leq \frac{\beta}{1-\beta} - \frac{\alpha}{1-\alpha}
    \end{equation*}
    since $\alpha<\beta$ implies the first term is negative.
\end{proof}

\begin{lemma}\label{app:lemma:u-cont}
    The optimal utility function $x\mapsto u(i,x,\overline{a}^g(i,x))$ is continuous for $g\in\{\textrm{NI},\textrm{NG}\}$.
\end{lemma}
\begin{proof}
    By \cref{prop:br-form}, the optimal instantaneous utility can be written as 
    \begin{equation*}
        u(i,x,\overline{a}(i,x)) = \max_{y\geq x-\mu_i}\; G(i,y) + c(x-\mu_i) + ir.
    \end{equation*}
    The last two terms are obviously continuous. For the firs term, define $F(x):=\max_{y\geq x-\mu_i}\; G(i,y)$. Let $x_n\downarrow x^*$ be a convergent and decreasing sequence. Then, $F(x_{n+1})\geq F(x_n)$ and this implies  $\lim_{n\to\infty}F(x_n)=\sup_nF(x_n)$. We next show that this equals $F(x^*)$. First notice that $F(x^*)\geq F(x_n)$ for all $n$ and thus $F(x^*)\geq \sup_nF(x_n)$. Suppose there exists $m$ such that $F(x_n)\leq m < F(x^*)$, $\forall n$. This implies $G(i,z)\leq m$, $\forall  z>x^*-\mu_i$. By continuity of $G(i,\cdot)$, we have $G(i,x^*-\mu_i)\leq m$, yielding $G(i,z)\leq m<\max_{y\geq x^*-\mu_i}G(i,y)$, $\forall z\geq x^*-\mu_i$ which is clearly a contradiction. Therefore, $F(x)$ is right continuous. 
    For $x_n\uparrow x^*$, we have $F(x_{n+1})\leq F(x_n)$. We apply similar argument as above to show that $F(x^*)=\inf_nF(x_n)=\lim_nF(x_n)$. This show that $F(x)$ is left continuous and hence complete the proof.
\end{proof}

\begin{lemma}\label{app:lemma:cheating-more-effort}
    $\overline{w}^\textrm{NI}_i>\overline{w}_i^\textrm{NG}$ for every $i\in[I]$. This would further imply $\sigma^\textrm{NI}>\sigma^\textrm{NG}$ where $\sigma^\textrm{NI}$ and $\sigma^\textrm{NG}$ are defined in \cref{thm:long-term-levels}.
\end{lemma}
\begin{proof}
    Observe
    \begin{equation*}
        \frac{\partial G^\textrm{NI}(i,z)}{\partial z}\Bigg\lvert_{z=\overline{w}_i^\textrm{NG}} = \underbrace{\frac{\partial G^\textrm{NG}(i,z)}{\partial z}\Bigg\lvert_{z=\overline{w}_i^\textrm{NG}}}_{=0} + (c^\textrm{NG}-c^\textrm{NI}) > 0.
    \end{equation*}
    Thus, the rightmost local maximizer of $G^\textrm{NI}(i,\cdot)$ must be larger than $\overline{w}_i^\textrm{NG}$ and completes the proof.
\end{proof}

\begin{lemma}\label{app:lemma:bar-larger}
    $-\underline{w}_i>\overline{w}_i$ for every $2\leq i\leq I-1$.
\end{lemma}
\begin{proof}
    By symmetry specified in \cref{asm:local-maximizer}, $G(i,\cdot)$ satisfies reflection symmetry around $z=0$, i.e., $G(i,z)=-G(i,-z)$, for $2\leq i\leq I-1$. We first show that $G(i,\overline{w}_2)>0$. Suppose the opposite holds. Then, recalling $\underline{w}_i:=\inf\{z:G(i,z)\leq G(i,\overline{w}_2)\}$ and given $G(i,0)=0\geq G(i,\overline{w}_2)$ due to reflection symmetry, we must have $\underline{w}_i\geq 0$. Then, we must have $G(i,\cdot)$ monotonically decreasing on $(-\infty,0)\subseteq(-\infty,\overline{w}_i)$ since otherwise $\overline{w}_i$ will not be the largest local maximizer, violating \cref{cond:level-cond}. This is clearly a contradiction with the reflection symmetry of $G(i,\cdot)$ around $z=0$.
    
    Using positivity of $G(i,\overline{w}_i)$, we have $G(i,-\overline{w}_i)=-G(i,\overline{w}_i) < G(i,\overline{w}_i)$. By definition of $\underline{w}_i$, this implies $-\overline{w}_i>\underline{w}_i\iff -\underline{w}_i>\overline{w}_i$.
\end{proof}

\begin{lemma}\label{app:lemma:compare-wbar}
    $\overline{w}_i\geq \overline{w}_I$ for every $i\leq I-1$.
\end{lemma}
\begin{proof}
    When $\overline{w}_I\leq 0$,  the statement obviously holds, since $\overline{w}_i>0$, $\forall i\leq I-1$ (see the proof of \cref{prop:boundary-pr}). In the following proof, we focus on the case when $\overline{w}_I>0$.
    Notice that $\overline{w}_I$ can be obtained by setting the first-order derivative of $G(I,\cdot)$ to zero:
    \begin{equation}\label{eq:wIbar}
        \frac{\partial G(I,\overline{w}_I)}{\partial z} = r\alpha\sigma'(\alpha\overline{w}_I)\left[1-h(\sigma_I)-h'(\sigma_I)(\sigma_I-1)\right] - c = 0,
    \end{equation}
    where we appreciated $\sigma_I:=\sigma(\alpha \overline{w}_I)$ for simplicity. Observe that for $2\leq i\leq I-1$,
    \begin{align*}
        \frac{\partial G(i,\overline{w}_I)}{\partial z} &= r\alpha\sigma'(\alpha\overline{w}_I)\left[2(1-h(\sigma_I))-h'(\sigma_I)(2\sigma_I-1)\right] - c \\
        &\overset{(a)}{=} r\alpha\sigma'(\alpha\overline{w}_I)\left[1-h(\sigma_I)-h'(\sigma_I)\sigma_I\right] \\
        &\overset{(b)}{\geq} 0
    \end{align*}
    where (a) is obtained by plugging in \cref{eq:wIbar} and (b) is due to monotonicity of $\sigma$ and that $h'(\sigma_I)<0$ as  $\sigma_I>\frac{1}{2}$. Due to analyticity, there the largest local maximizer of $G(i,\cdot)$ must be larger than $\overline{w}_I$. Similarly, for $i=1$, we have 
    \begin{align*}
        \frac{\partial G(1,\overline{w}_I)}{\partial z} &= r\alpha\sigma'(\alpha\overline{w}_I)\left[1-h(\sigma_I)-h'(\sigma_I)\sigma_I\right] - c \\
        & = -r\alpha\sigma'(\alpha\overline{w}_I)h'(\sigma_I) \geq 0.
    \end{align*}
    Thus, we also have $\overline{w}_1\geq \overline{w}_I$ and thereby completes the proof.
\end{proof}

\begin{lemma}\label{lemma:go-ergodicity}
    Let $f:[I]\times\mathbb{R}_{\geq0}\mapsto\mathbb{R}$ be bounded and continuous in the second argument. If the agent chooses the NI strategy, then, for any initial state $x_0\geq0$,
    \begin{equation*}
        \frac{1}{T}\sum_{t=1}^Tf(i_t,x_t)\to\sum_{i=1}^If(i,0)\pi_i,\;\;\textrm{as}\;\;T\to\infty
    \end{equation*}
    where $\vec\pi$ is the unique stationary distribution of $Q$ defined by
    \begin{equation*}
        Q=\begin{pmatrix}
            \vec{p}_1(\overline{\mu}_1^\textrm{NI}) \\
            \vdots \\
            \vec{p}_l(\overline{\mu}_l^\textrm{NI}) \\
            \vec{p}_{l+1}(0) \\
            \vdots \\
            \vec{p}_I(0)
        \end{pmatrix}
    \end{equation*}
    where $l:=\max\{i:\underline{\mu}_i^\textrm{NI}\leq 0\}$.
\end{lemma}
\begin{proof}
    Fix the arbitrary initial attribute $\vec{x}_0$. We omit the superscript NI for clarity of the proof.
    Consider the truncated attribute process $x^\varepsilon_t:=x_t\vee\varepsilon x_0$ (elementwise) for infinitesimal $\varepsilon>0$ (small enough such that $\varepsilon x_0<\mu_{l+1}+\underline{w}_{l+1}^\textrm{NI}$). 
    Using the discretization technique, the state process $s^\varepsilon_t=(i_t,x^\varepsilon_t)$ can be represented as a Markov chain with finite state space $\mathcal{S}^\varepsilon:=\{(i,x):i\in[I],\;x\in\{x_0,\gamma x_0,\dots,\gamma^{\lfloor\log_\gamma\varepsilon\rfloor}x_0,\varepsilon{x}_0\}\}$ (i.e., in NI, the only driving force of the attribute updates is the depreciation). 
    Denote $\mathcal{S}^\varepsilon=\mathcal{T}^\varepsilon\cup\mathcal{R}^\varepsilon$ where $\mathcal{T}^\varepsilon:=\{(i,{x}):i\in[I],\;{x}>\varepsilon{x}_0\}$ and $\mathcal{R}^\varepsilon:=[I]\times\{\varepsilon{x}_0\}$. 
    $\mathcal{T}^\varepsilon$ is the set of states where the truncated attribute sits strictly above $\varepsilon x_0$. In particular, from any state in $\mathcal{T}^\varepsilon$, there is a strictly positive probability of eventually entering $\mathcal{R}^\varepsilon$, because repeated multiplication by $\gamma$ will force the attribute below $\varepsilon x_0$. Thus, $\mathcal{T}^\varepsilon$ is a transient class.

    Due to depreciation, once the chain enters $\mathcal{R}^\varepsilon$, it cannot exit back to $\mathcal{T}^\varepsilon$. This creates a sub-chain that is restricted to $\mathcal{R}^\varepsilon$, and is a finite Markov chain with exactly $I$ states with strictly positive transition probabilities among them, due to the transition probabilities.

    Let $V$ be the first time when ${x}_t$ in the original process drops below $\varepsilon {x}_0$. Specifically, $V=\lfloor\log_{\gamma}\frac{\varepsilon}{x_0}\rfloor+1$. Then, for any $t\geq 0$, ${s}^\varepsilon_{t+V}\in\mathcal{R}^\varepsilon$ where ${x}_{t+V}^\varepsilon\equiv \varepsilon x_0$ and hence the process $({s}_{t+v}^\varepsilon)_{t\geq0}$ is equivalent to a Markov chain only on the levels $[I]$, with transition matrix $Q^\varepsilon$ such that
    \begin{equation*}
        Q^\varepsilon = \begin{pmatrix}
            \vec{p}_1(\overline{\mu}_1^\textrm{NI}) \\
            \vdots \\
            \vec{p}_l(\overline{\mu}_l^\textrm{NI}) \\
            \vec{p}_{l+1}(\varepsilon x_0) \\
            \vdots \\
            \vec{p}_{I}(\varepsilon x_0)
        \end{pmatrix}
    \end{equation*}
    {It is clear that, $Q^\varepsilon$ is finite, irreducible, and aperiodic (i.e., $Q_{ii}^\varepsilon>0$ for all $i\in[I]$) for all $\varepsilon\geq 0$ and $Q^0=Q$ by construction.} Thus, $\mathcal{R}^\varepsilon$ is a recurrent class and, according to the Perron-Frobenius theorem, there exists a unique stationary distribution $\vec\pi^{\varepsilon}$ (row vector) over the levels that satisfies $\vec\pi^{\varepsilon} Q^\varepsilon = \vec\pi^{\varepsilon}$.
    By applying the ergodic theorem, it is clear that 
    \begin{equation}\label{eq:app:proof-conv}
        \lim_{T\to\infty}\frac{1}{T}\sum_{t=1}^Tf(i_t,{x}_t^\varepsilon)= \lim_{T\to\infty}
        \frac{1}{T-V}\sum_{t=V}^Tf(i_t,\varepsilon{x}_0)=\sum_{i=1}^If(i,\varepsilon{x_0})\pi_i^\varepsilon.
    \end{equation}

    We next show that $\vec\pi^\varepsilon$ is continuous in $\varepsilon$. By construction, $Q^\varepsilon$ is continuous in $\varepsilon$. Let $\varepsilon_n\to\varepsilon$ and $\Delta:=\{x\in\mathbb{R}_{\geq0}^{I}:\sum_{i=1}^Ix_i=1\}$ be the simplex over $[I]$. Then, since $\Delta$ is compact, there exists a subsequence $\{\varepsilon_{n_k}\}_k$ such that $\vec\pi^{\varepsilon_{n_k}}\to\vec\omega\in \Delta$. Besides, due to the continuity of $Q^\varepsilon$ in $\varepsilon$, we have
    \begin{equation*}
        \vec\omega=\lim_{k\to\infty}\vec\pi^{\varepsilon_{n_k}}=\lim_{k\to\infty}\vec\pi^{\varepsilon_{n_k}}Q^{\varepsilon_{n_k}}=\vec\omega Q^\varepsilon.
    \end{equation*}
    So, $\vec\omega$ is a left eigenvector of $Q^\varepsilon$ corresponding to the eigenvalue 1. Since $Q^\varepsilon$ is irreducible and aperiodic, the uniqueness of the Perron-Frobenius theorem implies that $\vec\omega=\vec\pi^\varepsilon$. Notice that we can use the same methodology to find a convergent subsequence with the same limit $\vec\pi^\varepsilon$ for every subsequence of $\{\vec\pi^{\varepsilon_n}\}_n$. Thus, we have $\vec\pi^{\varepsilon_n}\overset{n\to\infty}{\to}\vec\pi^\varepsilon$. Since the sequence $\{\varepsilon_n\}_n$ is arbitrary over $[0,\infty)$, the stationary distribution $\vec\pi^\varepsilon$ is continuous in $\varepsilon\in[0,\infty)$. 

    Since $f$ is also continuous in the second argument, sending $\varepsilon$ to 0 in \cref{eq:app:proof-conv} completes the proof.
\end{proof}

\subsubsection{Analysis of the Stationary Distributions}
The following results \cref{app:thm:stationary,app:thm:ergodic} present the most important components in this proof by establishing the stationary distribution for either strategy and relating it to the properties of interest.
\begin{theorem}[Existence of Stationary Distribution]\label{app:thm:stationary}
    Let $\{\mu_i\}_i$ satisfy \cref{cond:level-cond}. Then,
    for each $g\in\{\textrm{NI},\textrm{NG}\}$, there exists a unique stationary distribution $\pi^g$ for the corresponding state process $\{s_t^g\}_t$, independent of the initial attribute $x_0$. Moreover, $\pi^\textrm{NI}(i,0)=\pi^0_i$, $\forall i\in[I]$, where ${\vec{\pi}^0}$ is the unique stationary distribution of $Q$ on the levels where
    \begin{equation*}
        Q=\begin{pmatrix}
            \vec{p}_1(\mu_1+\overline{w}_1^\textrm{NI}) \\
            \vdots \\
            \vec{p}_l(\mu_l+\overline{w}_l^\textrm{NI}) \\
            \vec{p}_{l+1}(0) \\
            \vdots \\
            \vec{p}_I(0)
        \end{pmatrix}
    \end{equation*}
    where $l:=\max\{i:\underline{\mu}_i^\textrm{NI}\leq 0\}$ defined in \cref{thm:long-term-levels} and $\pi^\textrm{NG}(i,x)=0$ for all $x<\gamma\min_{i\in[I]}\overline{\mu}_i^\textrm{NG}$.
\end{theorem}
\begin{proof}
    For the NI strategy, suppose the initial distribution over the states is $\pi^\textrm{NI}$ where $\pi^\textrm{NI}(i,0)=\pi^0_i$ for every $i\in[I]$. Then, since the agent's attribute stays forever at $0$ (i.e., $\gamma\times0=0$), the transitions only occur among the levels with transition matrix $P(0)$. Since $\vec\pi^0$ is the unique stationary distribution (over the levels) of $P(0)$, the state distribution after one-step transition remains the same as $\pi^\textrm{NI}$. Therefore, $\pi^\textrm{NI}$ is a stationary distribution over the states.

    Denote ${\pi}^{(t)}$ as the state distribution at step $t$ and ${\pi}_x^{(t)}:=\sum_{i\in[I]}\pi^{(t)}(i,x)$. Let $\tilde{\pi}^\textrm{NI}$ be some state distribution such that $\tilde{\pi}^\textrm{NI}(i',x')>0$ for some $i'\in[I]$ and $x'>0$ and $\tilde{\pi}^\textrm{NI}_x:=\sum_{i\in[I]}\tilde{\pi}^\textrm{NI}(i,x)$. Consider the state process with initial distribution $\pi^{(0)}=\tilde{\pi}^\textrm{NI}$. Then, it is obvious that $\pi^{(t)}_x=\pi^{(0)}_{x/\gamma^t}=\tilde{\pi}^{\textrm{NI}}_{x/\gamma^t}$ for all $x$ and $t\geq 0$. If $\tilde{\pi}^\textrm{NI}$ is a stationary distribution, then we must have
    \begin{equation*}
        \pi^{(t)}=\pi^{(0)},\;\forall t\geq 1 \iff \tilde{\pi}^\textrm{NI}_{\gamma^tx} = \tilde{\pi}^\textrm{NI}_{x},\;\forall x,\;\forall t\geq 1.
    \end{equation*}
    This implies $1\geq \sum_{t=1}^\infty\tilde{\pi}^\textrm{NI}_{\gamma^tx'}=\sum_{t=1}^\infty\tilde{\pi}^\textrm{NI}_{x'}>0$, which is clearly impossible. Thus, the stationary distribution for the NI strategy can only support zero attributes (i.e., $x=0$) and the only distribution satisfying the balance equation under $x=0$ is $\pi^\textrm{NI}$ since $P(0)$ is irreducible and aperiodic by construction. Therefore, $\pi^\textrm{NI}$ is the unique stationary distribution.

    For the NG strategy, notice that the discretized state space can be partitioned into $\mathcal{X}(x_0)=A(x_0)\cup B$ where $B=\{\gamma^t\overline{\mu}_i^\textrm{NG}:i\in[I],t\geq 1\}$ and $A(x_0)=\{\gamma^tx_0:t\geq 0\}\cap B^c$. Notice that whenever the attribute process enters $B$ with attribute $x_t\in B$, the next-step attribute is either $\gamma x_t$ or $\gamma\overline{\mu}_i^\textrm{NG}$ for some $i$, both belonging to the set $B$. Besides, since the state process starts with $i_0=1$, when $x_0\leq \overline{\mu}_1^\textrm{NI}$, the process will enter $B$ at time $t=1$ by \Cref{cond:level-cond}-(a). When $x_0>\overline{\mu}_1^\textrm{NI}$, the state process remains in $A(x_0)$ until $\tau:=\min\{t\geq0:\underline{\mu}_{i_t}^\textrm{NG}\leq x_t<\overline{\mu}_{i_t}^\textrm{NG}\}$ --- the first time the state process enters some incentivizable interval and that $x_{\tau+1}\in B$. Notice
    \begin{equation*}
        \mathbb{E}_{s_0}[\tau] = \sum_{i=1}^I\mathbb{E}_{s_0}[\tau\mathbf{1}\{i_\tau=i\}]\leq \sum_{i=1}^I\max\left\{\log_\gamma\left(\frac{\overline{\mu}_i^\textrm{NG}}{x_0}\right),0\right\}\mathbb{P}(i_t=i)<\infty.
    \end{equation*}
    So, the set $A(x_0)$ is transient. Thus, the stationary distribution for the NG strategy, if exists, must be independent of $x_0$.

    Finally, we show that $C:=\{(i,x):x\geq \gamma\min\{\overline{\mu}_i^\textrm{NG},\overline{\mu}_{i-1}^\textrm{NG}\},\;x\in B\}$ with the default case $\overline{\mu}_{0}^\textrm{NG}=\infty$ is a (positive) recurrent class starting with $i_0=1$. 
    Due to \Cref{cond:level-cond}-(a) and $i_0=1$, we must have $x_1\in {C}$. If $s_t\in C\iff x_t\geq \gamma\min\{\overline{\mu}_{i_t}^\textrm{NG},\overline{\mu}_{i_t-1}^\textrm{NG}\}\geq \underline{\mu}_{i_t}^\textrm{NG}$ where the last inequality is due to \Cref{cond:level-cond}-(c), then the attribute dynamics implies that $x_{t+1}\geq\gamma\overline{\mu}_{i_t}^\textrm{NG}$. Thus, for $i_{t+1}\in\{i_t,i_t+1\}$, we have $s_{t+1}\in C$. For $i_{t+1}=i_t-1$, first notice that the post-response attribute at the end of this time step satisfies $\tilde{x}_t\geq \overline{\mu}_{i_t}^\textrm{NG}$ due to improvement actions. Then, observe $x_{t+1}=\gamma\tilde{x}_t\geq \gamma\overline{\mu}_{i_t}^\textrm{NG}>\gamma\overline{\mu}_{i_t-1}^\textrm{NG}=\gamma\overline{\mu}_{i_{t+1}}^\textrm{NG}$ where the last inequality is due to \Cref{cond:level-cond}-(b). This implies $s_{t+1}\in C$ for $i_{t+1}=i_t-1$. Therefore, we have $s_t\in C$, $\forall t\geq 1$, which implies that $C$ is a closed set. It is also easy to see that $C$ is finite and irreducible, implying that all states in $C$ are (positive) recurrent. Thus, there exists a unique stationary distribution supported by $C$.
\end{proof}

For ease of presentation, we denote $\pi^\textrm{NI}_i:=\pi^\textrm{NI}(i,0)$ and $\pi_i^\textrm{NG}(x):=\pi^\textrm{NG}(i,x)$ as the stationary distributions for the respective strategies. For the NG strategy, we denote $\pi_i^\textrm{NG}(B):=\sum_{x\in B}\pi_i^\textrm{NG}(x)$ as the proportion of agents at level $i$ with attributes in $B\subseteq\mathbb{R}_{\geq0}$, and  $\pi^\textrm{NG}_i:=\sum_{x}\pi^\textrm{NG}_i(x)$ as the marginal stationary distribution of level $i$.

\begin{theorem}[Ergodicity]\label{app:thm:ergodic}
    Let $\{\mu_i\}_i$ satisfy \cref{cond:level-cond}. For each $g\in\{\textrm{NI},\textrm{NG}\}$, we have 
    \begin{equation*}
        \hat{q}_i^g=\pi_i^g,\quad \hat{x}^g=\sum_{i,x}x\pi^g_i(x),\quad \textrm{and}\quad\hat{u}^g=\sum_{i,x}u(i,x,\overline{a}^g(i,x))\pi^g_i(x).
    \end{equation*}
\end{theorem}
\begin{proof}
    The statements for $g=\textrm{NG}$ follows readily from the positive recurrent Markov chain (see the proof of \cref{app:thm:stationary}). For $g=\textrm{NI}$, apply \cref{lemma:go-ergodicity} with the functions $f(i,x)=\mathbf{1}\{i=j\}$, $\forall j\in[I]$, $f(i,x)=x$, and $f(i,x)=u(i,x,\overline{a}^g(i,x))$ respectively. The first two functions are obviously continuous in $x$, and the continuity of the optimal utility function is provided in \cref{app:lemma:u-cont}.
\end{proof}

The next result analyzes the stationary distribution of either strategy in detail, revealing the detailed-balance relation embedded in both cases.
\begin{theorem}[Detailed Balance]\label{app:thm:detailed-balance}
    Let $\{\mu_i\}_i$ satisfy \cref{cond:level-cond}. Then, both $\pi^\textrm{NG}$ and $\pi^\textrm{NI}$ satisfy detailed balance with transition probabilities: $\tilde{p}_i^+\pi_i^{\textrm{NG}} = \tilde{p}_{i+1}^-\pi_{i+1}^\textrm{NG},\;\forall 1\leq i\leq I-1$, and  
    \begin{align*}
            &p_i^+(\overline{\mu}_i^\textrm{NI})\pi^\textrm{NI}_i = p_{i+1}^-(\overline{\mu}_{i+1}^\textrm{NI})\pi_{i+1}^\textrm{NI} && 1\leq i \leq l-1 \\
            &p_i^+(\overline{\mu}_i^\textrm{NI})\pi_i^\textrm{NI} = p_{i+1}^-(0)\pi_{i+1}^\textrm{NI} && i = l \\
            &p_i^+(0)\pi_i^\textrm{NI} = p_{i+1}^+(0)\pi_i^\textrm{NI} && l+1\leq i \leq I
    \end{align*}
    where $l:=\max\{i:\underline{\mu}_i^\textrm{NI}\leq 0\}$ as defined in \cref{thm:long-term-levels} and, for each $d\in\{+,0,-\}$, 
    \begin{equation*}
        \tilde{p}^d_i:=\left(1-\sum_{x\geq\overline{\mu}_i^\textrm{NG}}\frac{\pi_i^\textrm{NG}(x)}{\pi_i^\textrm{NG}}\right)p_i^d(\overline{\mu}_i^\textrm{NG})+\sum_{x\geq\overline{\mu}_i^\textrm{NG}}\frac{\pi_i^\textrm{NG}(x)}{\pi_i^\textrm{NG}}p_i^d(x).
    \end{equation*}
\end{theorem}
\begin{proof}
    We write $\vec{\pi}^g:=[\pi_i^g]_{i\in[I]}$ to denote the probability vector over the levels for each $g\in\{\textrm{NI},\textrm{NG}\}$.
    For the NI strategy, we have $\vec{\pi}^\textrm{NI}Q=\vec{\pi}^\textrm{NI}$ from \cref{app:thm:stationary}. Then, we have
    \begin{equation*}
        \pi_1^\textrm{NI} = (1-Q_{12})\pi_1^\textrm{NI} + Q_{21}\pi_2^\textrm{NI} \iff Q_{12}\pi_1^\textrm{NI}=Q_{21}\pi_2^\textrm{NI}.
    \end{equation*}
    Assume $Q_{i,i+1}\pi_i^\textrm{NI}=Q_{i+1,i}\pi_{i+1}^\textrm{NI}$ holds for $i< I-1$. Then, 
    \begin{multline*}
            \pi_{i+1}^\textrm{NI} = Q_{i,i+1}\pi_i^\textrm{NI} + Q_{i+1,i+1}\pi_{i+1}^\textrm{NI} + Q_{i+2,i+1}\pi_{i+2} \\
            \implies (Q_{i+1,i} + Q_{i+1,i+2})\pi_{i+1}^\textrm{NI} = Q_{i,i+1}\pi_i^\textrm{NI} + Q_{i+2,i+1}\pi_{i+2} \\ \implies Q_{i+1,i+2}\pi_{i+1}^\textrm{NI} = Q_{i+2,i+1}\pi_{i+2}^\textrm{NI}.
    \end{multline*}
    Substituting the expressions for $Q$ recovers the detailed balance equations for the NI strategy.

    For the NG strategy, we know from the proof of \cref{app:thm:stationary} that $\pi^\textrm{NG}$ is supported by $C:=\{(i,x):x\geq \gamma\min\{\overline{\mu}_i^\textrm{NG},\overline{\mu}_{i-1}^\textrm{NG}\},\;x\in B\}$ with default $\overline{\mu}^\textrm{NG}_0=\infty$ and $B=\{\gamma^t\overline{\mu}_i^\textrm{NG}:i\in[I],t\geq 1\}$. This, together with \Cref{cond:level-cond}-(c), implies that whenever an NG agent from the stationary distribution enters level $i$ with attribute $x\leq \overline{\mu}_i^\textrm{NG}$, it's best response is to increase its attribute to $\overline{\mu}_i^\textrm{NG}$ (i.e., $x$ can never be less than $\underline{\mu}_i^\textrm{NG}$ in the stationary distribution). Then, we obtain, for $i=1$,
    \begin{multline*}
        \pi_1^\textrm{NG} = \left(\pi_1^\textrm{NG} - \sum_{x\geq \overline{\mu}_{1}^\textrm{NG}}\pi_{1}^\textrm{NG}(x)\right)
        \left(1-p_{1}^+(\overline{\mu}_{1}^\textrm{NG})\right) + \sum_{x\geq \overline{\mu}_{1}^\textrm{NG}}\pi_{1}^\textrm{NG}(x)\left(1- p_{1}^+(x)\right)  \\
        + \left(\pi_{2}^\textrm{NG} - \sum_{x\geq \overline{\mu}_{2}^\textrm{NG}}\pi_{2}^\textrm{NG}(x)\right)
        p_{2}^-(\overline{\mu}_{2}^\textrm{NG}) + \sum_{x\geq \overline{\mu}_{2}^\textrm{NG}}\pi_{2}^\textrm{NG}(x)p_{2}^-(x)\\
    \end{multline*}
    that implies
    \begin{multline*}
        \left(\pi_1^\textrm{NG} - \sum_{x\geq \overline{\mu}_{1}^\textrm{NG}}\pi_{1}^\textrm{NG}(x)\right)
        p_{1}^+(\overline{\mu}_{1}^\textrm{NG}) + \sum_{x\geq \overline{\mu}_{1}^\textrm{NG}}\pi_{1}^\textrm{NG}(x)p_{1}^+(x) \\ = \left(\pi_{2}^\textrm{NG} - \sum_{x\geq \overline{\mu}_{2}^\textrm{NG}}\pi_{2}^\textrm{NG}(x)\right)
        p_{2}^-(\overline{\mu}_{2}^\textrm{NG}) + \sum_{x\geq \overline{\mu}_{2}^\textrm{NG}}\pi_{2}^\textrm{NG}(x)p_{2}^-(x)
    \end{multline*}
    which is equivalent to $\tilde{p}_1^+\pi_1^\textrm{NG} = \tilde{p}_2^-\pi_2^\textrm{NG}$. Now, suppose $\tilde{p}_{i-1}^+\pi_{i-1}^\textrm{NG} = \tilde{p}_i^-\pi_{i}^\textrm{NG}$ holds for $i\geq 2$. Then, the balance equation for $\pi_i^\textrm{NG}$ implies
    \begin{align*}
        \pi_i^\textrm{NG} = &\left(\pi_{i-1}^\textrm{NG} - \sum_{x\geq \overline{\mu}_{i-1}^\textrm{NG}}\pi_{i-1}^\textrm{NG}(x)\right)
        p_{i-1}^+(\overline{\mu}_{i-1}^\textrm{NG}) + \sum_{x\geq \overline{\mu}_{i-1}^\textrm{NG}}\pi_{i-1}^\textrm{NG}(x)p_{i-1}^+(x)   \\
        &\qquad + \left(\pi_i^\textrm{NG} - \sum_{x\geq \overline{\mu}_{i}^\textrm{NG}}\pi_{i}^\textrm{NG}(x)\right)
        p_{i}^0(\overline{\mu}_{i}^\textrm{NG}) + \sum_{x\geq \overline{\mu}_{i}^\textrm{NG}}\pi_{i}^\textrm{NG}(x)p_{i}^0(x)  \\
        &\qquad + \left(\pi_{i+1}^\textrm{NG} - \sum_{x\geq \overline{\mu}_{i+1}^\textrm{NG}}\pi_{i+1}^\textrm{NG}(x)\right)
        p_{i+1}^-(\overline{\mu}_{i+1}^\textrm{NG}) + \sum_{x\geq \overline{\mu}_{i+1}^\textrm{NG}}\pi_{i+1}^\textrm{NG}(x)p_{i+1}^-(x) \\
        \iff&\left(\pi_i^\textrm{NG} - \sum_{x\geq \overline{\mu}_{i}^\textrm{NG}}\pi_{i}^\textrm{NG}(x)\right)\left(p_i^+(\overline{\mu}_i^\textrm{NG}) + p_i^-(\overline{\mu}_i^\textrm{NG})\right) + \sum_{x\geq \overline{\mu}_{i}^\textrm{NG}}\pi_{i}^\textrm{NG}(x)(p_{i}^+(x)  + p_{i}^-(x)) \\
        &\qquad =\left(\pi_{i-1}^\textrm{NG} - \sum_{x\geq \overline{\mu}_{i-1}^\textrm{NG}}\pi_{i-1}^\textrm{NG}(x)\right)
        p_{i-1}^+(\overline{\mu}_{i-1}^\textrm{NG}) + \sum_{x\geq \overline{\mu}_{i-1}^\textrm{NG}}\pi_{i-1}^\textrm{NG}(x)p_{i-1}^+(x)  \\
        &\qquad + \left(\pi_{i+1}^\textrm{NG} - \sum_{x\geq \overline{\mu}_{i+1}^\textrm{NG}}\pi_{i+1}^\textrm{NG}(x)\right)
        p_{i+1}^-(\overline{\mu}_{i+1}^\textrm{NG}) + \sum_{x\geq \overline{\mu}_{i+1}^\textrm{NG}}\pi_{i+1}^\textrm{NG}(x)p_{i+1}^-(x) \\
        \iff & \left(\pi_i^\textrm{NG} - \sum_{x\geq \overline{\mu}_{i}^\textrm{NG}}\pi_{i}^\textrm{NG}(x)\right)
        p_{i}^+(\overline{\mu}_{i}^\textrm{NG}) + \sum_{x\geq \overline{\mu}_{i}^\textrm{NG}}\pi_{i}^\textrm{NG}(x)p_{i}^+(x) \\ 
        &\qquad = \left(\pi_{i+1}^\textrm{NG} - \sum_{x\geq \overline{\mu}_{i+1}^\textrm{NG}}\pi_{i+1}^\textrm{NG}(x)\right)
        p_{i+1}^-(\overline{\mu}_{i+1}^\textrm{NG}) + \sum_{x\geq \overline{\mu}_{i+1}^\textrm{NG}}\pi_{i+1}^\textrm{NG}(x)p_{i+1}^-(x)
    \end{align*}
    where the last step applies the induction hypothesis. By rearranging terms, we obtain $\tilde{p}_{i}^+\pi_i^\textrm{NG}=\tilde{p}_{i+1}^-\pi_{i+1}^\textrm{NG}$ which completes the proof.
\end{proof}

Using the detailed balance equation in \cref{app:thm:detailed-balance}, we can derive an explicit formula for the long-term utility of either strategy using the stationary distributions.
\begin{theorem}[Long-term Utility Formula]\label{app:thm:long-term-utility}
    For $\{\mu_i\}_i$ satisfying \cref{cond:level-cond}, we have
    \begin{equation*}
        \hat{u}^\textrm{NI} =   r \sum_{i=1}^Ii\hat{q}_{i}^\textrm{NI}-c^\textrm{NI}\sum_{i=1}^l\overline{\mu}^\textrm{NI}_i\hat{q}^\textrm{NI}_i
    \end{equation*}
    where $l=\max\{i:\underline{\mu}^\textrm{NI}_i\leq 0\}$ as defined in \cref{thm:long-term-levels}; and
    \begin{equation*}
        \hat{u}^\textrm{NG} =    r \sum_{i=1}^Ii\hat{q}_{i}^\textrm{NG}-c^\textrm{NG}\sum_{i=1}^I\sum_{x\in B_i}(\overline{\mu}^\textrm{NG}_i-x)\pi^\textrm{NG}_i(x)
    \end{equation*}
    where $B_i=[\underline{\mu}^\textrm{NG}_i, \overline{\mu}^\textrm{NG}_i)\cap\{\gamma^t\overline{\mu}^\textrm{NG}_i:i\in[I],t\geq 1\}$.
\end{theorem} 
\begin{proof}
    It is easy to see that $\hat{x}^\textrm{NI}=0$ which, together with \cref{app:thm:ergodic}, implies that $\hat{q}_i^\textrm{NI}=\pi_i^\textrm{NI}$.
    Furthermore, \cref{app:thm:ergodic} implies that 
    \begin{align*}
        \hat{u}^\textrm{NI} = &\left(r + rp_1^+(\overline{\mu}_1^\textrm{NI})-c^\textrm{NI}\overline{\mu}_1^\textrm{NI}\right)\hat{q}_1^\textrm{NI} + \sum_{i=2}^l\left(ri + rp_i^+(\overline{\mu}_i^\textrm{NI})-r p_i^-(\overline{\mu}_i^\textrm{NI})-c^\textrm{NI}\overline{\mu}_i^\textrm{NI}\right)\hat{q}_i^\textrm{NI} \\ 
        &\qquad + \sum_{i=l+1}^{I-1}\left(ri + rp_i^+(0) - rp_i^-(0)\right)\hat{q}_i^\textrm{NI} +(rI- rp_I^-(0))\hat{q}_{I}^\textrm{NI} \\
        = &r\sum_{i=1}^{l-1}\left(p_i^+(\overline{\mu}_i^\textrm{NI})\hat{q}_i^\textrm{NI} - p_{i+1}^-(\overline{\mu}_{i+1}^\textrm{NI})\hat{q}_{i+1}^\textrm{NI}\right) + rp_l^+(\overline{\mu}_l^\textrm{NI})\hat{q}_l^\textrm{NI} - rp_{l+1}^-(0)\hat{q}_{l+1}^\textrm{NI} \\
        &\qquad + r\sum_{i=l+1}^{I-1}\left(p_i^+(0)\hat{q}_i^\textrm{NI} - p_{i+1}^-(0)\hat{q}_{i+1}^\textrm{NI}\right) +r\sum_{i=1}^Ii\hat{q}_i^\textrm{NI} - c^\textrm{NI}\sum_{i=1}^l\overline{\mu}_i^\textrm{NI}\hat{q}_i^\textrm{NI} \\
        = &r\sum_{i=1}^Ii\hat{q}_i^\textrm{NI}- c^\textrm{NI}\sum_{i=1}^l\overline{\mu}_i^\textrm{NI}\hat{q}_i^\textrm{NI}
    \end{align*} 
    where the last equality applies \cref{app:thm:detailed-balance}. For the NG strategy, 
    \begin{align}
        \hat{u}^\textrm{NG} = &\sum_{x\in B_1}\left(r + rp_1^+(\overline{\mu}_1^\textrm{NG})-c^\textrm{NG}(\overline{\mu}_1^\textrm{NG}-x)\right)\pi_1^\textrm{NG}(x) + \sum_{x> \overline{\mu}_1^\textrm{NG}}(r+rp_1^+(x))\pi_1^\textrm{NG}(x)\nonumber \\
        &\qquad + \sum_{i=2}^{I-1}\sum_{x\in B_i}\left(ri + rp_i^+(\overline{\mu}_i^\textrm{NG})-rp_i^-(\overline{\mu}_i^\textrm{NG})-c^\textrm{NG}(\overline{\mu}_i^\textrm{NG}-x)\right)\pi_i^\textrm{NG}(x)\nonumber \\
        &\qquad + \sum_{i=2}^{I-1}\sum_{x> \overline{\mu}_i^\textrm{NG}}(ri + rp_i^+(x)-rp_i^-(x))\pi_i^\textrm{NG}(x)\nonumber \\ 
        &\qquad + \sum_{x\in B_I}\left(rI - rp_I^-(\overline{\mu}_I^\textrm{NG})-c^\textrm{NG}(\overline{\mu}_I^\textrm{NG}-x)\right)\pi_I^\textrm{NG}(x) + \sum_{x>\overline{\mu}_I^\textrm{NG}}(rI-rp_I^-(x))\pi_I^{\textrm{NG}}\nonumber \\
        &=rp_1^+(\overline{\mu}_1^\textrm{NG})\pi_1^\textrm{NG}(B_1) + \sum_{x>\overline{\mu}_1^\textrm{NG}}rp_1^+(x)\pi_1^\textrm{NG}(x) + r\sum_{i=2}^{I-1}\left(p_i^+(\overline{\mu}_i^\textrm{NG}) - p_i^-(\overline{\mu}_i^\textrm{NG})\right)\pi_i^\textrm{NG}(B_i) \nonumber \\
        &\qquad +r \sum_{i=2}^{I-1}\sum_{x>\overline{\mu}_i^\textrm{NG}}\left(p_i^+(x) - p_i^-(x)\right)\pi_i^\textrm{NG}(x) - rp_I^-(\overline{\mu}_I^\textrm{NG})\pi_I^\textrm{NG}(B_I)- r\sum_{x>\overline{\mu}_I^\textrm{NG}}p_I^-(x)\pi_I^{\textrm{NG}} \nonumber \\
        & \qquad + \sum_{i=1}^Iri(\pi_i^\textrm{NG}(B_i)+\sum_{x>\overline{\mu}_i^\textrm{NG}}\pi_i^\textrm{NG}(x)) - c^\textrm{NG}\sum_{i=1}^I\sum_{x\in B_i}(\overline{\mu}_i^\textrm{NG}-x)\pi_i^\textrm{NG}(x). \nonumber
    \end{align}
    According to the proof of \cref{app:thm:stationary}, in the support of the stationary distribution, agent attribute at level $i$ must satisfy $x\geq \gamma\min\{\overline{\mu}_i^\textrm{NG},\overline{\mu}_{i-1}^\textrm{NG}\}$. Then, by \cref{cond:level-cond}-(c), we have $x>\overline{\mu}_i^\textrm{NG}$.  This means 
    \begin{equation*}
        \pi_i^\textrm{NG}(B_i) = \sum_{x\leq \overline{\mu}_i^\textrm{NG}}\pi_i^\textrm{NG}(x) = \pi_i^\textrm{NG} - \sum_{x> \overline{\mu}_i^\textrm{NG}}\pi_i^\textrm{NG}(x) = \hat{q}_i^\textrm{NG} - \sum_{x> \overline{\mu}_i^\textrm{NG}}\pi_i^\textrm{NG}(x).
    \end{equation*}
    Substituting $\tilde{p}_i^d$ for $d\in\{+,0,-\}$ from \cref{app:thm:detailed-balance} into the above equation, we then obtain 
    \begin{align*}
        \hat{u}^\textrm{NG} &= r\sum_{i=1}^{I-1}(\tilde{p}_i^+\pi_i^\textrm{NG}-\tilde{p}_{i+1}^-\pi_{i+1}^\textrm{NG}) + r\sum_{i=1}^Ii\hat{q}_i^\textrm{NG} - c^\textrm{NG}\sum_{i=1}^I\sum_{x\in B_i}(\overline{\mu}_i^\textrm{NG}-x)\pi_i^\textrm{NG}(x) \\
        &=r\sum_{i=1}^Ii\hat{q}_i^\textrm{NG}- c^\textrm{NG}\sum_{i=1}^I\sum_{x\in B_i}(\overline{\mu}_i^\textrm{NG}-x)\pi_i^\textrm{NG}(x)
    \end{align*}
    and thus complete the proof.
\end{proof}

\subsubsection{Proof of Theorem~\ref{thm:long-term-levels}}
We present and show a detailed version of Theorem~\ref{thm:long-term-levels} as follows.
\begin{reptheorem}{thm:long-term-levels}
    Let $\sigma^\textrm{NI}:=\sigma(\alpha\overline{w}_2^\textrm{NI})$ and $\sigma^\textrm{NG}:=\sigma(\alpha\overline{w}_2^\textrm{NG})$. Then,
    for $\{\mu_i\}_i$ satisfying \cref{cond:level-cond}, $\hat{q}_i^\textrm{NI}= \Omega((\frac{1-\sigma^\textrm{NI}}{\sigma^\textrm{NI}})^{l-i})$ for $i\leq l$ and $\hat{q}_i^\textrm{NI}= o((\frac{1-\sigma^\textrm{NI}}{\sigma^\textrm{NI}})^{i-l})$ for $i> l$ where $l:=\max\{j\in[I]:\underline{\mu}_j^\textrm{NI}\leq0\}$; and $\hat{q}_i^\textrm{NG}= O((\frac{1-\sigma^\textrm{NG}}{\sigma^\textrm{NG}})^{I-i})$ for all $i\in[I]$.
\end{reptheorem}
\begin{proof}
For the NI strategy, by incorporating the classifier expression into \cref{app:thm:detailed-balance}, we see that $\hat{q}_i^\textrm{NI}= \left(\frac{1-\sigma^\textrm{NI}}{\sigma^\textrm{NI}}\right)^{l-i}\hat{q}_l^\textrm{NI}$ for $2\leq i\leq l$, $\hat{q}_1^\textrm{NI}= \frac{p_2^-(\overline{\mu}_i^\textrm{NI})}{p_1^+(\overline{\mu_i^\textrm{NI}})}\left(\frac{1-\sigma^\textrm{NI}}{\sigma^\textrm{NI}}\right)^{l-2}\hat{q}_l^\textrm{NI}$, and 
\begin{equation*}
    \hat{q}_i^\textrm{NI}=\frac{p_{l+1}^+(0)}{p_{l+1}^-(0)}\dots\frac{p_{i-1}^+(0)}{p_{i-1}^-(0)}\frac{p_l^+(\overline{\mu}_l^\textrm{NI})}{p_i^-(0)}\hat{q}_l^\textrm{NI}\leq \frac{p_l^+(\overline{\mu}_l^\textrm{NI})}{p_{l+1}^-(0)}\left(\frac{\sigma(-\alpha\mu_{l+1})}{1-\sigma(-\alpha\mu_{l+1})}\right)^{i-l-1}\hat{q}_l^\textrm{NI},\; \forall i\geq l+1. 
\end{equation*}
For the last case, we consider $l<I$ since otherwise the bound is vacuous. 
By the definition of $l$, we have $\mu_{l+1}+\underline{w}_i^\textrm{NI}>0\implies \mu_{l+1}>-\underline{w}_i^\textrm{NG}$. By \cref{app:lemma:bar-larger}, we further have $\mu_{l+1}>-\underline{w}_i^\textrm{NG}>\overline{w}_i^\textrm{NG}$.
This implies $\frac{\sigma(-\alpha\mu_{l+1})}{1-\sigma(-\alpha\mu_{l+1})}=\frac{1-\sigma(\alpha\mu_{l+1})}{\sigma(\alpha\mu_{l+1})}<\frac{1-\sigma^\textrm{NG}}{\sigma^\textrm{NG}}$.
Thus, $\hat{q}_i^\textrm{NI}= o((\frac{\sigma^\textrm{NI}}{1-\sigma^\textrm{NI}})^{i-l})$ for $i\geq l+1$.  

For the NG strategy, observe that $\tilde{p}_i^+\geq p_i^+(\overline{\mu}_i^{\textrm{NG}})$ and $\tilde{p}_i^-\leq p_i^-(\overline{\mu}_i^{\textrm{NG}})$. Substituting these into \cref{app:thm:detailed-balance} reveals that 
\begin{equation*}
p_i^+(\overline{\mu}_i^{\textrm{NG}})\hat{q}_i^\textrm{NG}\leq p_{i+1}^-(\overline{\mu}_{i+1}^{\textrm{NG}})\hat{q}_{i+1}^\textrm{NG}\;\; \textrm{for } i\leq I-1. 
\end{equation*}
Recursively expanding the right-hand side yields 
\begin{equation*}
    \hat{q}_i^\textrm{NG}\leq\left(\frac{1-\sigma^\textrm{NG}}{\sigma^\textrm{NG}}\right)^{I-i}\hat{q}_I^\textrm{NG},\;\forall 2\leq i\leq I,\;\textrm{ and }\hat{q}_1^\textrm{NG}\leq \frac{p_2^-(\overline{\mu}_2^\textrm{NG})}{p_1^+(\overline{\mu}_1^\textrm{NG})}\left(\frac{1-\sigma^\textrm{NG}}{\sigma^\textrm{NG}}\right)^{I-i}\hat{q}_I^\textrm{NG}
\end{equation*}
which completes the proof.
\end{proof}

\subsubsection{Proof of Theorem~\ref{thm:long-term-attribute}}
\begin{reptheorem}{thm:long-term-attribute}
    Using classifiers given by  \cref{cond:level-cond}, $\hat{x}^\textrm{NI}=0$ and $\hat{x}^\textrm{NG}\geq \gamma\min_{i\in[I]}\overline{\mu}_i^\textrm{NG}$.
\end{reptheorem}
\begin{proof}
    This is an immediate implication of \cref{app:thm:stationary}.
\end{proof}

\subsubsection{Proof of Theorem~\cref{thm:long-term-utility}}
\begin{reptheorem}{thm:long-term-utility}
    {When $\underline{w}_1^\textrm{NG}\leq 0$}, for $\{\mu_i\}_{i}$ generated by \cref{alg:find-level}, we have $\hat{u}^\textrm{NG}\geq\hat{u}^\textrm{NI}$ if 
    \begin{equation*}
        \Delta\mu\leq \frac{2\sigma^\textrm{NI}-1}{\sigma^\textrm{NI}}\left(\frac{2\sigma^\textrm{NI}-1}{(\sigma^\textrm{NI})^2}\overline{w}_l^\textrm{NI}-\underline{w}_2^\textrm{NI}-\frac{c^\textrm{NG}}{c^\textrm{NI}\sigma^\textrm{NI}}\Delta w^\textrm{NG}\right) -\frac{r}{c^\textrm{NI}\sigma^\textrm{NI}}\left(\frac{\sigma^\textrm{NG}(2\sigma^\textrm{NI}-1)}{\sigma^\textrm{NI}(2\sigma^\textrm{NG}-1)}-1\right),
    \end{equation*}
    where $\sigma^\textrm{NI}$ and $\sigma^\textrm{NG}$ are defined in \cref{thm:long-term-levels} and $\Delta w^\textrm{NG} =\max_{i\in[I]}(\overline{w}_i^\textrm{NG}-\underline{w}_i^\textrm{NG})$.
\end{reptheorem}
\begin{proof}
    From \cref{app:thm:long-term-utility}, we obtain 
    \begin{equation*}
        \begin{aligned}
            \hat{u}^\textrm{NI}&=  r \sum_{i=1}^Ii\hat{q}_i^\textrm{NI}-c^\textrm{NI}\sum_{i=1}^l\overline{\mu}_i^\textrm{NI}\hat{q}_i^\textrm{NI} \\
            &=  r \sum_{i=1}^Ii\hat{q}_i^\textrm{NI}-c^\textrm{NI}\sum_{i=1}^l\overline{w}_i^\textrm{NI}\hat{q}_i^\textrm{NI} - c^\textrm{NI}\sum_{i=1}^l\mu_i\hat{q}_i^\textrm{NI} \\
            &\leq  r \sum_{i=1}^Ii\hat{q}_i^\textrm{NI} - c^\textrm{NI}\overline{w}_l^\textrm{NG}\hat{q}_l^\textrm{NI} - c^\textrm{NI}\sum_{i=1}^l\mu_i\hat{q}_i^\textrm{NI} \\
            &\overset{(a)}{\leq}  r \sum_{i=1}^Ii\hat{q}_i^\textrm{NI} - c^\textrm{NI}\overline{w}_l^\textrm{NG}\hat{\pi}_l - c^\textrm{NI}\sum_{i=1}^l\mu_i\hat{q}_i^\textrm{NI}\\
            &\overset{(b)}{<}  r \sum_{i=1}^Ii\hat{q}_i^\textrm{NI} - c^\textrm{NI}\overline{w}_l^\textrm{NG}\frac{2\sigma^\textrm{NI}-1}{\sigma^\textrm{NI}} - c^\textrm{NI}\sum_{i=1}^l\mu_i\hat{q}_i^\textrm{NI} \\
        \end{aligned}
    \end{equation*}
    where in (a), we define the auxiliary distribution $\hat{\vec\pi}$ by $\hat{\pi}_i=(\frac{1-\sigma^\textrm{NI}}{\sigma^\textrm{NI}})^{|l-i|}\hat{\pi}_l$ for every $i\in[I]$ and observe $\hat{q}^\textrm{NI}_l\geq \hat{\pi}_l$ since $\hat{\vec\pi}$ allocates more probability mass on levels $i>l$; and in (b), we observe $\hat{\pi}_l> 1/(\sum_{i=-\infty}^\infty(\frac{1-\sigma^\textrm{NI}}{\sigma^\textrm{NI}})^{|l-i|})= 1/(2\cdot\frac{\sigma^\textrm{NI}}{2\sigma^\textrm{NI}-1}-1)=2\sigma^\textrm{NI}-1$ and simplify. When the levels are constructed following \cref{alg:find-level}, it is clear that $l=\lfloor\frac{-\underline{w}_2}{\Delta\mu}\rfloor+1$ and $\mu_i=(i-1)\Delta\mu$. Substituting this into the last term yields
    \begin{equation*}
        \begin{aligned}
            c^\textrm{NI}\Delta\mu\sum_{i=1}^l(i-1)\hat{q}_i^\textrm{NI} &\geq c^\textrm{NI}\Delta\mu\hat{\pi}_l\sum_{i=1}^l(i-1)\left(\frac{1-\sigma^\textrm{NI}}{\sigma^\textrm{NI}}\right)^{l-i} \\
            &\overset{(c)}{=}c^\textrm{NI}\Delta\mu\hat{\pi}_l 
            \frac{(l-1)(1-\frac{1-\sigma^\textrm{NI}}{\sigma^\textrm{NI}})+(\frac{1-\sigma^\textrm{NI}}{\sigma^\textrm{NI}})^l-\frac{1-\sigma^\textrm{NI}}{\sigma^\textrm{NI}}}{(1-\frac{1-\sigma^\textrm{NI}}{\sigma^\textrm{NI}})^2}
            \\
            &= c^\textrm{NI}\Delta\mu\hat{\pi}_l\left\{\frac{\sigma^\textrm{NI}}{2\sigma^\textrm{NI}-1}(l-1)+\left(\frac{\sigma^\textrm{NI}}{2\sigma^\textrm{NI}-1}\right)^2\left[\left(\frac{1-\sigma^\textrm{NI}}{\sigma^\textrm{NI}}\right)^l-\frac{1-\sigma^\textrm{NI}}{\sigma^\textrm{NI}}\right]\right\} \\
            &\geq c^\textrm{NI}\Delta\mu\left\{\sigma^\textrm{NI}(l-1)-\frac{\sigma^\textrm{NI}(1-\sigma^\textrm{NI})}{2\sigma^\textrm{NI}-1}\right\} \\
            &=c^\textrm{NI}\sigma^\textrm{NI}\left\{(l-1)\Delta\mu - \frac{1-\sigma^\textrm{NI}}{2\sigma^\textrm{NI}-1}\Delta \mu\right\} \\
            &\overset{(d)}{\geq} c^\textrm{NI}\sigma^\textrm{NI}\left\{-\underline{w}_2^\textrm{NI} - \left(1+\frac{1-\sigma^\textrm{NI}}{2\sigma^\textrm{NI}-1}\right)\Delta\mu\right\} \\
            &= c^\textrm{NI}\sigma^\textrm{NI}\left\{-\underline{w}_2^\textrm{NI} - \frac{\sigma^\textrm{NI}}{2\sigma^\textrm{NI}-1}\Delta\mu\right\}
        \end{aligned}
    \end{equation*}
    where (c) applies the finite-sum equation of arithmetic-geometric sequences and (d) uses the fact that $x\cdot\lfloor\frac{m}{x}\rfloor\geq m-x$ for $m>0$. In a nutshell, the long-term utility for $\hat{u}^\textrm{NI}$ can be bounded above by
    \begin{equation}\label{app:eq:simplified-ni}
        \hat{u}^\textrm{NI}\leq   r \sum_{i=1}^Ii\hat{q}_i^\textrm{NI} - c^\textrm{NI}\sigma^\textrm{NI}\left\{\frac{2\sigma^\textrm{NI}-1}{(\sigma^\textrm{NI})^2}\overline{w}_l^\textrm{NG}-\underline{w}_2^\textrm{NI}-\frac{\sigma^\textrm{NI}}{2\sigma^\textrm{NI}-1}\Delta\mu\right\}.
    \end{equation}
        
    For the long-term NG strategy, observe
    \begin{align}\label{app:eq:simplified-ng}
        \hat{u}^\textrm{NG} &\geq   r \sum_{i=1}^Ii\hat{q}_i^\textrm{NG} - c^\textrm{NG}\sum_{i=1}^I(\overline{\mu}_i^\textrm{NG}-\underline{\mu}_i^\textrm{NG})\sum_{\underline{\mu}_i^\textrm{NG}x\leq \overline{\mu_i}^\textrm{NG}}\pi_i^\textrm{NG}(x) && \nonumber \\
        &\geq   r \sum_{i=1}^Ii\hat{q}_i^\textrm{NG} - c^\textrm{NG}\sum_{i=1}^I(\overline{\mu}_i^\textrm{NG}-\underline{\mu}_i^\textrm{NG})\hat{q}_i^\textrm{NG} \nonumber \\
        &\geq   r \sum_{i=1}^Ii\hat{q}_i^\textrm{NG} - c^\textrm{NG}\Delta w^\textrm{NG}.
    \end{align}

    The last step requires bounding the benefit (first) term of both $\hat{u}^\textrm{NI}$ and $\hat{u}^\textrm{NG}$. We define the auxiliary level distributions $\tilde{\vec{q}}^\textrm{NI}$ and $\tilde{\vec{q}}^\textrm{NG}$ that satisfy $\tilde{q}_i^\textrm{NI}=(\frac{1-\sigma^\textrm{NI}}{\sigma^\textrm{NI}})^{I-i}\tilde{q}_I^\textrm{NI}$ and $\tilde{q}_i^\textrm{NG}=(\frac{1-\sigma^\textrm{NG}}{\sigma^\textrm{NG}})^{I-i}\tilde{q}_I^\textrm{NG}$ respectively. By \cref{thm:long-term-levels}, it is obvious that $\sum_{i=1}^Ii\tilde{q}_i^\textrm{NI}\geq \sum_{i=1}^Ii\hat{q}_i^\textrm{NI}$ as the former concentrates more probability mass on higher levels. On the contrary, $\sum_{i=1}^Ii\tilde{q}_i^\textrm{NG}\leq \sum_{i=1}^Ii\hat{q}_i^\textrm{NG}$ because while both distributions concentrate around level $I$, $\tilde{\vec{q}}^\textrm{NG}$ distributes less probability mass on levels less than $I$ (see \cref{app:thm:ergodic} and the proof of \cref{thm:long-term-levels}). Since $\sigma^\textrm{NI}>\sigma^\textrm{NG}$ under assumption of cheaper cheating (see \cref{app:lemma:cheating-more-effort}), we have $\frac{1-\sigma^\textrm{NI}}{\sigma^\textrm{NI}}<\frac{1-\sigma^\textrm{NG}}{\sigma^\textrm{NG}}$. Then, according to \cref{app:lemma:benefit-gain}, 
    \begin{equation}\label{app:eq:benefit}
        \sum_{i=1}^Ii\tilde{q}_i^\textrm{NI}\leq \sum_{i=1}^Ii\tilde{q}_i^\textrm{NG} + \frac{\sigma^\textrm{NG}}{2\sigma^\textrm{NG}-1}-\frac{\sigma^\textrm{NI}}{2\sigma^\textrm{NI}-1}.
    \end{equation}
    
    Therefore, combining \cref{app:eq:simplified-ni,app:eq:simplified-ng,app:eq:benefit}, we obtain a sufficient condition under which $\hat{u}^\textrm{NG}\geq \hat{u}^\textrm{NI}$:
    \begin{equation*}
        \Delta\mu\leq \frac{2\sigma^\textrm{NI}-1}{\sigma^\textrm{NI}}\left(\frac{2\sigma^\textrm{NI}-1}{(\sigma^\textrm{NI})^2}\overline{w}_l^\textrm{NI}-\underline{w}_2^\textrm{NI}-\frac{c^\textrm{NG}}{c^\textrm{NI}\sigma^\textrm{NI}}\Delta w^\textrm{NG}\right) -\frac{  r }{c^\textrm{NI}\sigma^\textrm{NI}}\left(\frac{\sigma^\textrm{NG}(2\sigma^\textrm{NI}-1)}{\sigma^\textrm{NI}(2\sigma^\textrm{NG}-1)}-1\right)
    \end{equation*}
    and thereby complete the proof.
\end{proof}

\subsection{Proof of the algorithm}\label{app:alg-proof}
\begin{theorem}
    When $\underline{w}_1^\textrm{NG}\leq 0$, \cref{alg:find-level} returns the longest sequence $\{\mu_i\}_i$ that satisfies \cref{cond:level-cond} not exceeding $M$.
\end{theorem}
\begin{proof}
    \cref{alg:find-level} consists of a forward and a backward pass. The forward pass tries to identify the maximum number of intermediate levels, while the backward pass finds the first level from the back that satisfies the requirements of being the last level. Let $I$ be the length of the sequence $\{\mu_i\}_i$ returned by \cref{alg:find-level}. Then, clearly $\mu_i\leq M$, $\forall i$, and $I<\frac{M}{\Delta\mu}<\infty$.
    Notice that \cref{cond:level-cond}-(a) is readily satisfied because $\underline{\mu}_1^\textrm{NG}=\underline{w}_1^\textrm{NG}\leq0$. 

    If $I\geq 3$, then from line 2 of \cref{alg:find-level}, we have $\overline\mu_2^\textrm{NG}=\Delta\mu+\overline{w}_2^\textrm{NG}>\overline{w}_1^\textrm{NG}=\overline{\mu}_1^\textrm{NG}$. For each $3\leq i\leq I-1$, we have $\overline{\mu}_{i}^\textrm{NG}=\overline{\mu}_{i-1}^\textrm{NG}+\Delta\mu>\overline\mu_{i-1}^\textrm{NG}$ since $\Delta\mu>0lt$. For the last level, line 10 implies $\overline{\mu}_I^\textrm{NG}=\mu_I+\overline{w}_I^\textrm{NG}>\mu_{I-1}+\overline{w}_2^\textrm{NG}=\overline{\mu}_{I-1}^\textrm{NG}$. Thus, \cref{cond:level-cond}-(b) is satisfied. Similarly, observe $\underline\mu_{i}^\textrm{NG}\leq \gamma\overline{\mu}_{i-1}^\textrm{NG}$ for $2\leq i\leq I-1$ by line 2 and 3. Applying \cref{cond:level-cond}-(b) yields $\gamma\overline{\mu}_i^\textrm{NG}\geq\min\{\underline{\mu}_{i+1}^\textrm{NG},\underline{\mu}_i^\textrm{NG}\}$ for $1\leq i \leq I-2$. The second half of line 10 implies $\underline{\mu}_I^\textrm{NG}\leq \gamma\overline{\mu}_{I-1}^\textrm{NG}$, which further implies, by \cref{cond:level-cond}-(b), that $\underline{\mu}_{I-1}^\textrm{NG}<\underline{\mu}_I^\textrm{NG}\leq \gamma\overline{\mu}_{I-1}^\textrm{NG}<\gamma\overline{\mu}_I^\textrm{NG}$. Thus, \cref{cond:level-cond}-(c) is also satisfied.
    If $I=2$, $\mu_2$ is the terminal level. It is easy to check by line 10 that \cref{cond:level-cond} is satisfied.

    To see why it is the longest, suppose there exists another sequence $\{\mu'_j\}_{j=1}^J$ that also satisfies \cref{cond:level-cond} but with $J>I$. We let $\mu'_1=0$ without loss of generality, since we can always set $\mu_j'\gets\mu_j'-\mu_1$ and obtain a new valid sequence with the same length. Notice that in this case $\mu_i=\mu_i'$ for $i\leq I$. However, $\{\mu'_j\}_{j=I+1}^{J-1}$ are valid intermediate levels, and thus will be included in the forward pass (lines 2 to 7). Lines 8 and 9 ensure $\mu_J'$ will also be included before the backward pass (lines 11 to 15). Since $\{\mu_j\}_{j=1}^J$ satisfies \cref{cond:level-cond}, the backward pass will break at $i=J$ and return exactly $\{\mu'_j\}_{j=1}^J$. Therefore, there is no valid sequence longer than the output of \cref{alg:find-level}.
\end{proof}

\section{Simulation Setup and Additional Results}\label{app:experiment}

For the following simulations, the following parameters were used unless stated otherwise: $\alpha = 4$, $\beta = 0.604$, $\gamma = 0.9$, $c^{NI} = 0.75$, $c^{NG} = 0.80$, and $T = 10,000$ time steps. 
We perform the experiments over three abstention functions: normalized entropy (i.e., $h(\sigma)=\tilde{\beta}\frac{-\sigma\ln\sigma-(1-\sigma)\ln(1-\sigma)}{\ln2}$), polynomial (i.e., $h(\sigma) = \tilde{\beta}(4\sigma(1-\sigma)))^t$, and absolute value (i.e., $h(\sigma) = \tilde{\beta}|2\sigma-1|$), where $\tilde \beta$ denotes the abstention sensitivity (see \cref{app:beta}).  
We also call the polynomial abstention function the parabolic abstention function when $t=1$.
Unless stated otherwise, we assume the polynomial abstention function. 

For both \cref{fig:ex4_summary_panel} and the results in \cref{app:dmu,app:M,app:exp10}, we make a slight adjustment to \cref{alg:find-level} by relaxing \cref{cond:level-cond}-(b) for the last level. This is equivalent to removing the first inequality of line 10 (resulting in $\mu_j\leq \gamma(\mu_{j-1}+\overline{w}_{k_j-1}^\textrm{NG})-\underline{w}_I^\textrm{NG}$). Due to a lack of incentivization of the last level, where the agent never acts for promotion, the increment in the last level shouldn't be too small to ensure \cref{cond:level-cond}-(b) that forces continually increasing post-response attributes. This is incompatible with large numbers of levels where the threshold gap is minimal. While this modification can lead to $\mu_I-\mu_{I-1}>\Delta\mu$, all previous levels still satisfy the fixed step difference $\Delta\mu$ and \cref{cond:level-cond}, and, by empirical evidence, the threshold sequence returned by \cref{alg:find-level} still satisfies $\mu_I=M$ and $I=M/\Delta\mu$, for $M$ specific to the different experiments below. Thus, there is no loss of generalization in the conclusions.

\subsection{Compute resources}
\label{app:resources}
All simulations were executed on clustered servers. 
Each job requested 64 CPU cores and completed in approximately 3 hours. This covered 30 sample paths for each experimental condition across a range of parameter settings. The total compute time of all experiments was approximately 80 hours. 

\subsection{Simulation error characterization}
\label{app:error}
For each experiment, we ran 30 independent simulation trials. For each trial, we computed the time-averaged values of agent metrics over the final 2,000 steps. We then computed the mean and standard deviation of these 30 trial averages. Error bars/regions in the figures are plotted as $\pm2$ standard deviations from the mean. We do not assume normality of the underlying distribution and do not report formal confidence intervals.

\subsection{On the impact of threshold increment ($\Delta\mu$)}\label{app:dmu}

We fix the highest threshold level $M=3$ and vary the threshold increment $\Delta\mu$. We plot various properties of interest as functions of the level counts $I=M/\Delta\mu$ for different abstention functions.

\begin{figure}[htbp]
    \centering

    \begin{subfigure}[t]{0.32\textwidth}
        \centering
        \includegraphics[width=\textwidth]{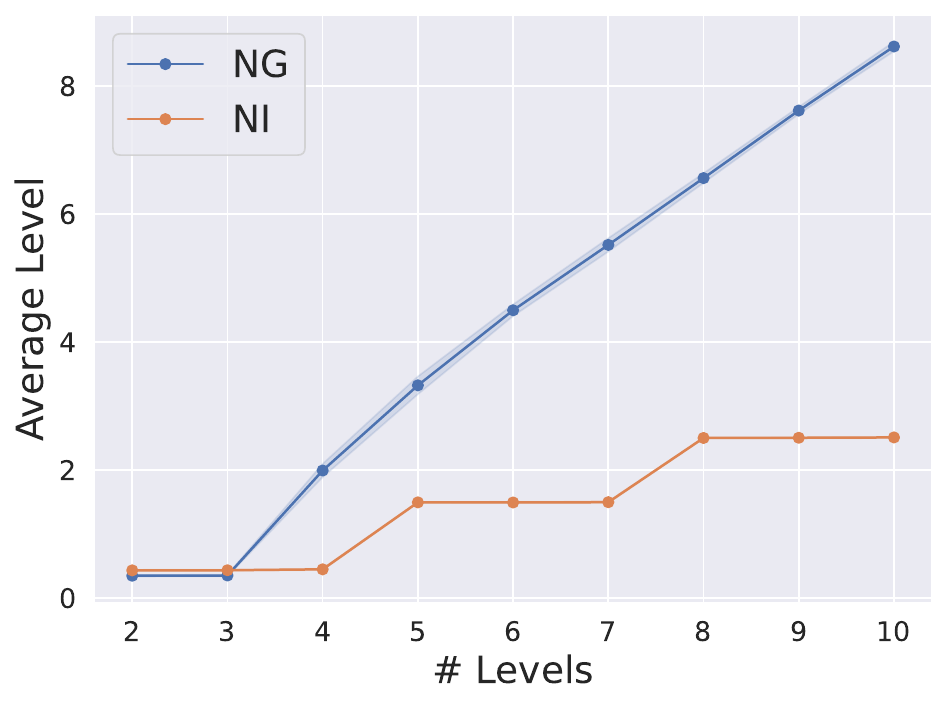}
        \caption{Normalized entropy $h(\sigma)$}
        \label{fig:ex1_level}
    \end{subfigure}%
    \hfill
    \begin{subfigure}[t]{0.32\textwidth}
        \centering
        \includegraphics[width=\textwidth]{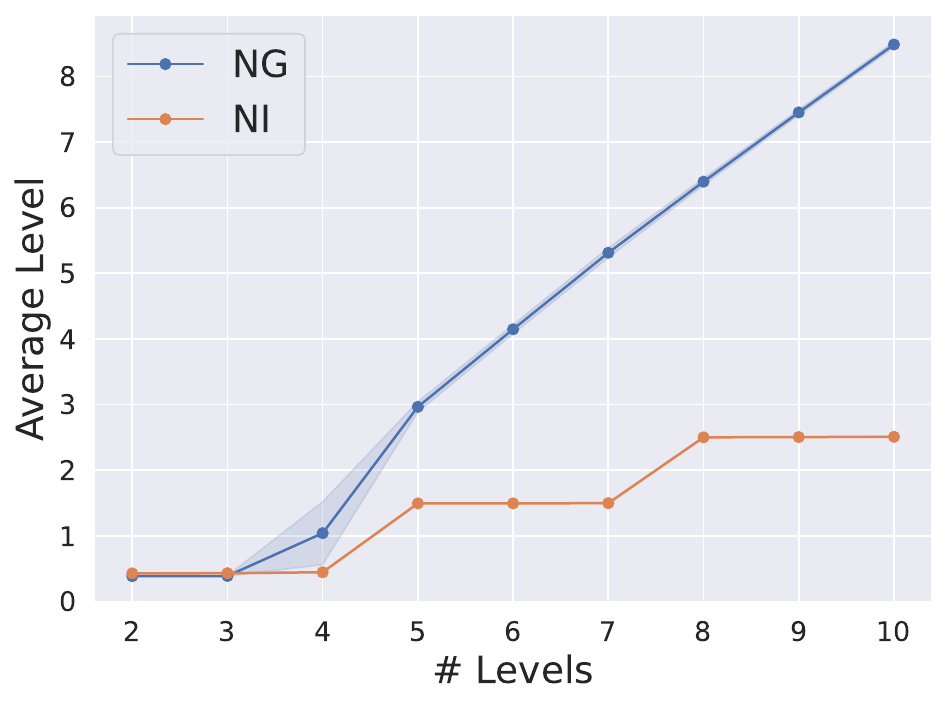}
        \caption{Polynomial $h(\sigma)$ with $t=25.5$}
        \label{fig:ex3_level}
    \end{subfigure}%
    \hfill
    \begin{subfigure}[t]{0.32\textwidth}
        \centering
        \includegraphics[width=\textwidth]{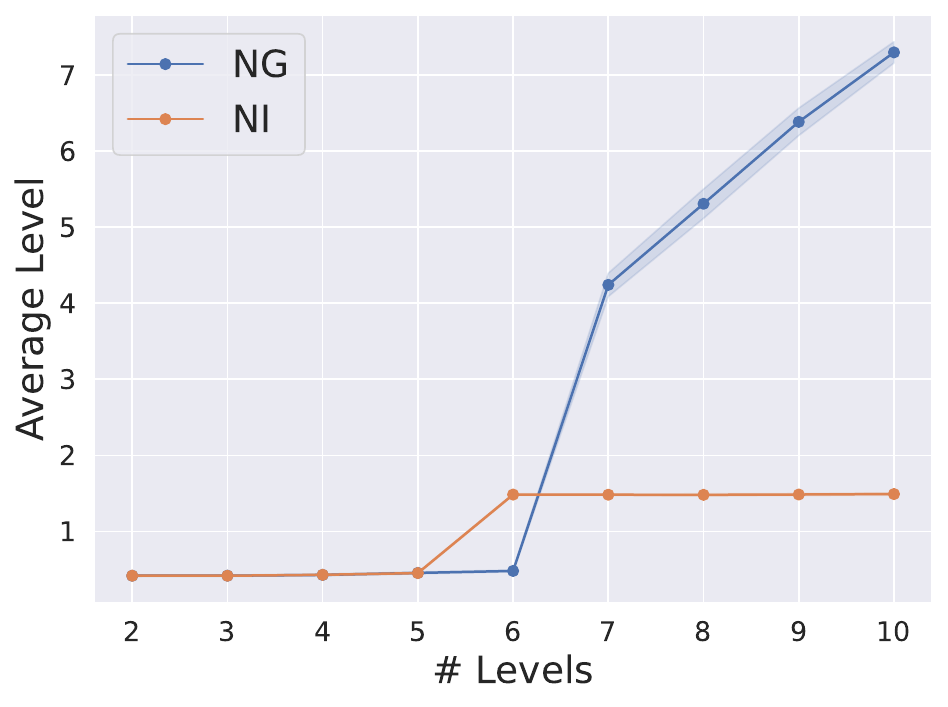}
        \caption{Absolute value $h(\sigma)$}
        \label{fig:ex2_level}
    \end{subfigure}

    \caption{\textbf{Average level} across different level counts with fixed $M=3$ under No Gaming (NG) and No Improvement (NI) strategies $\pm$2 STD.}
    \label{fig:level_summary_panel}
\end{figure}

\cref{fig:level_summary_panel} compares the average level attained by agents under the three different abstention functions: normalized entropy, absolute value, and parabolic. In all cases, we observe a trend that the NG agent reaches significantly higher levels than the NI agent, when there are a certain number of levels between $\mu_0$ and $\mu_I$. For the normalized entropy and parabolic abstention functions, when there are at least 4 levels between $\mu_0$ and $\mu_I$, the NG agent reaches higher levels than the NI agent. For the absolute value abstention function, the NG agent reaches higher average level when there are at least 7 levels between $\mu_0$ and $\mu_I$.

\begin{figure}[htbp]
    \centering
    \begin{subfigure}[t]{0.32\textwidth}
        \centering
        \includegraphics[width=\textwidth]{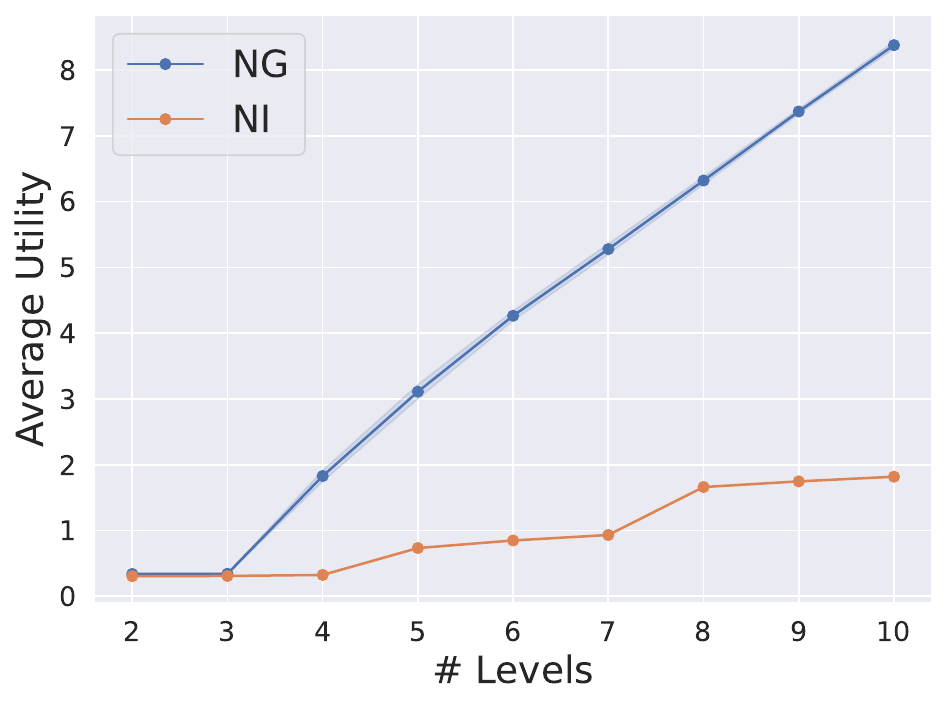}
        \caption{Normalized entropy $h(\sigma)$}
        \label{fig:ex1_util}
    \end{subfigure}%
    \hfill
    \begin{subfigure}[t]{0.32\textwidth}
        \centering
        \includegraphics[width=\textwidth]{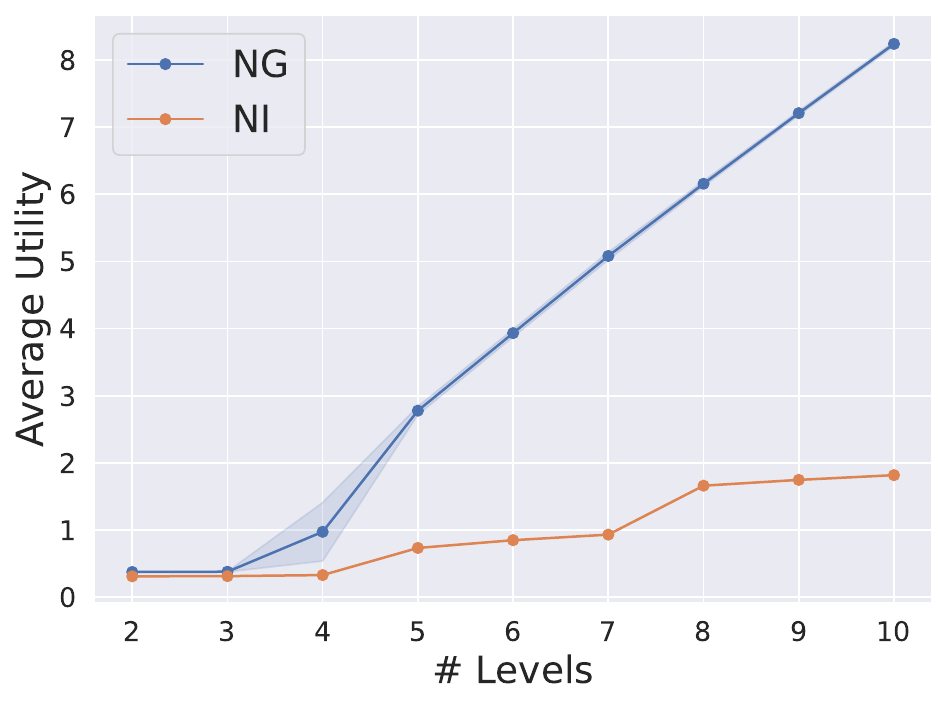}
        \caption{Polynomial $h(\sigma)$ with $t=25.5$}
        \label{fig:ex3_util}
    \end{subfigure}%
    \hfill
    \begin{subfigure}[t]{0.32\textwidth}
        \centering
        \includegraphics[width=\textwidth]{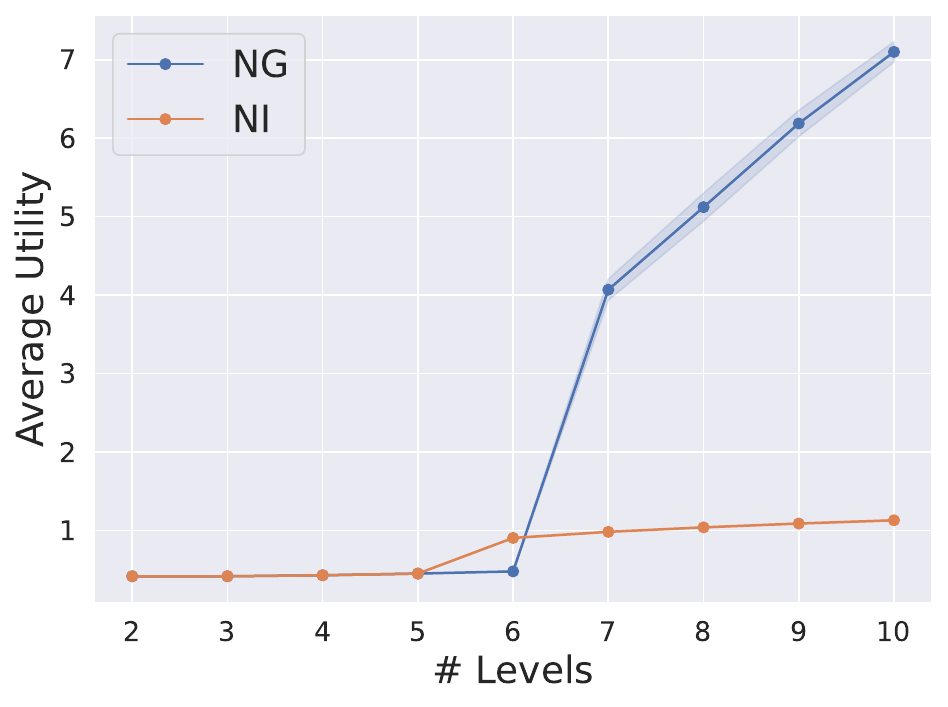}
        \caption{Absolute value $h(\sigma)$}
        \label{fig:ex2_util}
    \end{subfigure}
    \caption{\textbf{Average utility} across different level counts with fixed $M=3$ under No Gaming (NG) and No Improvement (NI) strategies $\pm$2 STD.}
    \label{fig:util_summary_panel}
\end{figure}

\cref{fig:util_summary_panel} compares the average utility attained by agents with the same three abstention functions. Again we see the same trend that the NG agent reaches significantly higher utility than the NI agent, when there are a certain number of levels between $\mu_0$ and $\mu_I$. For the normalized entropy and polynomial abstention functions, 3 levels are required for the NG agent to achieve higher average utility than the NI agent. For the absolute value abstention function, 7 levels are required to see the same pattern. 

\begin{figure}[htbp]
    \centering

    \begin{subfigure}[t]{0.32\textwidth}
        \centering
        \includegraphics[width=\textwidth]{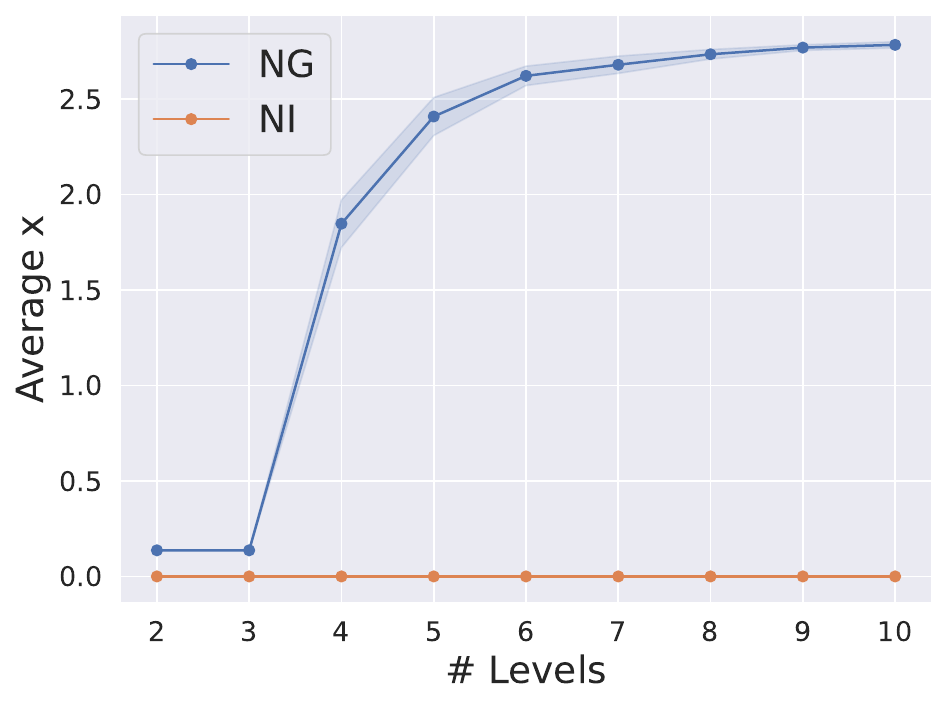}
        \caption{Normalized entropy $h(\sigma)$}
        \label{fig:ex1_qual}
    \end{subfigure}%
    \hfill
    \begin{subfigure}[t]{0.32\textwidth}
        \centering
        \includegraphics[width=\textwidth]{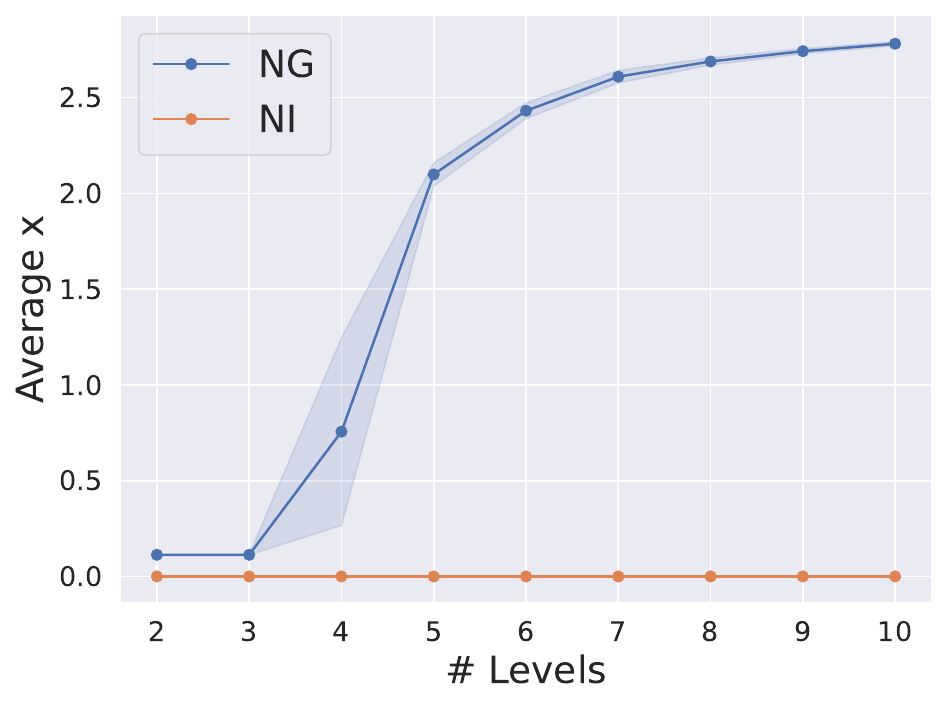}
        \caption{Polynomial $h(\sigma)$  with $t=25.5$.}
        \label{fig:ex3_qual}
    \end{subfigure}%
    \hfill
    \begin{subfigure}[t]{0.32\textwidth}
        \centering
        \includegraphics[width=\textwidth]{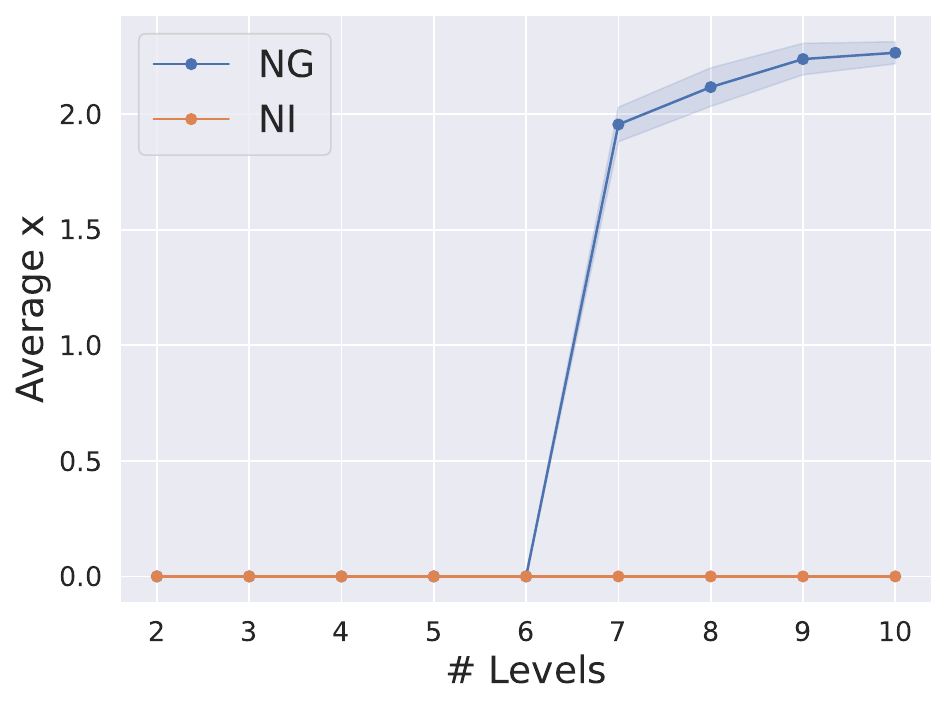}
        \caption{Absolute value $h(\sigma)$}
        \label{fig:ex2_qual}
    \end{subfigure}

    \caption{\textbf{Average qualification} across different level counts with fixed $M=3$ under No Gaming (NG) and No Improvement (NI) strategies $\pm$2 STD.}
    \label{fig:qual_summary_panel}
\end{figure}

\cref{fig:qual_summary_panel} compares the average qualification (true feature x) attained by agents with the same three abstention functions. Here we see that the NI agent will always have a true feature $x = 0$, while the true feature of the NG agent increases. 

\subsection{On the impact of the highest threshold ($M$)}\label{app:M}

\begin{figure}[htbp]
    \centering

    \begin{subfigure}[t]{0.32\textwidth}
        \centering
        \includegraphics[width=\textwidth]{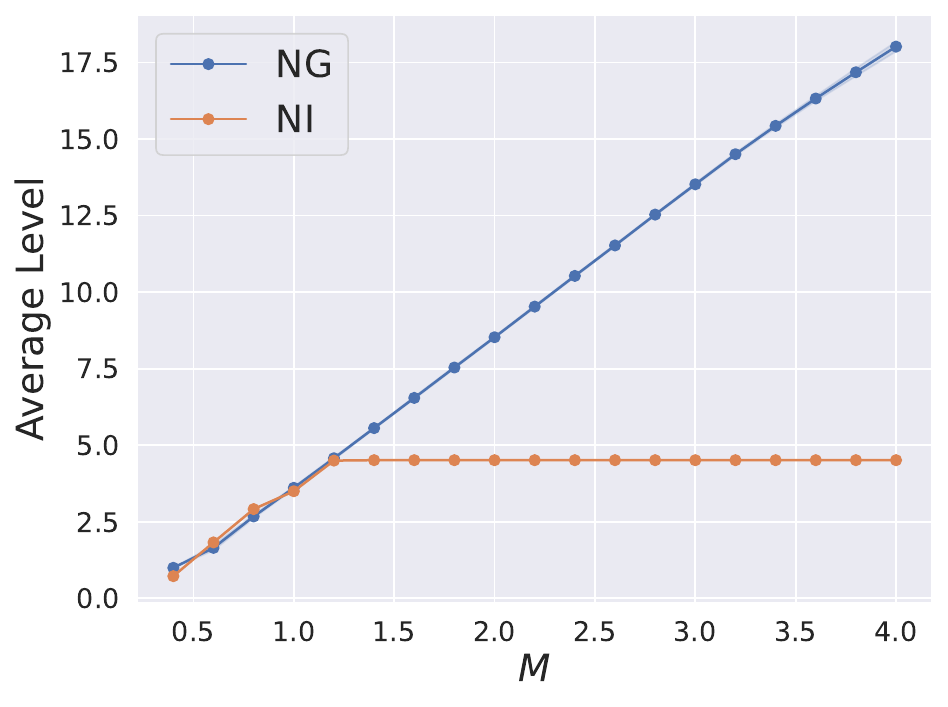}
        \caption{Normalized entropy $h(\sigma)$}
        \label{fig:ex01_level}
    \end{subfigure}%
    \hfill
    \begin{subfigure}[t]{0.32\textwidth}
        \centering
        \includegraphics[width=\textwidth]{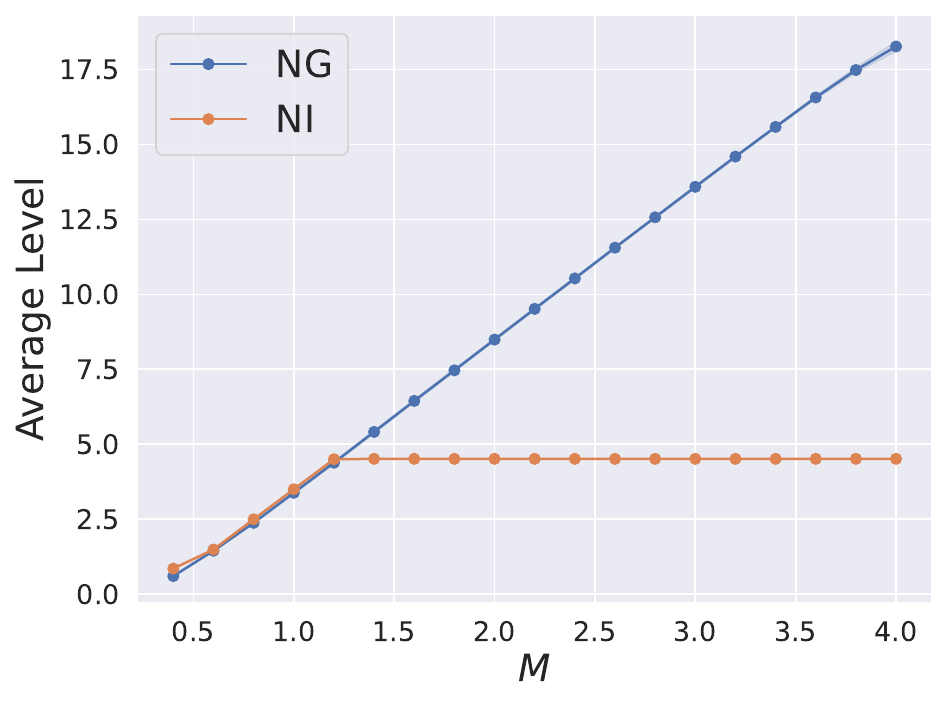}
        \caption{Polynomial $h(\sigma)$ with $t=25.5$}
        \label{fig:ex03_level}
    \end{subfigure}%
    \hfill
    \begin{subfigure}[t]{0.32\textwidth}
        \centering
        \includegraphics[width=\textwidth]{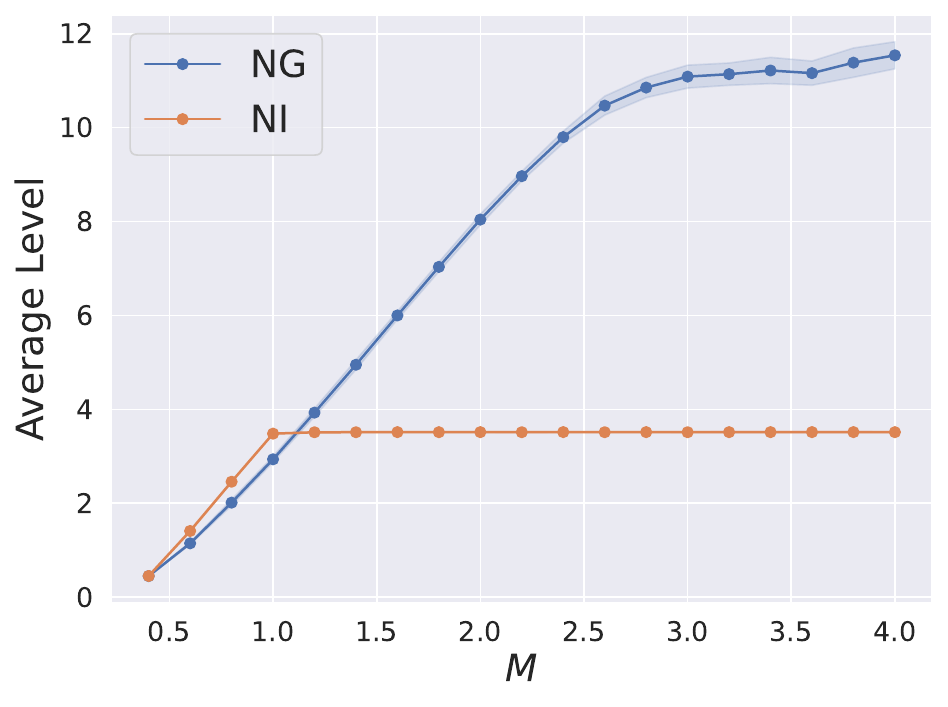}
        \caption{Absolute value $h(\sigma)$}
        \label{fig:ex02_level}
    \end{subfigure}

    \caption{\textbf{Average level} across different $M$ with fixed $\Delta\mu=0.2$ under No Gaming (NG) and No Improvement (NI) strategies $\pm$2 STD.}
    \label{fig:level_M}
\end{figure}

\cref{fig:level_M,fig:util_M,fig:qual_M} repeat the experiments of \cref{app:dmu} but with varying $M$ and fixed $\Delta\mu=0.2$. Similarly, we conclude that the NG strategy attains higher average levels, attributes, and utilities under higher-spread threshold sequences when the threshold gap is fixed.

\begin{figure}[htbp]
    \centering
    \begin{subfigure}[t]{0.32\textwidth}
        \centering
        \includegraphics[width=\textwidth]{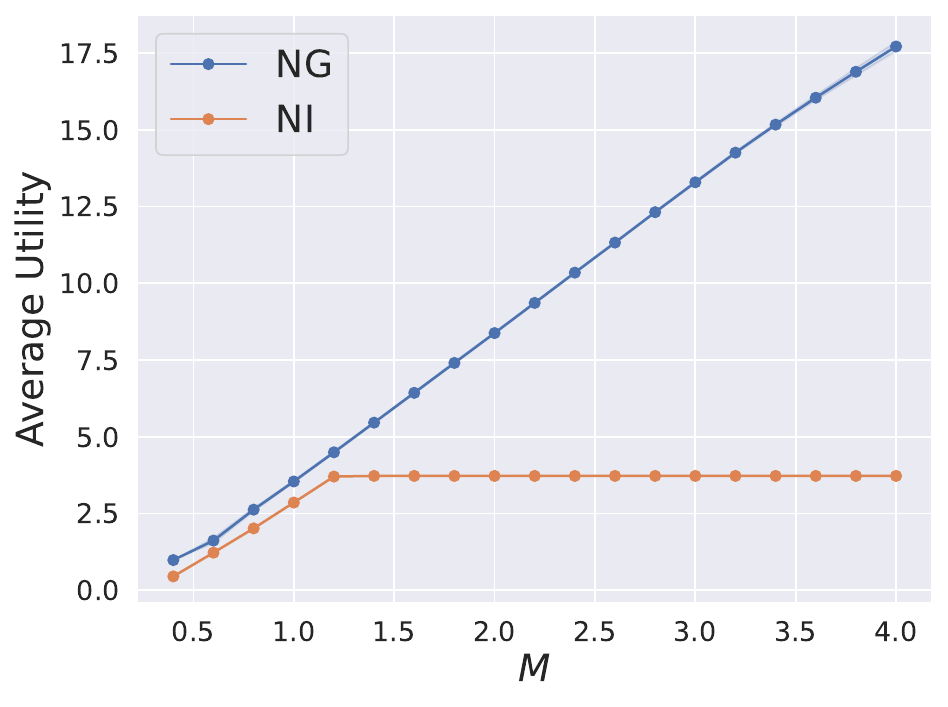}
        \caption{Normalized entropy $h(\sigma)$}
        \label{fig:ex01_util}
    \end{subfigure}%
    \hfill
    \begin{subfigure}[t]{0.32\textwidth}
        \centering
        \includegraphics[width=\textwidth]{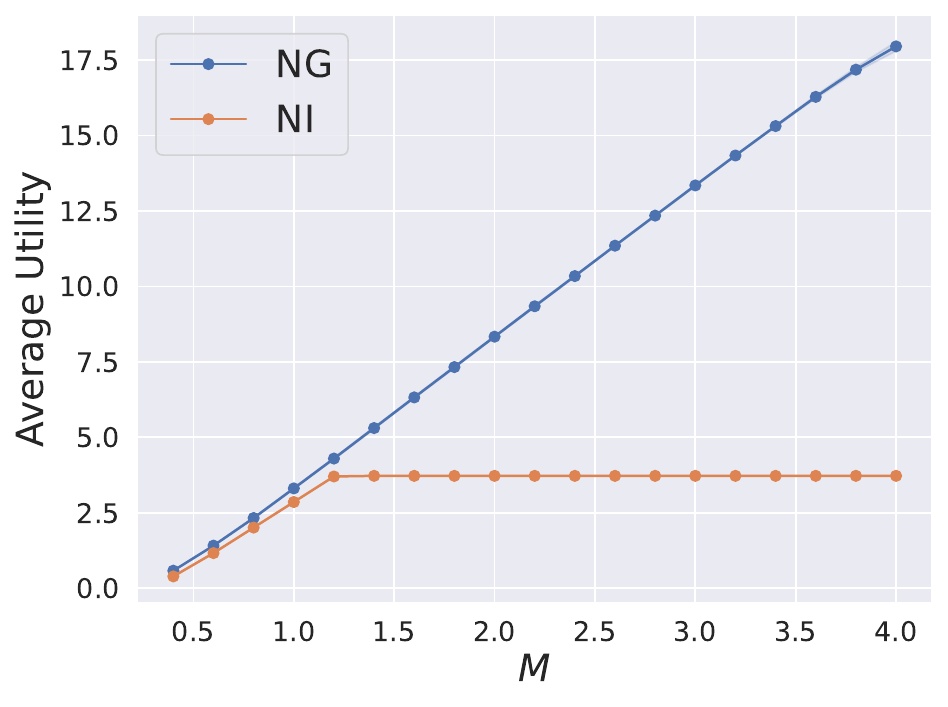}
        \caption{Polynomial $h(\sigma)$ with $t=25.5$}
        \label{fig:ex03_util}
    \end{subfigure}%
    \hfill
    \begin{subfigure}[t]{0.32\textwidth}
        \centering
        \includegraphics[width=\textwidth]{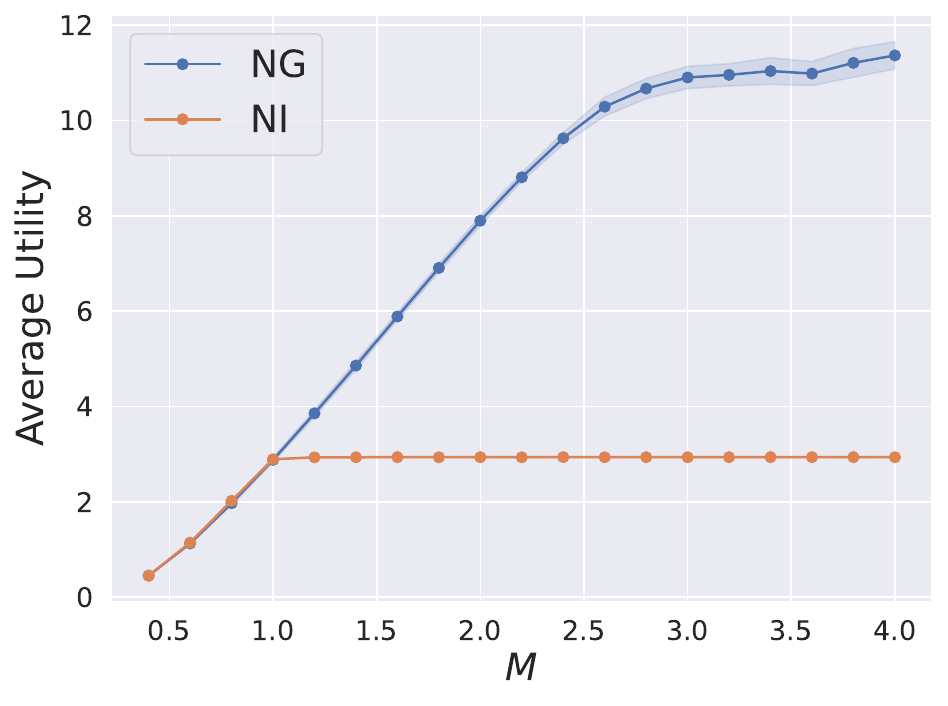}
        \caption{Absolute value $h(\sigma)$}
        \label{fig:ex02_util}
    \end{subfigure}
    \caption{\textbf{Average utility} across different $M$ with fixed $\Delta\mu=0.2$ under No Gaming (NG) and No Improvement (NI) strategies $\pm$2 STD.}
    \label{fig:util_M}
\end{figure}

\begin{figure}[htbp]
    \centering

    \begin{subfigure}[t]{0.32\textwidth}
        \centering
        \includegraphics[width=\textwidth]{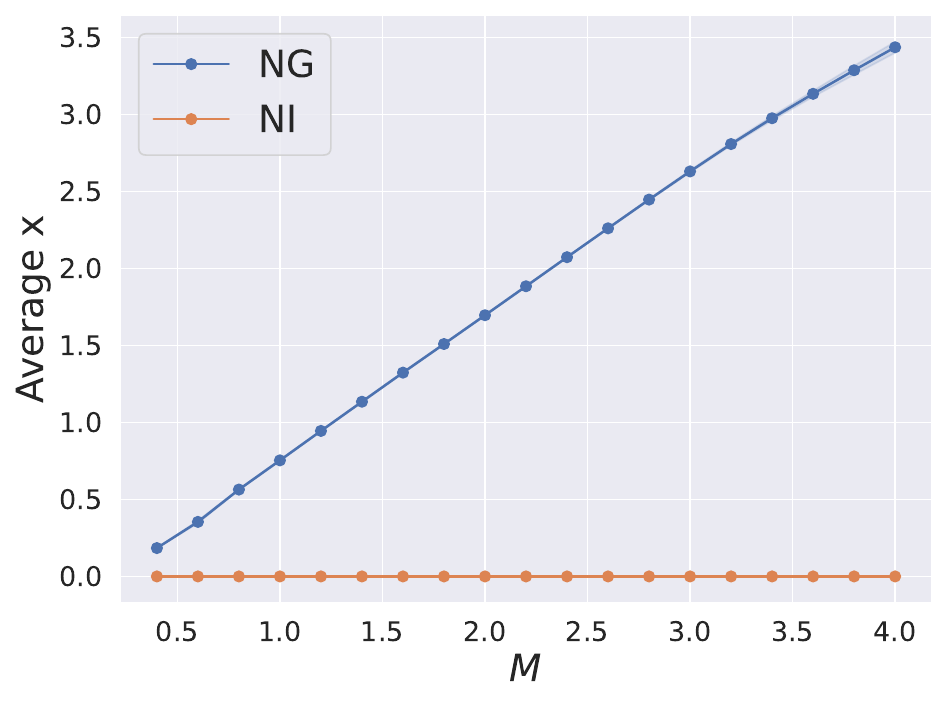}
        \caption{Normalized entropy $h(\sigma)$}
        \label{fig:ex01_qual}
    \end{subfigure}%
    \hfill
    \begin{subfigure}[t]{0.32\textwidth}
        \centering
        \includegraphics[width=\textwidth]{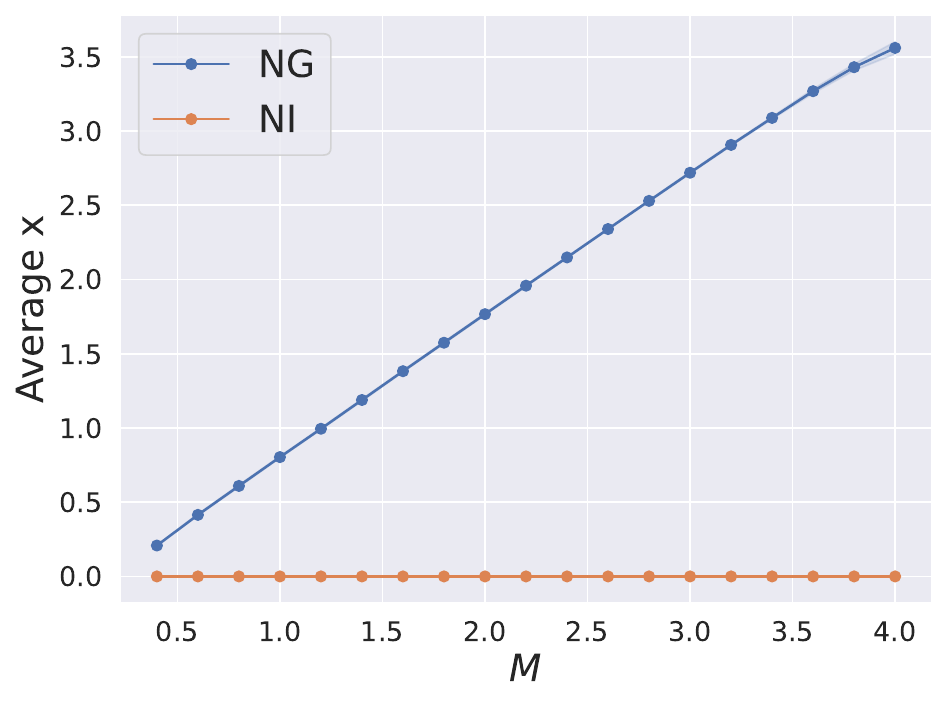}
        \caption{Polynomial $h(\sigma)$  with $t=25.5$.}
        \label{fig:ex03_qual}
    \end{subfigure}%
    \hfill
    \begin{subfigure}[t]{0.32\textwidth}
        \centering
        \includegraphics[width=\textwidth]{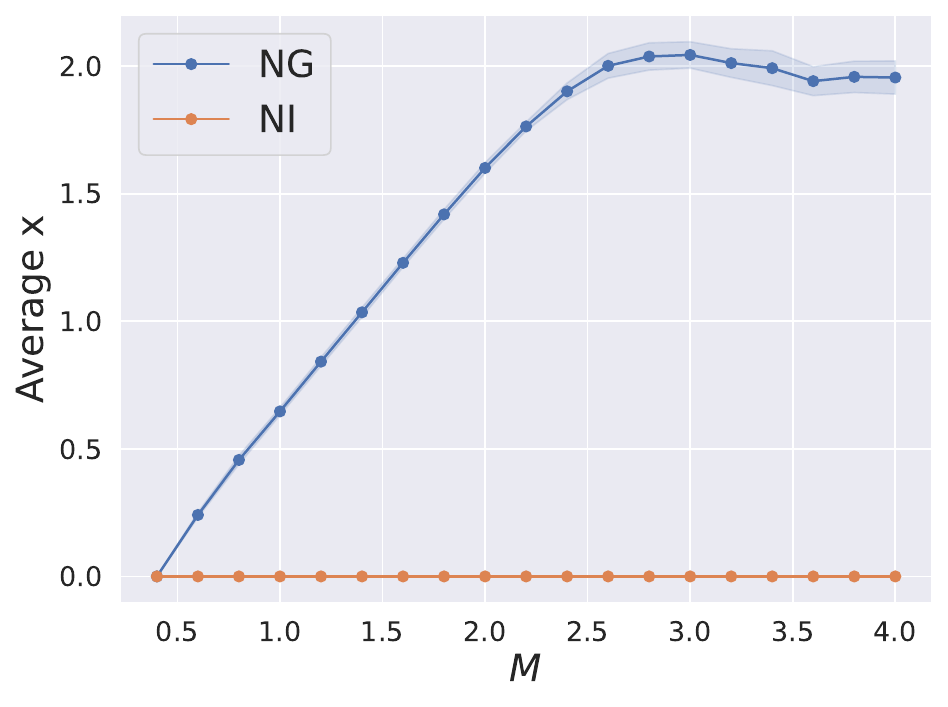}
        \caption{Absolute value $h(\sigma)$}
        \label{fig:ex02_qual}
    \end{subfigure}

    \caption{\textbf{Average qualification} across different $M$ with fixed $\Delta\mu=0.2$ under No Gaming (NG) and No Improvement (NI) strategies $\pm$2 STD.}
    \label{fig:qual_M}
\end{figure}

\subsection{On the truly optimal agent strategy}\label{app:exp10}
In this experiment, we approximate truly optimal and non-restricted (where both gaming and improvement can be freely selected) policies instead of myopic and restricted (either no-improvement or no-gaming) policies. For ease of simulation, we aim to solve for the agent a discounted-reward MDP over a large finite horizon. The reward at each step is the (non-random) instantaneous utility. 
We first discretize the attribute and action spaces by $\mathcal{X}=\{0,dx,\cdots,n\cdot dx\}$ with $dx=0.2$ and $n=100$, and $\mathcal{A}=\{0,da,\cdots,n\cdot da\}^2$ with $da=0.2$ and $n=15$. Particularly, we consider one action dimension for either improvement or gaming without loss of generality. The parameters $n$, $dx$, and $da$ are chosen primarily to ensure reasonable computational complexity (the size of the state-action space is approximately $n^3$), while keeping $dx$ and $da$ as small as possible. 

Then, we implement an on-policy SARSA algorithm with the UCB exploration and train, for each parameter setup, over 1999 episodes, each with $T=2500$ steps. However, since SARSA only approximates the optimal policy, we repeat the training process (including 1999 episodes and 2500 steps) 10 times with different random seeds and take the resulting policy with the largest average reward over the last 100 episodes (about 5\% of total episodes, for more accurate estimates) as the ``truly optimal policy''. We follow this procedure to find truly optimal policies under different values of $I$ but fixed $\mu_1=0$ and $\mu_I=M=10$. \cref{fig:rl} shows the average difference in improvement and gaming efforts, and the average decision maker's utility for each truly optimal policy evaluated in a new episode with $T=2500$ steps, with averages taken over the $T=2500$ steps. 
It can be seen that, under truly optimal policies, the average improvement effort is higher than the average gaming effort at larger numbers of levels (i.e., smaller threshold gaps) and that the decision maker's utility generally decreases with the number of levels, as smaller threshold gaps make the classifier less confident in the next level. 

\begin{figure}[!htbp]
    \centering
    \includegraphics[width=0.5\linewidth]{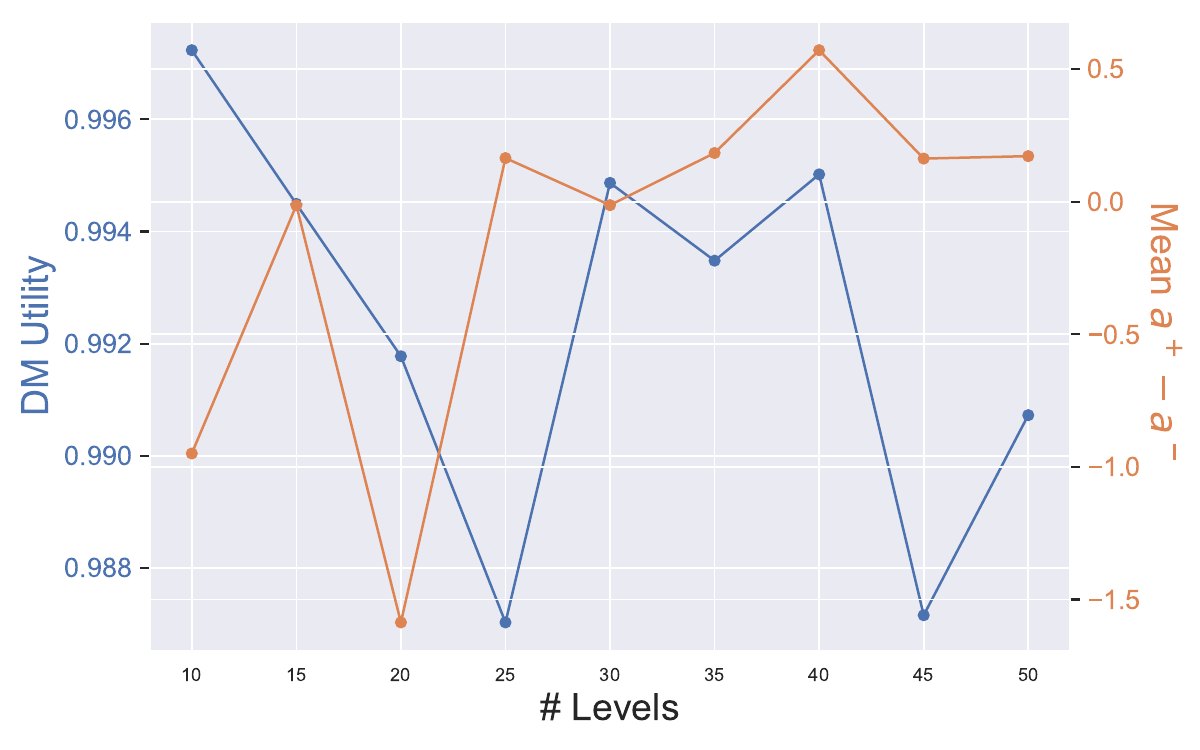}
    \caption{Truly optimal policies across different numbers of levels.}
    \label{fig:rl}
\end{figure}

\subsection{On the stationary level distribution ($\hat{\vec q}^\textrm{NI}$ and $\hat{\vec q}^\textrm{NG}$)}

\begin{figure}[htbp]
    \centering
    \begin{subfigure}[t]{0.48\textwidth}
        \centering
        \includegraphics[width=.8\textwidth]{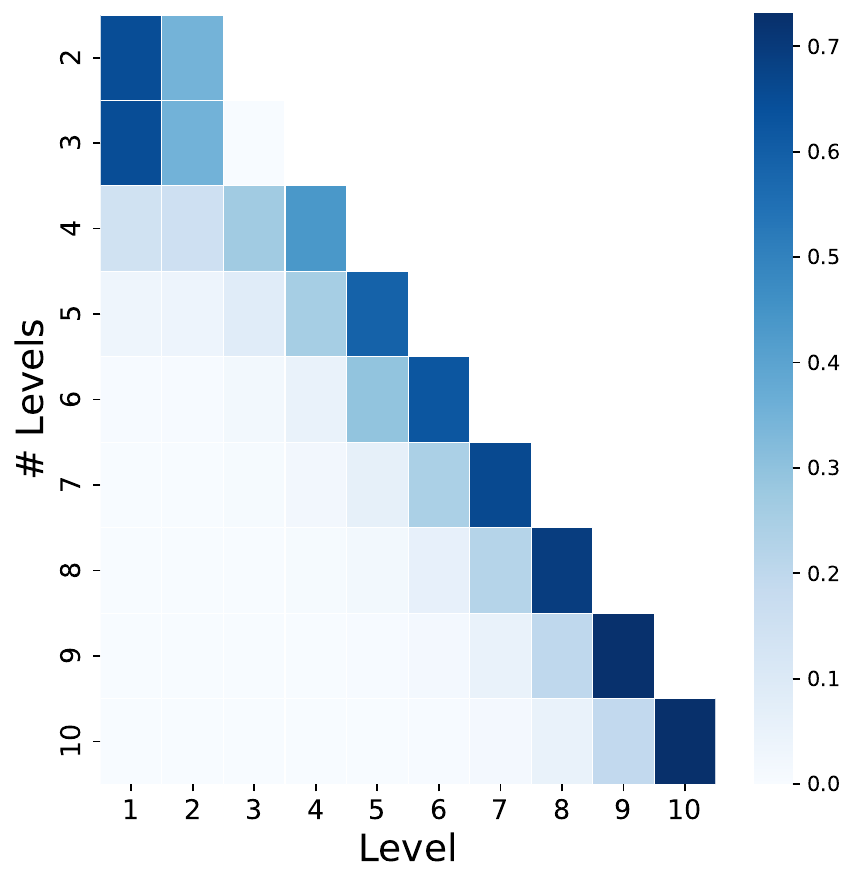}
        \caption{NG agents}
        \label{fig:ex1_statdistn_ng}
    \end{subfigure}%
    \hfill
    \begin{subfigure}[t]{0.48\textwidth}
        \centering
        \includegraphics[width=.8\textwidth]{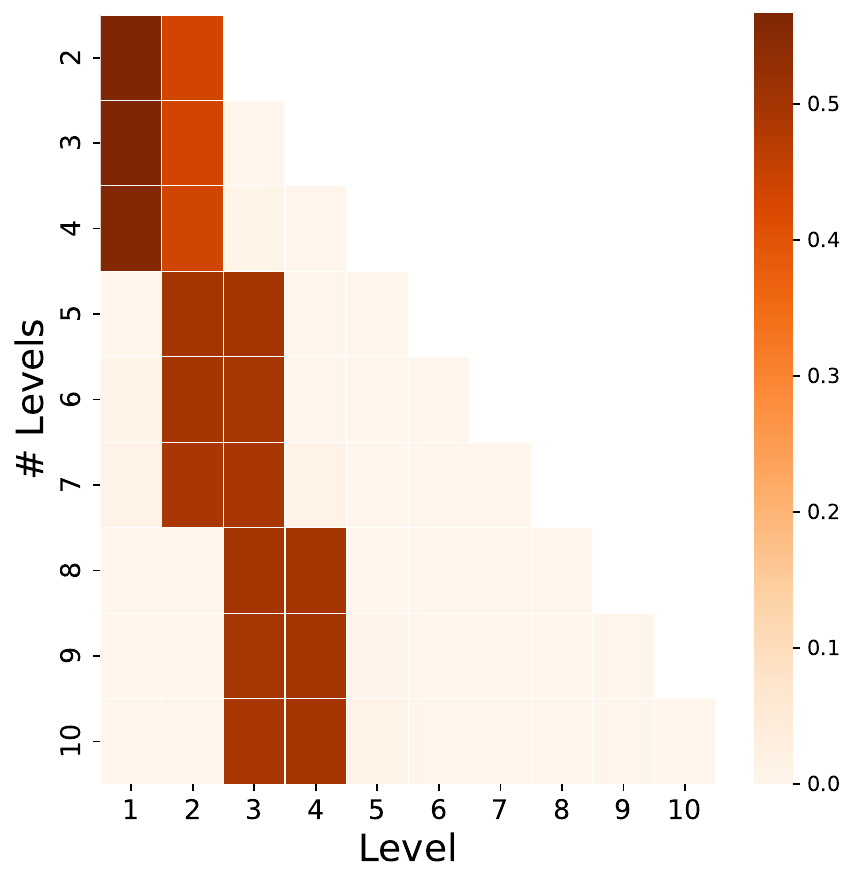}
        \caption{NI agents}
        \label{fig:ex1_statdistn_ni}
    \end{subfigure}
    \caption{Stationary distributions over levels under \textbf{normalized entropy abstention $h(\sigma)$}.}
    \label{fig:ex1_statdistn}
\end{figure}

\begin{figure}[htbp]
    \centering
    \begin{subfigure}[t]{0.48\textwidth}
        \centering
        \includegraphics[width=.8\textwidth]{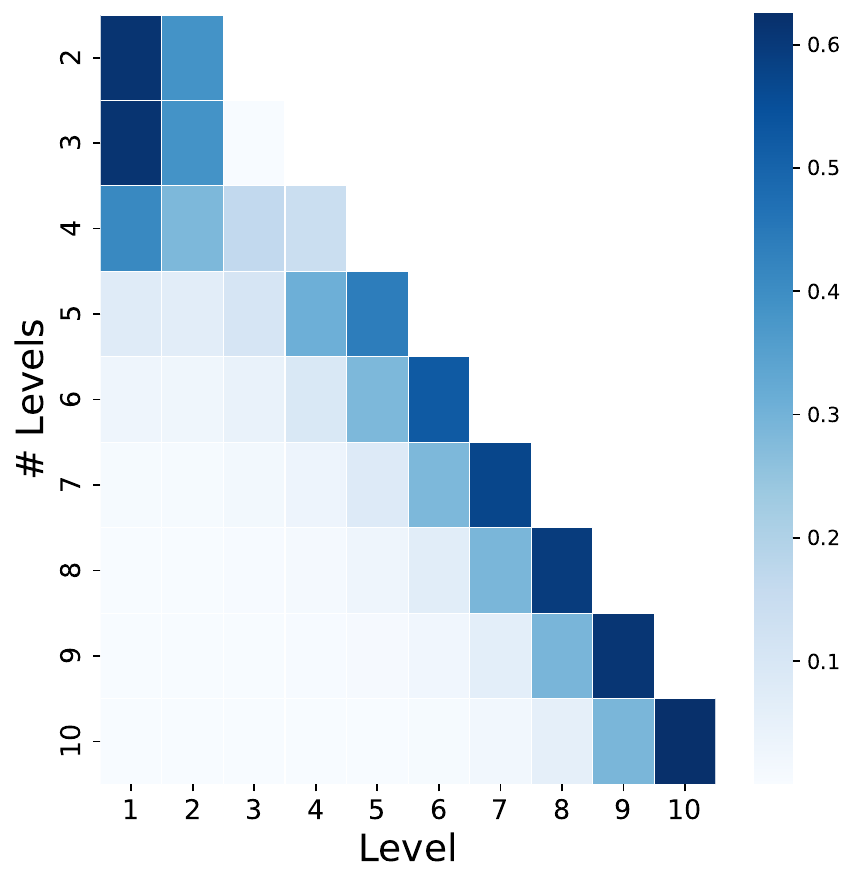}
        \caption{NG agents}
        \label{fig:ex3_statdistn_ng}
    \end{subfigure}%
    \hfill
    \begin{subfigure}[t]{0.48\textwidth}
        \centering
        \includegraphics[width=.8\textwidth]{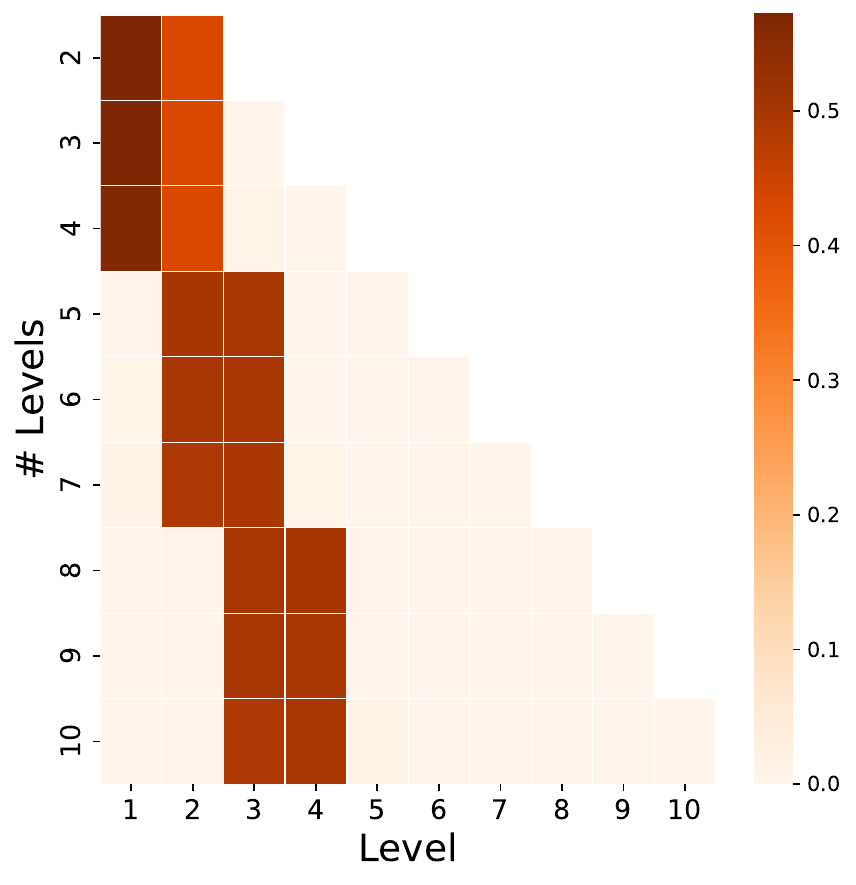}
        \caption{NI agents}
        \label{fig:ex3_statdistn_ni}
    \end{subfigure}
    \caption{Stationary distributions over levels under \textbf{polynomial abstention $h(\sigma)$} with $t=25.5$.}
    \label{fig:ex3_statdistn}
\end{figure}

\begin{figure}[htbp]
    \centering
    \begin{subfigure}[t]{0.48\textwidth}
        \centering
        \includegraphics[width=.8\textwidth]{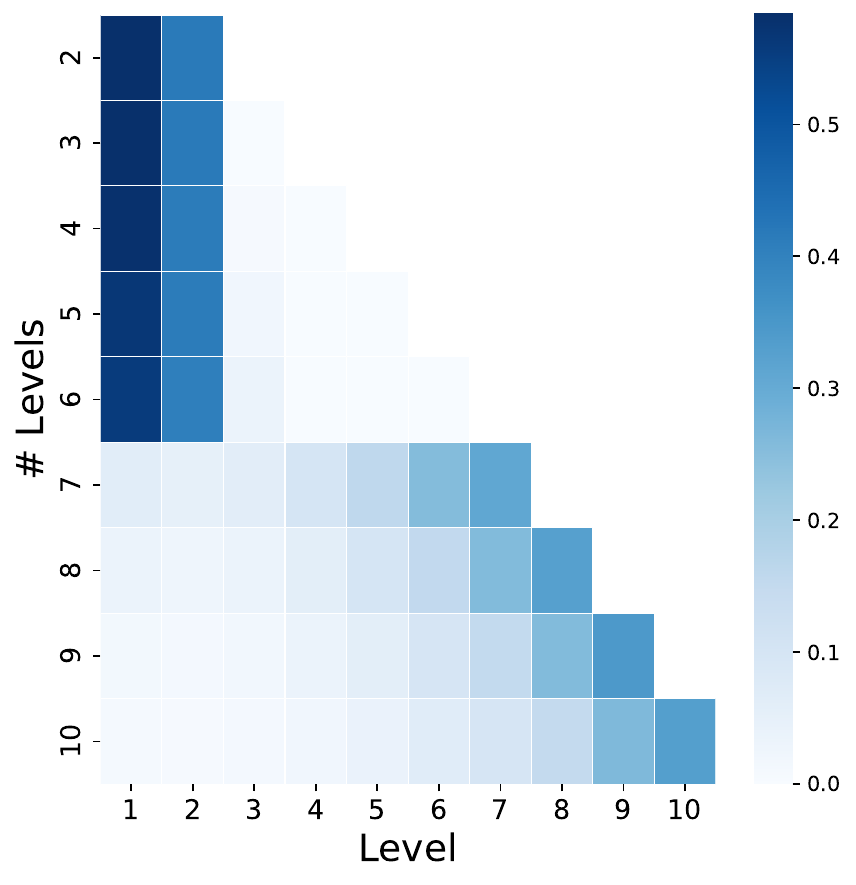}
        \caption{NG agents}
        \label{fig:ex2_statdistn_ng}
    \end{subfigure}%
    \hfill
    \begin{subfigure}[t]{0.48\textwidth}
        \centering
        \includegraphics[width=.8\textwidth]{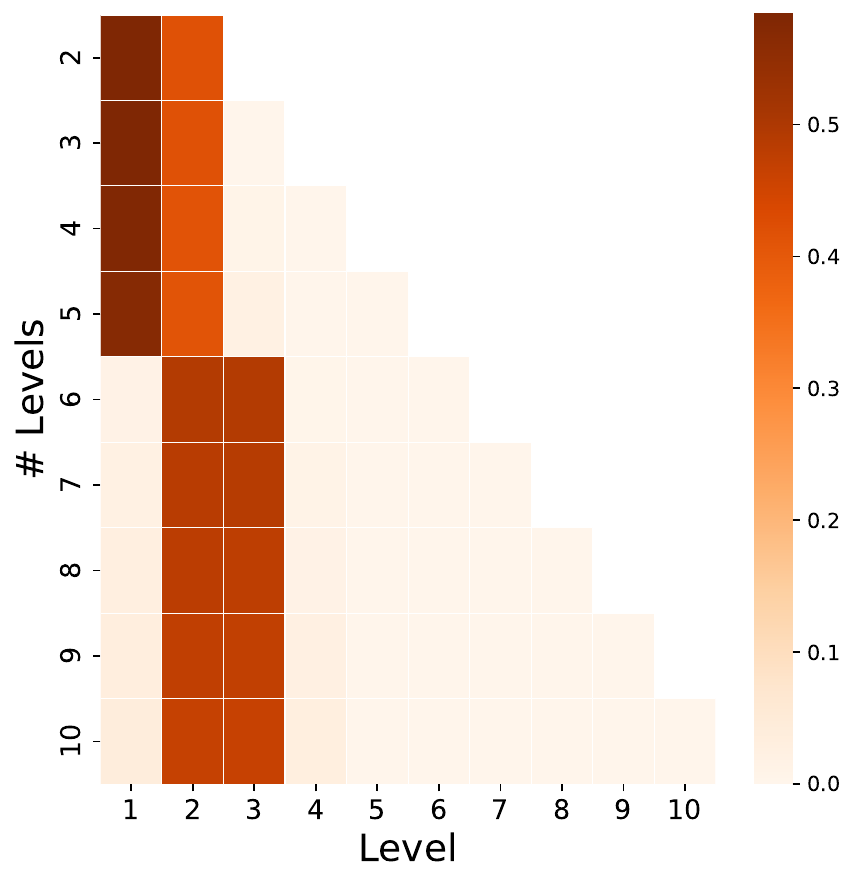}
        \caption{NI agents}
        \label{fig:ex2_statdistn_ni}
    \end{subfigure}
    \caption{Stationary distributions over levels under \textbf{absolute value abstention $h(\sigma)$}.}
    \label{fig:ex2_statdistn}
\end{figure}

\cref{fig:ex1_statdistn} compares the stationary distributions of agent levels across different total level counts for the NG and NI agents. Under the NG strategy, when the number of levels is high enough (above 3), we observe a clear diagonal pattern, where the agent tends to concentrate around the highest level attainable. In contrast, under the NI strategy, most of the density is distributed among two levels, which are towards the lower end of the level spectrum. The same trend is seen in \cref{fig:ex3_statdistn}, where we perform the same test but for the polynomial abstention function. For the absolute value abstention function, we see in \cref{fig:ex2_statdistn} that more levels are needed until the NG agent begins to concentrate at the higher levels.

\subsection{On the reward per level ($r$)}\label{app:r}
Here we explore how the size of the per level reward, denoted by $r$, influences agent average level and utility under NI and NG. We fix the number of levels to 10, use a polynomial abstention function, and sweep $r$ from 0 to 3. We use unit costs of effort 0.75 and 0.80 for NI and NG, respectively.

\begin{figure}[htbp]
    \centering
    \begin{subfigure}[t]{0.48\textwidth}
        \centering
        \includegraphics[width=\textwidth]{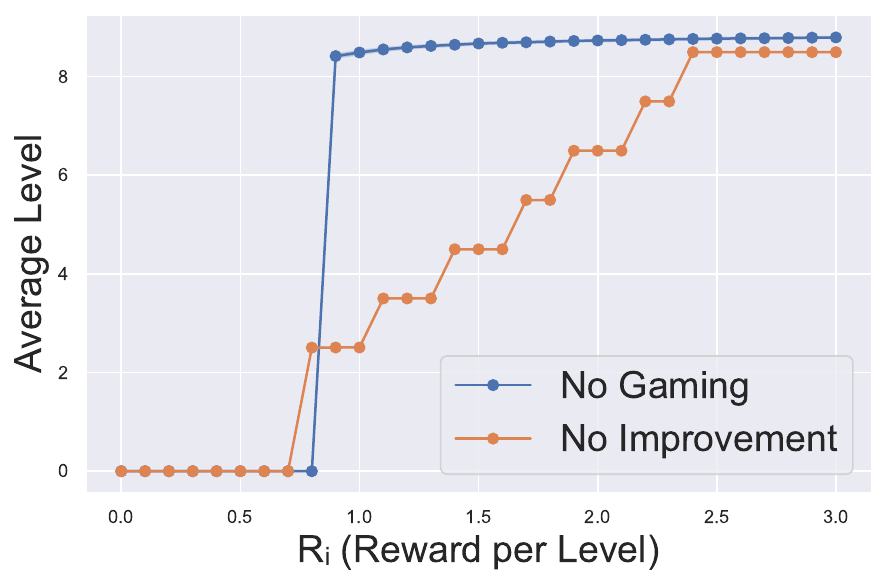}
        \caption{Average level.}
        \label{fig:ex6_level}
    \end{subfigure}
    \hfill
    \begin{subfigure}[t]{0.48\textwidth}
        \centering
        \includegraphics[width=\textwidth]{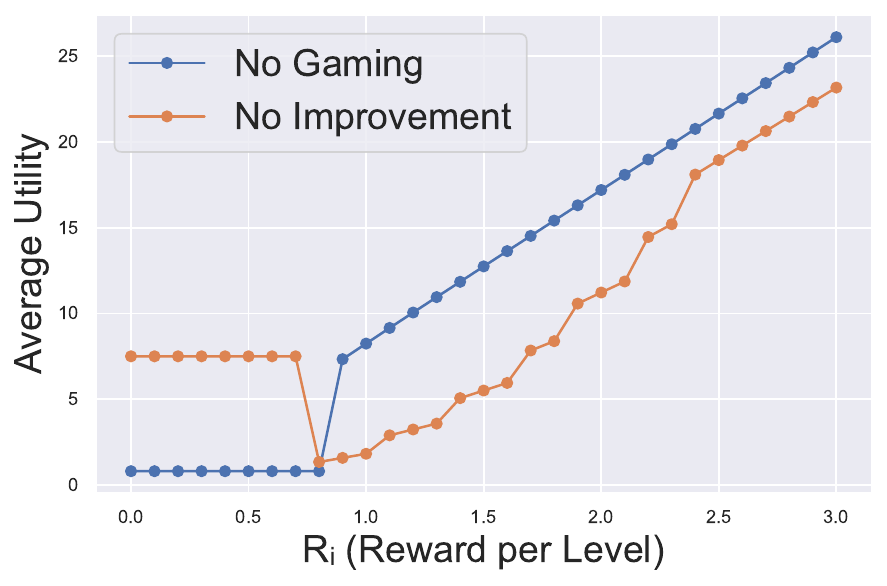}
        \caption{Average utility.}
        \label{fig:ex6_util}
    \end{subfigure}
    \caption{Agent's average level and utility versus $r$ $\pm$2 STD.}
    \label{fig:ex6_summary_panel}
\end{figure}

\cref{fig:ex6_level} shows how the agent's average level varies with reward per level under both NG and NI. When $r$ is below the unit cost of effort, both agent types remain at the lowest level since the marginal reward does not justify the cost of effort. When $r$ increases beyond the cost of effort, the NI agent gradually ascends through the levels, responding to the rise in incentives with more gaming effort. The NG agents show a steep increase in average level. Once $r$ increases beyond the cost of improvement, the average level of the NG agents rapidly jumps and saturates at around the highest level. This behavior appears because NG effort yields gains in true qualification, making it very rewarding when incentives are high enough.

\cref{fig:ex6_util} shows how the agent's average utility varies with reward per level. For the NI agent, utility starts off higher compared to the NG agent, because of the cheaper cost of gaming. For both the NG and NI agents, utility remains flat at lower $r$ values. This can be understood through the lens of the auxiliary function $G(i,z)$. When $r$ is small, the function may lack a clear minimizer or have multiple local peaks that do not satisfy \cref{asm:local-maximizer}, leading to preference for no action. As $r$ increases, the agent begins to exert effort, leading to an increase in utility. After a certain threshold, we see that our assumptions hold and the NG utility is higher than the NI utility.

\subsection{On the role of classifier sharpness ($\alpha$)}\label{app:alpha}
Here we explore how the sharpness of the classifier, controlled by $\alpha$, influences average level, utility, and qualification for NG and NI agents. We fix all other parameters and sweep $\alpha$ from 4 to 11. For each value of $\alpha$, we simulate both NI and NG agents using a parabolic, absolute value, and entropy-based abstention function. 

\begin{figure}[htbp]
    \centering
    \begin{subfigure}[t]{0.48\textwidth}
        \centering
        \includegraphics[width=\textwidth]{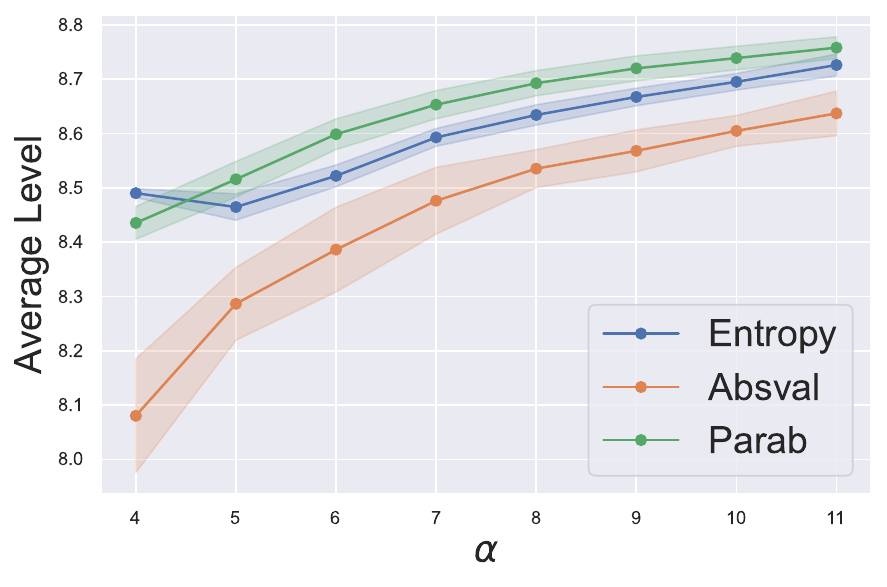}
        \caption{NG agents}
        \label{fig:ex8_level_ng}
    \end{subfigure}%
    \hfill
    \begin{subfigure}[t]{0.48\textwidth}
        \centering
        \includegraphics[width=\textwidth]{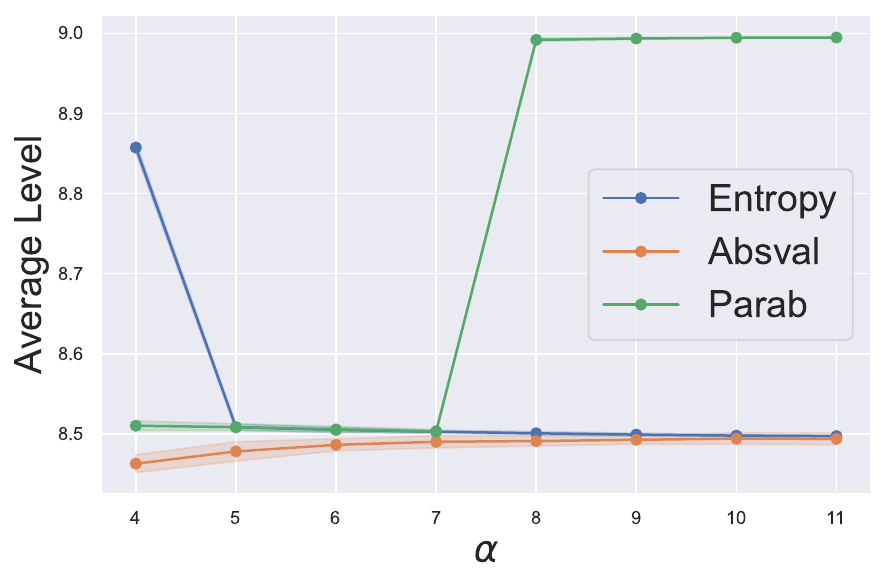}
        \caption{NI agents}
        \label{fig:ex8_level_ni}
    \end{subfigure}
    \caption{\textbf{Average level} versus $\alpha$ for NG and NI agents under different abstention functions $\pm$2 STD.}
    \label{fig:ex8_level_panel}
\end{figure}

\begin{figure}[htbp]
    \centering
    \begin{subfigure}[t]{0.48\textwidth}
        \centering
        \includegraphics[width=\textwidth]{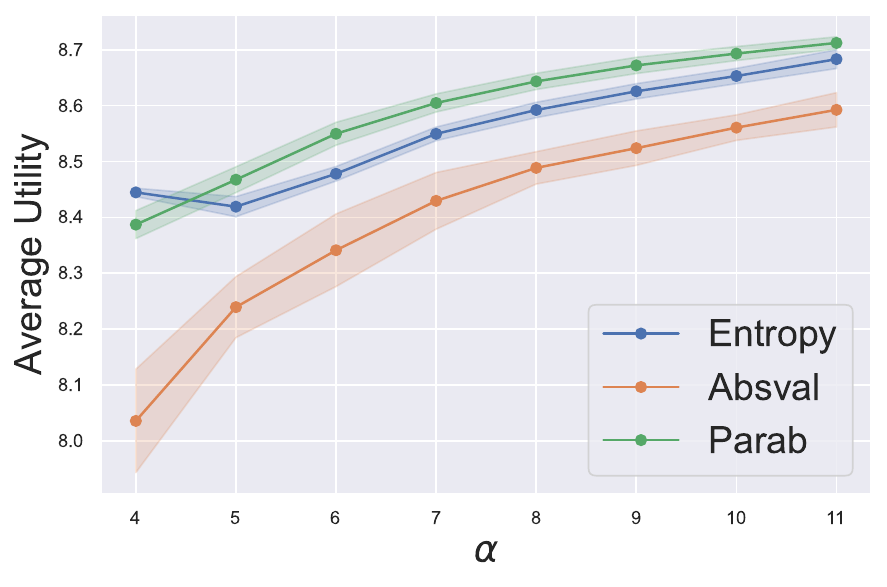}
        \caption{NG agents}
        \label{fig:ex8_util_ng}
    \end{subfigure}%
    \hfill
    \begin{subfigure}[t]{0.48\textwidth}
        \centering
        \includegraphics[width=\textwidth]{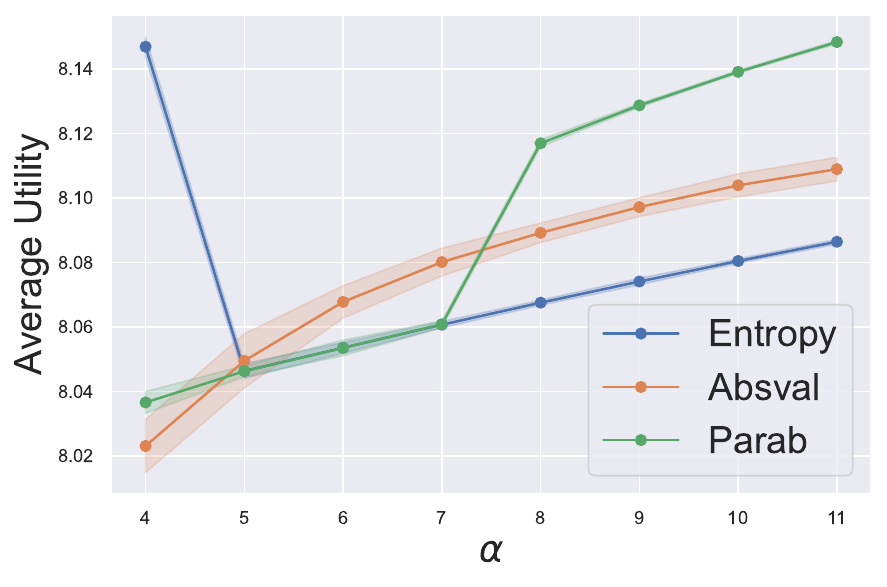}
        \caption{NI agents}
        \label{fig:util_ni}
    \end{subfigure}
    \caption{\textbf{Average utility} versus $\alpha$ under different abstention functions $\pm$2 STD.}
    \label{fig:ex8_util_panel}
\end{figure}

\begin{figure}[htbp]
    \centering
    \begin{subfigure}[t]{0.48\textwidth}
        \centering
        \includegraphics[width=\textwidth]{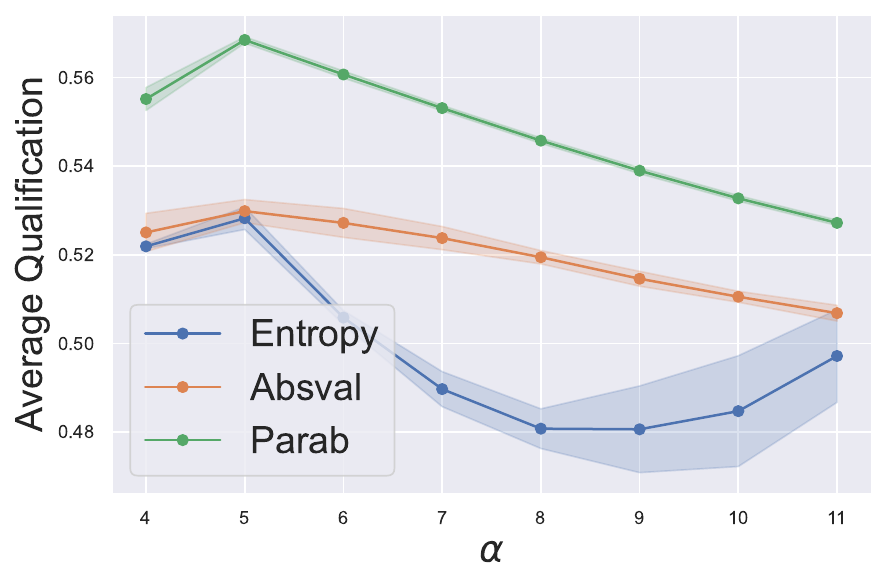}
        \caption{NG agents}
        \label{fig:qual_ng}
    \end{subfigure}%
    \hfill
    \begin{subfigure}[t]{0.48\textwidth}
        \centering
        \includegraphics[width=\textwidth]{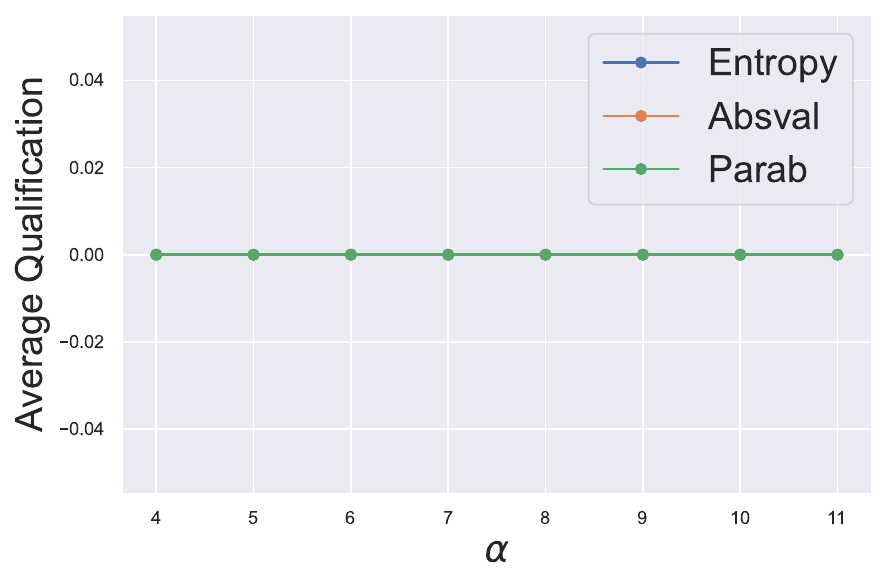}
        \caption{NI agents}
        \label{fig:qual_ni}
    \end{subfigure}
    \caption{\textbf{Average qualification} versus $\alpha$ under different abstention functions $\pm$2 STD.}
    \label{fig:ex8_qual_panel}
\end{figure}

\cref{fig:ex8_level_panel} shows that for the NG agents, increasing $\alpha$ generally leads to higher average levels. This trend suggests that as the classifier becomes sharper, it is more confident, enabling agents who improve their true attribute to ascend through the levels more effectively. For the NI agents, their average level is relatively flat for the absolute value abstention function, and shows a jump for the entropy and parabolic abstention functions. Overall, \cref{fig:ex8_level_panel} shows that increasing classifier sharpness $\alpha$ benefits improvement more consistently than gaming. A similar trend can be seen in \cref{fig:ex8_util_panel}, where in the NG setting, average utility generally increases as $\alpha$ increases, which is because that more confident classifications yield greater rewards. Finally, in \cref{fig:ex8_qual_panel} we see that the average qualification exhibits mild sensitivity to $\alpha$.

\subsection{On the role of abstention sensitivity ($\tilde{\beta}$)}\label{app:beta}
In this subsection, we test the influence of the abstention sensitivity parameter $\tilde{\beta}:=h(1/2)$, defined as the abstention probability when the classifier outcome is the most uncertain. 
As $\tilde{\beta}$ increases, the abstention function reacts more strongly to changes in classifier confidence. Here we fix all parameters except $\tilde{\beta}$ and sweep it from 0.1 to 0.6, and examine outcomes under the three different abstention functions mentioned in the previous experiments. 

\begin{figure}[htbp]
    \centering
    \begin{subfigure}[t]{0.48\textwidth}
        \centering
        \includegraphics[width=\textwidth]{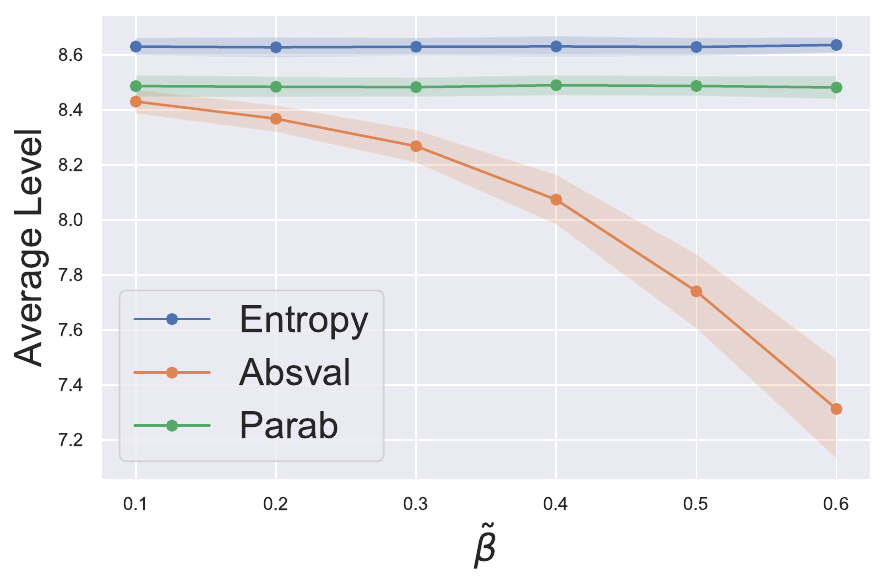}
        \caption{NG agents}
        \label{fig:ex9_level_ng}
    \end{subfigure}%
    \hfill
    \begin{subfigure}[t]{0.48\textwidth}
        \centering
        \includegraphics[width=\textwidth]{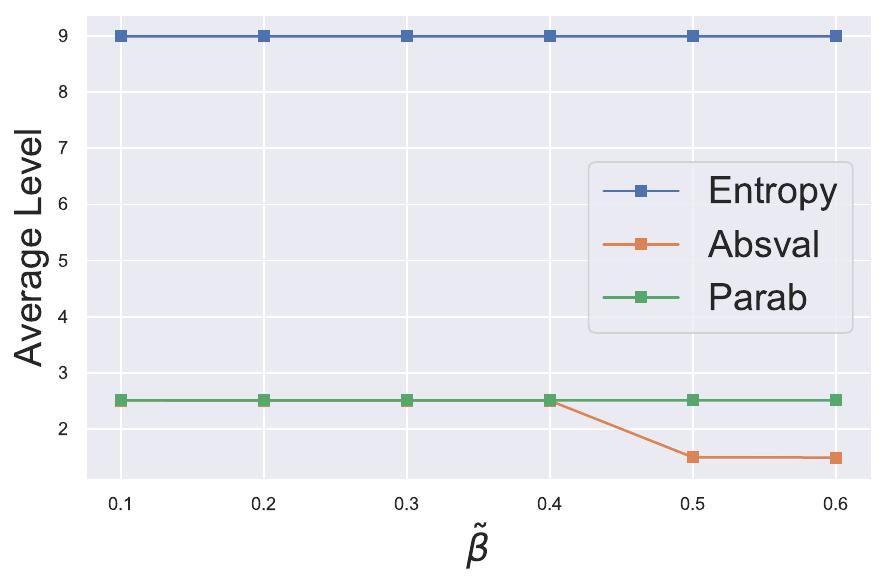}
        \caption{NI agents}
        \label{fig:ex9_level_ni}
    \end{subfigure}
    \caption{\textbf{Average level} versus $\tilde{\beta}$ for NG and NI agents under different abstention functions $\pm$2 STD.}
    \label{fig:ex9_level_panel}
\end{figure}

\begin{figure}[htbp]
    \centering
    \begin{subfigure}[t]{0.48\textwidth}
        \centering
        \includegraphics[width=\textwidth]{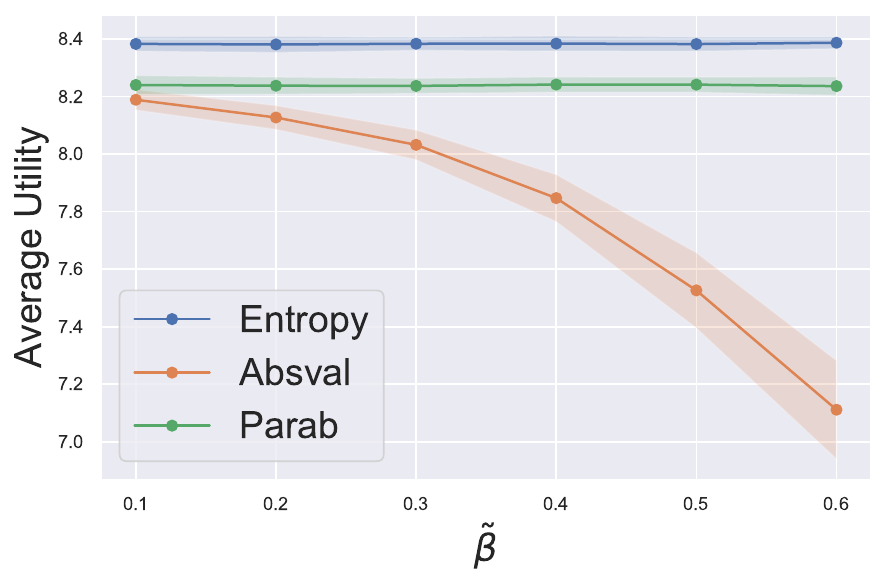}
        \caption{NG agents}
        \label{fig:ex9_util_ng}
    \end{subfigure}%
    \hfill
    \begin{subfigure}[t]{0.48\textwidth}
        \centering
        \includegraphics[width=\textwidth]{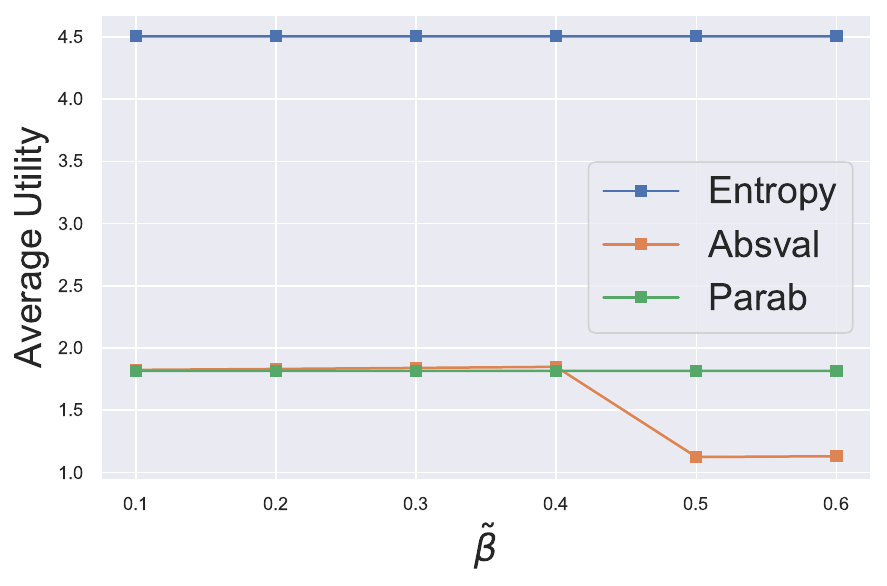}
        \caption{NI agents}
        \label{fig:ex9_util_ni}
    \end{subfigure}
    \caption{\textbf{Average utility} versus $\tilde{\beta}$ for NG and NI agents under different abstention functions $\pm$2 STD.}
    \label{fig:ex9_util_panel}
\end{figure}

\begin{figure}[htbp]
    \centering
    \begin{subfigure}[t]{0.48\textwidth}
        \centering
        \includegraphics[width=\textwidth]{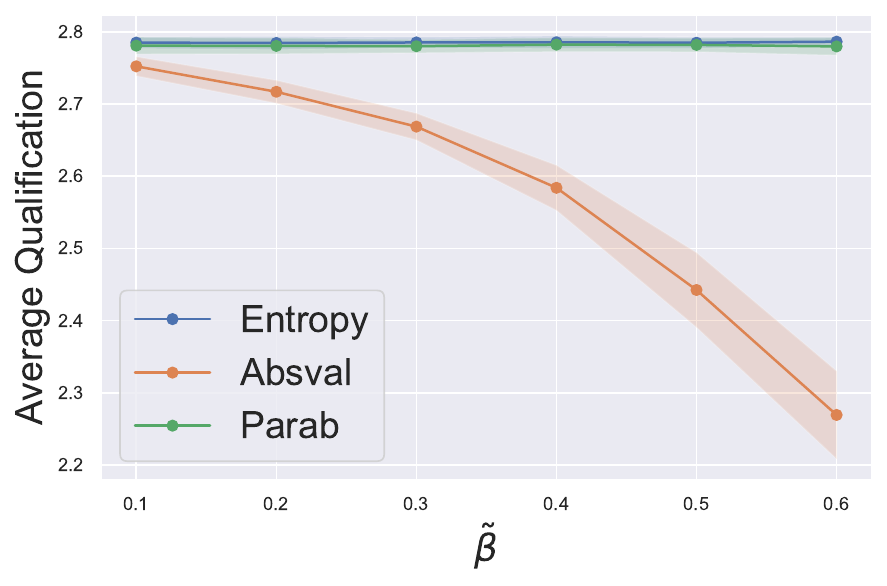}
        \caption{NG agents}
        \label{fig:ex9_qual_ng}
    \end{subfigure}%
    \hfill
    \begin{subfigure}[t]{0.48\textwidth}
        \centering
        \includegraphics[width=\textwidth]{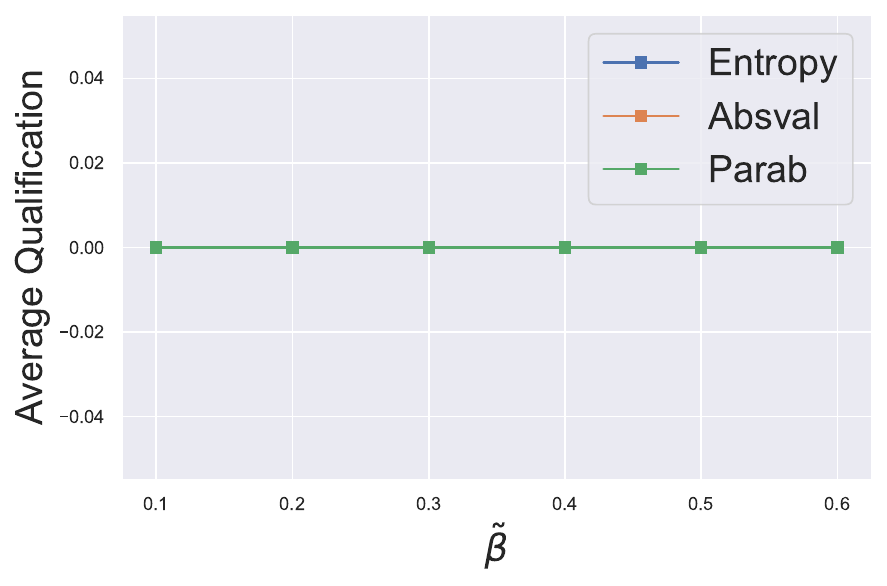}
        \caption{NI agents}
        \label{fig:ex9_qual_ni}
    \end{subfigure}
    \caption{\textbf{Average qualification} versus $\tilde{\beta}$ for NG and NI agents under different abstention functions $\pm$2 STD.}
    \label{fig:ex9_qual_panel}
\end{figure}

\cref{fig:ex9_level_panel} shows how varying the abstention sensitivity $\tilde{\beta}$ affects the average level of the NG and NI agents for all three abstention functions. Here we see that for the NG agent, the entropy and parabolic abstention functions are mildly sensitive to $\tilde{\beta}$. However, for the absolute value abstention function, we see a decrease in average level as $\tilde{\beta}$ increases. This is because the absolute value abstention function is more aggressive in penalizing uncertainty near the decision boundary. This suggests that having a very reactive abstention function can hurt NG agents. For the NI agents, we see that the entropy abstention function yields the highest average level. This is because the entropy abstention function is the most tolerant abstention function near the decision boundary. Since NI agents do not change their true qualification, a more tolerant abstention function can lead to more promotions and a higher level. For the average utility plot, \cref{fig:ex9_util_panel}, we observe the same trends. The same goes for the average qualification plot \cref{fig:ex9_qual_panel}.

\end{document}